\ifdefined\XeTeXversion\else\ifdefined\pdfoutput\pdfoutput=1\fi\fi
\documentclass[11pt,letterpaper]{article}
\usepackage[utf8]{inputenc}
\usepackage[T1]{fontenc}
\usepackage{amsmath,amsthm,amssymb,amsfonts,mathtools,bm}
\usepackage{graphicx,booktabs,nicefrac,microtype,xcolor,pifont}
\usepackage{tikz}
\usepackage[ruled,vlined]{algorithm2e}
\usepackage{placeins}
\usepackage{url}
\usepackage{authblk}
\usepackage{iftex}
\ifXeTeX\PassOptionsToPackage{xetex}{hyperref}\fi
\PassOptionsToPackage{authoryear}{natbib}
\usepackage{arxiv-2}
\usepackage{setspace}
\usepackage{lmodern}

\newtheorem{theorem}{Theorem}
\newtheorem{lemma}{Lemma}
\newtheorem{assumption}{Assumption}
\newtheorem{definition}{Definition}
\newtheorem{remark}{Remark}
\newtheorem{property}{Property}
\newtheorem{corollary}{Corollary}
\newtheorem{proposition}{Proposition}

\newcommand{\Estoff}{\mathcal{E}_{\mathrm{off}}}

\let\hat\widehat
\let\tilde\widetilde

\newcommand{\cA}{\mathcal{A}}

\newcommand{\cC}{\mathcal{C}}
\newcommand{\cD}{\mathcal{D}}
\newcommand{\cE}{\mathcal{E}}
\newcommand{\cF}{\mathcal{F}}

\newcommand{\cH}{\mathcal{H}}

\newcommand{\cK}{\mathcal{K}}

\newcommand{\cM}{\mathcal{M}}
\newcommand{\cN}{\mathcal{N}}
\newcommand{\cO}{\mathcal{O}}
\newcommand{\cP}{\mathcal{P}}

\newcommand{\cS}{{\mathcal{S}}}

\newcommand{\cV}{\mathcal{V}}

\newcommand{\cX}{\mathcal{X}}

\newcommand{\cZ}{\mathcal{Z}}

\newcommand{\EE}{\mathbb{E}}

\newcommand{\PP}{\mathbb{P}}

\newcommand{\RR}{\mathbb{R}}

\newcommand{\argmax}{\mathop{\mathrm{argmax}}}

\SetAlgoNlRelativeSize{0}
\SetNlSty{textbf}{\scriptsize}{ }

\title{\normalfont Eluder-Dimension-Free Regret for Stage-Wise Constrained Contextual Bandits via Offline Regression}
\author[1]{Haichen Hu\thanks{Email: \texttt{huhc@mit.edu}}}
\author[1]{David Simchi-Levi\thanks{Email: \texttt{dslevi@mit.edu}}}
\author[1]{Navid Azizan\thanks{Email: \texttt{azizan@mit.edu}}}
\affil[1]{Massachusetts Institute of Technology\\Cambridge, USA, 02139}
\date{}
\hypersetup{
  pdftitle={Eluder-Dimension-Free Regret for Stage-Wise Constrained Contextual Bandits via Offline Regression},
  pdfauthor={Haichen Hu, David Simchi-Levi, Navid Azizan},
  pdfsubject={Contextual bandits with stage-wise constraints}
}

\begin{document}
\maketitle
\singlespacing
\allowdisplaybreaks[1]
\begin{abstract}
Contextual bandits with stage-wise constraints model decisions that must
control risk while learning from their outcomes. We study this problem
when only the selected action's reward and cost are observed, with a
general reward class and a convex class of nonlinear cost functions.
We propose CS-UCCB, which combines offline regression, cost
confidence sets, and upper counterfactual confidence bounds.
Each regression fit defines a valid cost confidence set, while a
geometry derived from covering entropy controls cumulative cost
uncertainty along the sampled decisions. Under polylogarithmic covering
entropy for both classes and polylogarithmic offline estimation error,
the finite-action algorithm guarantees, with high
probability, feasibility at every round and
$\widetilde{\mathcal O}(\sqrt T)$ regret without an eluder-dimension term.
We also extend the analysis to infinite actions under an explicit
reward-design condition and compare our bound with existing
eluder-dimension bounds.
\end{abstract}

\section{Introduction}\label{sec:intro}
Contextual sequential decision-making has applications in online
platforms \citep{ban2024neural} and healthcare \citep{zhou2023spoiled},
and is commonly modeled through contextual bandits
\citep{tewari2017ads,tang2014ensemble}. Many applications also impose
safety, fairness, or risk constraints
\citep{min2024hard,ceusters2023safe,verma2024online,galichet2013exploration}.
These requirements complicate exploration: a decision can be informative
yet inadmissible because its cost is uncertain. Existing
frameworks include knapsack bandits with cumulative resource budgets
\citep{badanidiyuru2018bandits,immorlica2022adversarial}, survival bandits
\citep{riou2024the,veroutis2024survival}, and conservative bandits that
compare cumulative performance with a baseline
\citep{wu2016conservative,kazerouni2017conservative}.

We consider stage-wise constraints on the expected cost of each
context-dependent randomized decision. The learner observes the reward
and cost of its selected action after acting, and receives no cost
labels for other actions. It must learn both responses while keeping
each chosen distribution feasible. Linear models have been studied
extensively \citep{amani2019linear,moradipari2020stage,pacchiano2021stochastic,liu2021efficient}.
For nonlinear models, \citet{pacchiano2024contextual} analyze regression
confidence sets using the eluder dimension, which describes how model
uncertainty can persist along sequentially chosen inputs
\citep{russo2013eluder}. This quantity can be much larger than the
logarithm of a covering number. Offline regression, by contrast,
provides a direct way to estimate these functions from a batch of data.
These observations motivate the question:
\begin{center}
\itshape
Can we use offline regression oracles to guarantee stage-wise
feasibility and optimal regret for nonlinear contextual bandits
without dependence on the eluder dimension?
\end{center}

\paragraph{Our contribution.}
We propose CS-UCCB, an offline-oracle-based algorithm that combines
cost confidence sets with constrained counterfactual policy
simulation. At each round, we test the cost class against the
current offline fit, certifying feasible policies and controlling
cost uncertainty in optimistic exploration. Counterfactual simulation
evaluates historical policies at the current context to construct
reward confidence bounds from offline regression estimates.

Our main technical step connects cost entropy to sequential uncertainty.
Using convexity and the entropy of cost differences, we construct a
finite-dimensional uniform approximation and a normalized geometric
representation for analysis. The current confidence-set constraint supplies a common
empirical error budget for models in the current set. Combining this budget
with the derived geometry controls cumulative cost width under the
policies actually sampled. A cost-width bonus then compensates
for the reward loss of maintaining a pessimistic feasible region.
Together with UCCB reward confidence, this gives feasibility from the
first round and an explicit regret bound in terms of covering entropy
and offline estimation error. Polylogarithmic budgets yield
$\widetilde{\mathcal O}(\sqrt T)$ regret for finite actions.

We also extend the decision and regret analysis to infinite actions
under a cumulative reward-design condition. The geometric
representation is a proof tool; the algorithm does not require a
supplied kernel or feature map. Each cost confidence set uses one
convex regression constraint. With a tractable class representation,
this permits convex support queries for upper costs and
widths, without enumerating a covering net. We specify the required
regression and optimization access separately.

\paragraph{Paper structure.}
Table~\ref{tab:comparison} summarizes related methods, with a detailed
literature review in Appendix~\ref{app:related-work}.
Section~\ref{sec:model} introduces the problem setting and assumptions,
followed by the algorithm in Section~\ref{sec:algo_design}.
We establish its finite-action guarantees in
Section~\ref{sec:theoretical_guarantees} and extend the analysis to
infinite actions in Section~\ref{sec:extension}.
Nonlinear examples satisfying our assumptions are provided in
Appendix~\ref{app:parametric-entropy}. In
Appendix~\ref{app:pacchiano-comparison}, we compare our regret guarantee
with the published bound of \citet{pacchiano2024contextual} under the
same feedback model, exhibiting polynomial reward and logarithmic
cost improvements with the comparison's scope made explicit.
Computational requirements and implementations are
discussed in Appendix~\ref{app:computation}.

\begin{table}[t]
\caption{Comparison under each work's assumptions and feedback model.
Rates suppress logarithms and fixed safety and complexity parameters,
with polylogarithmic regression error where applicable. We show the
general rate of \citet{amani2019linear}, without an additional known
positive optimal-action safety gap. Our rate uses the stated reward
covering profile and convex cost entropy assumptions. ED denotes an
explicit eluder-dimension term; $d_E$ abbreviates its combined reward
and cost dependence.}
\label{tab:comparison}
\centering\small
\setlength{\tabcolsep}{3pt}
\resizebox{\textwidth}{!}{%
\begin{tabular}{lllllcl}
\toprule
Method & Function classes & Regression & Constraint & \shortstack[l]{Feedback for\\constraints} & ED & Regret \\
\midrule
\citet{amani2019linear} & Linear & Ridge LS & Stage-wise & \shortstack[l]{Loss; shared\\parameter} & -- & $\widetilde O(T^{2/3})$ \\
\citet{pacchiano2021stochastic} & Linear & Ridge LS & Stage-wise & \shortstack[l]{Selected-action\\cost} & -- & $\widetilde O(\sqrt T)$ \\
\citet{slivkins2023contextual} & Regression classes & Online oracle & Packing/covering & \shortstack[l]{Selected-action\\consumption} & -- & Problem dependent \\
\citet{deb2025conservative} & Nonlinear & Online oracle & \shortstack[l]{Cumulative\\baseline} & \shortstack[l]{Loss; known\\baseline} & -- & $\widetilde O(\sqrt T)$ \\
\citet{pacchiano2024contextual} & Finite nonlinear & Offline LS & Stage-wise & \shortstack[l]{Selected-action\\cost} & Yes & $\widetilde O(\sqrt{Td_E}+d_E)$ \\
This work & \shortstack[l]{General reward;\\convex cost} & Offline oracle & Stage-wise & \shortstack[l]{Selected-action\\cost} & No & $\widetilde O(\sqrt T)$ \\
\bottomrule
\end{tabular}}
\end{table}

\section{Problem Setup}\label{sec:model}

\paragraph{Notation.}
For a set $B$, let $\cP(B)$ denote its probability distributions.
We write $[n]=\{1,\ldots,n\}$, $\cZ=\cX\times\cA$, and
$\|h\|_\infty=\sup_{z\in\cZ}|h(z)|$.
The covering number $\cN(\cK,u,\|\cdot\|_\infty)$ uses centers in
$\cK$ and balls of radius at most $u$. For a distribution $p$ on
actions, $h(x,p)=\int h(x,a)\,p(da)$; for a policy $\pi$, we abbreviate
$h(x,\pi)=h(x,\pi(\cdot\mid x))$.
We use standard asymptotic notation $\mathcal O$, $\Omega$, and
$\Theta$, and write $\widetilde{\mathcal O}$ to suppress
polylogarithmic factors.

\subsection{Contextual Bandits with Stage-wise Constraints}
We study a $T$-round sequential decision-making problem with context
space $\cX$ and finite action set $\cA=\{a^1,\ldots,a^K\}$.
The extension to infinite action spaces is presented in
Section~\ref{sec:extension}. At round $t$, the learner observes
$X_t\overset{\mathrm{i.i.d.}}{\sim}Q_x$, selects a distribution
$\pi_t(\cdot\mid X_t)$, and samples $A_t$ from it. After taking the
action, the learner observes its reward $R_t\in[0,1]$ and cost
$Y_t\in[0,1]$. We assume realizability in known classes
$\cF,\cC\subseteq[0,1]^{\cZ}$:
\[
\begin{aligned}
\EE[R_t\mid\cH_{t-1},X_t,A_t]&=f^*(X_t,A_t),\quad f^*\in\cF,\\
\EE[Y_t\mid\cH_{t-1},X_t,A_t]&=c^*(X_t,A_t),\quad c^*\in\cC.
\end{aligned}
\]
The history $\cH_{t-1}$ records past observations and learner
randomization used before the current context. Reward and cost noises need not be
independent. Both feedback signals concern only the selected action;
the learner receives no observations of unselected actions.

A policy is a measurable kernel $\pi:\cX\to\cP(\cA)$.
Given a known threshold $\tau\in(0,1)$, its stage-wise constraint is
\begin{equation}\label{eq:stagewise_constraint}
c^*(x,\pi)\le\tau\qquad\text{for every }x\in\cX.
\end{equation}
This is the conditional expected cost of the randomized decision at
that context. An individual sampled action or realized cost can exceed
$\tau$; the guarantee concerns the distribution chosen before sampling.
As in the stage-wise model of \citet{pacchiano2024contextual}, replacing
\eqref{eq:stagewise_constraint} by an average over contexts would impose
a different constraint.
Concretely, $c^*(x,\pi)$ averages only over the policy's action
randomization while keeping the current context fixed; it does not
average over $Q_x$. A population constraint could offset a violation at
one context with slack at another, even though the former decision is
unsafe. Our requirement rules out this compensation: after observing
$x$, the learner must choose a distribution whose conditional mean cost
is at most $\tau$.

\begin{assumption}\label{ass:safe_action}
There are a known action $a_0$ and a known constant $0\le c_0<\tau$
such that $c^*(x,a_0)=c_0$ for every $x\in\cX$.
Write $\Delta=\tau-c_0$ and let $\pi_0$ always select $a_0$.
\end{assumption}
Since $a_0$ and $c_0$ are known, we discard all candidate cost functions
that violate this identity, retaining $c^*$ in the class. Thus, without
loss of generality, we assume $c(x,a_0)=c_0$ for every $c\in\cC$ and
$x\in\cX$ throughout; $\cC$ denotes this restricted class below.
The reward of $a_0$ need not be known.
The baseline makes the feasible set nonempty and allows exploration
through safe mixtures, a standard device in stage-wise constrained
bandits \citep{pacchiano2021stochastic,pacchiano2024contextual}.

The benchmark is a pointwise optimal feasible policy,
\[
\begin{aligned}
\pi^*(\cdot\mid x)\in\argmax_{p\in\cP(\cA)}\quad
&f^*(x,p)\\
\text{s.t.}\quad &c^*(x,p)\le\tau.
\end{aligned}
\]
If $f^*$ and $c^*$ were known, this is the policy a clairvoyant decision
maker would use at each context. Regret measures the reward lost
relative to this benchmark, which obeys the same pointwise safety
requirement as the learner.
We seek simultaneous feasibility at every round and context, with high
probability, together with a bound on
\begin{equation}\label{eq:regret_definition}
\operatorname{Reg}(T)=\sum_{t=1}^T
\EE_{X_t\sim Q_x}[f^*(X_t,\pi^*)-f^*(X_t,\pi_t)].
\end{equation}
Each expectation is conditional on the pre-round history $\cH_{t-1}$
and averages over the current context $X_t\sim Q_x$; the action
expectation is included in $f^*(X_t,\pi)$.
Regret therefore remains random through the learned policies. We assume
measurability of the suprema and policy rules below, and measurable
maximizers with a fixed tie-breaking rule.

\paragraph{Illustrative example.}
In an adaptive clinical trial, $X_t$ records a patient's characteristics,
$A_t$ is the assigned dosage, $R_t$ measures treatment benefit, and
$Y_t$ indicates an adverse event observed after treatment.
The unknown function $c^*(x,a)$ is the probability of that event.
Before assigning a dosage, the trial requires the patient's risk under
its randomized treatment rule to be at most $\tau$. Outcomes then
inform future decisions. A conservative treatment with known risk
provides the safe action. Thus both therapeutic benefit and risk are
learned from outcomes of the administered treatments.

\subsection{Function Classes and Offline Regression}
The analysis uses separate complexity assumptions for the reward and
cost classes, together with an offline regression guarantee.

\begin{assumption}\label{ass:reward_entropy}
The reward class $\cF$ contains $f^*$ and all possible reward-oracle
outputs. There are known constants $D_{\cF}\ge1$, $A_{\cF}\ge e$,
and $p_{\cF}\ge1$, independent of $T$ and $\delta$, such that
for $0<u\le1$,
\begin{equation}\label{eq:reward_entropy_assumption}
\mathcal H_{\cF}(u)=\log\cN(\cF,u,\|\cdot\|_\infty)
\le D_{\cF}[\log(A_{\cF}/u)]^{p_{\cF}}.
\end{equation}
\end{assumption}
The reward class need not be convex or admit a linear parameterization.

\begin{assumption}\label{ass:cost_entropy}
The class $\cC$ is closed under convex combinations of functions. There
are known constants $D_{\cC}\ge1$, $A_{\cC}\ge e$, and $p_{\cC}\ge1$,
independent of $T$ and $\delta$, such that, for $0<u\le1$,
\begin{equation}\label{eq:cost_entropy_assumption}
\log\cN(\cC,u,\|\cdot\|_\infty)
\le D_{\cC}[\log(A_{\cC}/u)]^{p_{\cC}}.
\end{equation}
\end{assumption}
Convexity here means closure under mixtures of cost functions, not
convexity of each function in its inputs. Nonlinear basis expansions
and suitably bounded infinite series provide examples. The exponent
$p_{\cC}>1$ permits infinite-dimensional convex classes; when
$p_{\cC}=1$, the all-scale bound forces finite affine dimension.
We give examples in Appendix~\ref{app:parametric-entropy}.
In Appendix~\ref{app:cost-eluder-relation}, we prove that this cost
condition also bounds the eluder dimension polylogarithmically.
Appendix~\ref{app:entropy-only-counterexample} explains why small
covering entropy alone is insufficient for general nonconvex costs.

\begin{definition}\label{def:offline_online-oracle}
For observations $L_i=h^*(Z_i)+\xi_i$ with $h^*\in\cK$ and conditionally mean-zero
noise, an offline regression oracle returns a predictor $\hat h^t$
from the first $t-1$ observations. For every admissible data-collection
rule and failure probability $\rho\in(0,1)$, its guarantee is
\begin{equation}\label{eq:offline_guarantee}
\PP\left\{\sum_{i<t}(\hat h^t(Z_i)-h^*(Z_i))^2
\le\Estoff(\cK,t,\rho)\right\}\ge1-\rho.
\end{equation}
The inputs may depend on past observations, but are chosen before their
current labels are observed. The same fitted predictor is evaluated on
all $t-1$ samples in this guarantee.
\end{definition}
This distinction matters: the offline bound refits once on the full
historical batch and evaluates $\hat h^t$ retrospectively at every
$Z_i$. By contrast, an online guarantee controls $\tilde h^i(Z_i)$,
where the predictor may change with $i$ and is formed before seeing
$L_i$. The confidence-set construction uses the batch guarantee to test candidate cost models
against each newly fitted predictor on the available observations.

\begin{assumption}\label{ass:oracle}
We have offline oracles for $\cF$ and $\cC$, with reward output
$\hat f^t\in\cF$ and known budgets in~\eqref{eq:offline_guarantee}.
The budgets are nondecreasing in sample size and nonincreasing in
the failure probability $\rho$. All oracle outputs can be evaluated at queried
context--action pairs.
\end{assumption}
Clipping an output to $[0,1]$ preserves its error guarantee. The cost
output need not belong to $\cC$ because it serves only as a reference
in the confidence-set constraints. The statistical results keep the entropy
parameters and both oracle budgets explicit. Polylogarithmic offline
budgets give the $\widetilde{\cO}(\sqrt T)$ specialization.

\section{Algorithm Design}\label{sec:algo_design}
We combine cost confidence sets with upper counterfactual confidence
bounds in \textsf{CS-UCCB} (Confidence Set UCCB). Both regression
oracles use only previously observed rewards and costs. The cost
confidence sets define the feasible region, and UCCB chooses a robustly feasible distribution
that balances predicted reward and uncertainty. We give the main
algorithm and its subroutines in
Algorithms~\ref{alg:UCCB_pseudo_code}--\ref{alg:counter-factual_simulation}.

\begin{algorithm}[tb]
\setcounter{AlgoLine}{0}\SetAlgoLined\LinesNumbered
\caption{\textbf{CS-UCCB}}\label{alg:UCCB_pseudo_code}
\KwIn{Classes $\cF,\cC$, horizon $T$, confidence $\delta$, safe action
$a_0$.}
Set $\alpha_r=1/(\tau-c_0)$; initially no permanent safe-action instruction is recorded\;
\For{$t=1,\ldots,T$}{
\eIf{$t=1$}{
Choose any $\hat f^1\in\cF$; no regression call is needed\;
}{
Fit $\hat f^t$ to $\{(Z_i,R_i)\}_{i<t}$ and $\hat c^t$ to
$\{(Z_i,Y_i)\}_{i<t}$\;
}
Construct the cost confidence set (Algorithm~\ref{alg:cost_confidence}); archive $\hat f^t$, the defined
cost objects, coefficients, and any instruction to use $a_0$ permanently\;
Observe $X_t$ and run UCCB replay
(Algorithm~\ref{alg:counter-factual_simulation}) at $X_t$\;
Sample $A_t\sim\pi_t(\cdot\mid X_t)$ and observe $R_t,Y_t$\;
Append $(Z_t,R_t,Y_t)$, where $Z_t=(X_t,A_t)$, to the data\;
}
\end{algorithm}

\begin{algorithm}[!t]
\setcounter{AlgoLine}{0}\SetAlgoLined\LinesNumbered
\caption{\textbf{Cost Confidence Set Construction}}\label{alg:cost_confidence}
\KwIn{Round $t$, cost class $\cC$, threshold $\tau$, safe action $a_0$,
data $\{Z_i=(X_i,A_i),Y_i\}_{i<t}$, cost fit $\hat c^t$
when $t\ge2$, budget
$B_t^{\cC}=\Estoff(\cC,t,\delta/(4T))$ for $t\ge2$,
and any recorded instruction to choose $a_0$ permanently.}
\KwOut{The convex cost confidence set, its upper costs and widths,
and the robust feasible region, or the instruction to choose $a_0$
in this and all remaining rounds.}
\If{a previous round recorded the instruction to choose $a_0$ permanently}{Return the same instruction\;}
\eIf{$t=1$}{Set $\cS_1=\cC$\;}{
Construct the current regression confidence set:
\[
\cS_t=\left\{c\in\cC:
\sum_{i=1}^{t-1}(c(Z_i)-\hat c^t(Z_i))^2\le B_t^{\cC}\right\}.
\]
}
\If{$\cS_t=\varnothing$}{Record and return the instruction to choose $a_0$ in this and all remaining rounds\;}
For every context $x$, action $a$, and action distribution $p$, define
\[
\begin{aligned}
V_t^\delta(x,p)&=\sup_{c\in\cS_t}\int_{\cA}c(x,a)\,p(da),
&\quad
w_t(x,a)&=\sup_{c,c'\in\cS_t}|c(x,a)-c'(x,a)|,\\
w_t(x,p)&=\int_{\cA}w_t(x,a)\,p(da),
&
\Pi_t^\delta(x)&=\{p\in\cP(\cA):V_t^\delta(x,p)\le\tau\}.
\end{aligned}
\]
\nl Archive the confidence set and its policy objects; when $t\ge2$, retain the current fit, data prefix, and budget\;
Return $\cS_t,V_t^\delta,w_t,\Pi_t^\delta$\;
\end{algorithm}

\begin{algorithm}[tb]
\setcounter{AlgoLine}{0}\SetAlgoLined\LinesNumbered
\caption{\textbf{UCCB replay at a context}}\label{alg:counter-factual_simulation}
\KwIn{Round $t$, context $x$, archived reward fits, cost objects,
coefficients, and any recorded instruction to choose $a_0$ through $t$.}
Initialize $\omega_1(x,a)=0$ for all $a\in\cA$\;
\For{$i=1,\ldots,t$}{
\eIf{the archive at $i$ records the instruction to choose $a_0$}{
Set $\widetilde\pi_{t,i}=\pi_0$\;
}{
Set $q_i(x,a)=1/(1+\omega_i(x,a))$ for each action\;
Choose $\widetilde\pi_{t,i}\in\Pi_i^\delta(x)$ maximizing
$\hat f^i(x,p)+\beta_iq_i(x,p)+\alpha_rw_i(x,p)$,
with the fixed historical tie-breaking rule\;
}
Set $\omega_{i+1}(x,a)=\omega_i(x,a)+\widetilde\pi_{t,i}(a)$ for every $a$\;
}
Return $\pi_t(\cdot\mid x)=\widetilde\pi_{t,t}$\;
\end{algorithm}

\subsection{Cost Confidence Sets}
Let $B_1^{\cC}=B_1^{\cF}=0$. For $t\ge2$, we use the
known offline estimation budgets
\begin{equation}\label{eq:oracle_budgets}
B_t^{\cK}=\Estoff(\cK,t,\delta/(4T)),\qquad
\cK\in\{\cF,\cC\}.
\end{equation}
We test the full cost class against the current fit on all past
observations, using one convex squared-error constraint:
\begin{equation}\label{eq:spps_versions}
\begin{aligned}
\cS_1&=\cC,\\
\cS_t&=\left\{c\in\cC:
\sum_{i<t}(c(Z_i)-\hat c^t(Z_i))^2\le B_t^{\cC}\right\}.
\end{aligned}
\end{equation}
Thus $\cS_t$ contains the candidates consistent with the current
regression estimate. The test is convex in $c$ and uses the full data
prefix. Each round constructs its own confidence set. For replay,
we archive its defining fit, data prefix, and estimation budget.

The oracle guarantee ensures that the true cost function passes all
these tests simultaneously, with high probability. The current test
also gives a uniform empirical pair bound:
\begin{equation}\label{eq:spps_pair_certificate}
\sum_{i<t}(c(Z_i)-c'(Z_i))^2\le4B_t^{\cC},
\qquad c,c'\in\cS_t.
\end{equation}
We establish these statements in Lemma~\ref{lemma:cost_confidence}. They turn
the offline estimation guarantee into a common error budget for all
cost models considered by the decision rule.

\paragraph{Choosing the safe action.}
If a required confidence set is empty at round $t$, we choose the safe
action $a_0$ in this and all remaining rounds, setting
$\pi_s(\cdot\mid x)=\pi_0=\delta_{a_0}$ for every $s\ge t$ and context
$x$. This instruction is recorded for replay. The rule keeps the policy well defined
on histories where the regression confidence guarantees fail.
Lemma~\ref{lemma:cost_confidence} shows that the instruction is never needed on the
cost-oracle confidence event. The algorithm starts with robust policy
optimization at round one; it has no initial exploration phase.

For a nonempty confidence set, define
\begin{equation}\label{eq:cost_width}
\begin{aligned}
V_t^\delta(x,p)&=\sup_{c\in\cS_t}c(x,p),\\
w_t(x,a)&=\sup_{c,c'\in\cS_t}|c(x,a)-c'(x,a)|,\\
\Pi_t^\delta(x)&=\{p\in\cP(\cA):V_t^\delta(x,p)\le\tau\}.
\end{aligned}
\end{equation}
We write $w_t(x,p)=\int w_t(x,a)\,p(da)$. The safe action has upper
cost $c_0$ and width zero, so it is always robustly feasible.

\paragraph{Role of the cost confidence sets.}
The convex confidence sets serve both estimation and decision-making.
Truth retention certifies the robust feasible region, while the
empirical pair bound supplies the input to our cumulative-width
analysis. Convexity and covering entropy then provide a geometric
representation in which this budget controls uncertainty along the
sampled decisions. The width bonus uses the resulting bound to
compensate for the reward loss of conservative policies in the
optimism and regret proofs.
The cost support queries maximize linear functionals over $\cS_t$,
and the finite-action policy problem has a linear objective and a
convex robust-cost constraint. Both can therefore be formulated as
convex programs. With a tractable finite-dimensional affine class
representation, each support problem has a single empirical
second-order cone constraint. We give these formulations and their
computational requirements in Appendix~\ref{app:computation}.

\begin{figure}[tb]
\centering
\resizebox{0.8\linewidth}{!}{%
\begin{tikzpicture}[
  every node/.style={font=\scriptsize,align=center},
  box/.style={draw=black!65,rounded corners=1.5pt,
    text width=2.1cm,minimum width=2.3cm,minimum height=.72cm,inner sep=3pt},
  flow/.style={->,>=stealth,line width=.55pt,draw=black!75}]
\node[box,fill=black!3] (fit) at (1.2,1.65)
  {Past cost data\\Offline fit $\hat c^t$};
\node[box,fill=blue!5] (confidence) at (4,1.65)
  {Cost confidence set\\$\cS_t\subseteq\cC$};
\node[box,fill=blue!5] (bounds) at (6.8,1.65)
  {Cost support queries\\$V_t^\delta$, $w_t$};
\node[box,fill=black!3] (data) at (1.2,.35)
  {Observed feedback\\$(Z_t,R_t,Y_t)$};
\node[box,fill=black!3] (sample) at (4,.35)
  {Action sampling\\$A_t\sim\pi_t(\cdot\mid X_t)$};
\node[box,fill=blue!5] (policy) at (6.8,.35)
  {Feasible UCCB policy\\$\pi_t\in\Pi_t^\delta$};
\draw[flow] (fit) -- (confidence);
\draw[flow] (confidence) -- (bounds);
\draw[flow] (bounds) -- (policy);
\draw[flow] (policy) -- (sample);
\draw[flow] (sample) -- (data);
\draw[flow] (data.west) -- ++(-.45,0) |- (fit.west)
  node[pos=.45,left,font=\scriptsize] {next round};
\end{tikzpicture}%
}
\caption{CS-UCCB with selected-action cost feedback. The current offline fit
defines a convex confidence set. Its upper costs
certify feasible policies, and its widths enter the UCCB objective.
The selected action supplies the observations for the next update.}
\label{fig:spps-overview}
\end{figure}
Figure~\ref{fig:spps-overview} shows how regression, confidence-set
updates, and safe decisions interact.

\subsection{Constrained Counterfactual Simulation}
For finite actions, we adapt UCCB \citep{xu2020upper} with the
cumulative policy weights and the deterministic reward bonus
\begin{equation}\label{eq:finite_bonus}
\omega_t(x,a)=\sum_{i<t}\pi_i(a\mid x),\qquad
q_t(x,a)=\frac1{1+\omega_t(x,a)}.
\end{equation}
For a distribution $p$, let $q_t(x,p)=\sum_a p(a)q_t(x,a)$.
The unit pseudocount regularizes the denominator and does not represent
an executed action or a regression observation. We set
$\beta_t=\sqrt{tH_t/K}$, with the known reward budget $H_t$ specified
in~\eqref{eq:reward_budget}, and choose
\begin{equation}\label{equation:ucb_policy_space}
\pi_t(\cdot\mid x)\in\argmax_{p\in\Pi_t^\delta(x)}
\{\hat f^t(x,p)+\beta_tq_t(x,p)+\alpha_rw_t(x,p)\}.
\end{equation}
The reward bonus favors actions with little historical policy mass
at the current context. The cost-width bonus compensates for
conservative reward loss \citep{pacchiano2024contextual}.
The robust constraint controls feasibility throughout this tradeoff.
The rule applies from the first round and can explore an uncertain
action by mixing it with $a_0$ while maintaining the policy-level
constraint.

To evaluate~\eqref{eq:finite_bonus} at a newly observed context, we
replay the archived policy rules at that context. The fits, confidence
sets, coefficients, and tie-breaking choices at historical stages are
kept fixed. Replay therefore reconstructs the actual historical
kernels without refitting an oracle or generating fictitious labels
(Lemma~\ref{lemma:equivalent_finite}). Computational formulations for
confidence-set queries, policy optimization, and replay appear in
Appendix~\ref{app:computation}; their cost depends on the available
class representation and optimization access.

\section{Theoretical Guarantees}\label{sec:theoretical_guarantees}
We first state the finite-action guarantee and then explain how
cost confidence sets and convex cost entropy control its cost term.
Fix $T\ge2$ and $\delta\in(0,1)$; set the analysis radius $r_T=1/T$.
For $\eta_{\cF}=1/(6T)$ and
$N_{\cF}=\cN(\cF,\eta_{\cF},\|\cdot\|_\infty)$, define
\begin{equation}\label{eq:reward_budget}
\begin{aligned}
H_t&=1+2B_t^{\cF}+2\log\frac{4TN_{\cF}}\delta
          +6(t-1)\eta_{\cF},\\
\bar H_T&=\max_{1\le t\le T}H_t,\qquad
\beta_t=\sqrt{tH_t/K}.
\end{aligned}
\end{equation}
A known upper bound on $N_{\cF}$ can replace the covering number.
The policy rule does not enumerate a reward covering net.

For the cost geometry, let
$\cD=\overline{\operatorname{conv}}(\cC-\cC)$, with closure in uniform
norm. Let $m_T$ denote the least dimension of a linear space that
approximates every $h\in\cD$ within $r_T$ uniformly. Convexity of
$\cC$ and~\eqref{eq:cost_entropy_assumption} imply
\begin{equation}\label{eq:entropy_dimension_bound}
m_T\le C_{p_{\cC}}D_{\cC}
 [\log(eA_{\cC}D_{\cC}T)]^{p_{\cC}-1}.
\end{equation}
This space is used only in the analysis. The algorithm requires
neither its basis nor a kernel representation. The main argument is
in Appendix~\ref{app:cost-property}, with the supporting entropy and
geometric proofs in Appendix~\ref{app:omitted-spps-proofs}.

Set $\operatorname{diam}_\infty(\cC)=\sup_{c,c'\in\cC}\|c-c'\|_\infty\le1$ and
\[
\begin{aligned}
B_T^{\rm pair}&=4B_T^{\cC},\\
R_T^2&=4m_T+2B_T^{\rm pair}+2/T.
\end{aligned}
\]
For $m_T\ge1$, define
\begin{equation}\label{eq:cost_width_total_bound}
\begin{aligned}
L_T^{\cC}&=4m_T\log(1+T/m_T)+2\log(4/\delta),\\
\mathcal W_T&=\min\{T\operatorname{diam}_\infty(\cC),\;1+\sqrt{TR_T^2L_T^{\cC}}\}.
\end{aligned}
\end{equation}
For $m_T=0$, set $\mathcal W_T=\min\{T\operatorname{diam}_\infty(\cC),1\}$.
The diameter term retains information about small-amplitude cost
classes that a dimension-only estimate can obscure.

\begin{theorem}\label{thm:main_guarantee}
Under Assumptions~\ref{ass:safe_action}--\ref{ass:oracle},
Algorithm~\ref{alg:UCCB_pseudo_code}, with the parameters above and
$\alpha_r=1/\Delta$, satisfies, with probability at least $1-\delta$,
$c^*(x,\pi_t)\le\tau$ for every $x$ and $1\le t\le T$, and
\begin{equation}\label{eq:finite_main_bound}
\begin{aligned}
\operatorname{Reg}(T)
\le\sqrt{KT\bar H_T}
 \left[1+\frac{2\log(1+T/K)}{\log2}\right]
 +\frac{\mathcal W_T}{\Delta}.
\end{aligned}
\end{equation}
\end{theorem}
Substituting the setup quantities, with $\rho=\delta/(4T)$, gives
\[
\begin{aligned}
&\operatorname{Reg}(T)=\widetilde{\cO}\Bigl(
\sqrt{KT[D_{\cF}+\Estoff(\cF,T,\rho)]}\\
&\quad+\frac{\sqrt{TD_{\cC}[D_{\cC}+\Estoff(\cC,T,\rho)]}}{\tau-c_0}
\Bigr)=\widetilde{\cO}(\sqrt T).
\end{aligned}
\]
The last equality holds for polylogarithmic offline errors, fixed $K$
and entropy parameters, and a fixed positive safety gap. We specify
the suppressed logarithmic factors in Appendix~\ref{app:proofs1},
give nonlinear examples in
Appendix~\ref{app:parametric-entropy} and compare with
\citet{pacchiano2024contextual} in Appendix~\ref{app:pacchiano-comparison}.

\begin{remark}\label{rem:nonconvex-costs}
For a nonconvex cost class $\cC$ containing $c^*$ and satisfying
the safe-action identity, we may construct confidence sets over
$\cC_{\rm hull}=\overline{\operatorname{conv}}(\cC)$, taking closure in uniform
norm. If this hull satisfies the entropy bound in
Assumption~\ref{ass:cost_entropy} with parameters
$(D_{\rm hull},A_{\rm hull},p_{\rm hull})$, the same analysis gives
stage-wise feasibility and
\[
\begin{aligned}
&\operatorname{Reg}(T)=\widetilde{\cO}\Bigl(
\sqrt{KT[D_{\cF}+\Estoff(\cF,T,\rho)]}\\
&\quad+\frac{\sqrt{TD_{\rm hull}
[D_{\rm hull}+\Estoff(\cC,T,\rho)]}}{\tau-c_0}\Bigr),
\end{aligned}
\]
where $\rho=\delta/(4T)$. The cost oracle may still fit $\cC$:
confidence-set construction uses its error relative to $c^*\in\cC$, and does not
require its output to belong to the working class.
\end{remark}
Without convexity or another structural condition, logarithmic
covering entropy and even zero empirical offline regression error do
not ensure a uniform sublinear regret guarantee, as we show in
Appendix~\ref{app:entropy-only-counterexample}.

\begin{lemma}\label{lemma:cost_confidence}
The cost confidence sets are convex subsets of $\cC$.
With probability at least $1-\delta/4$, simultaneously for all $t\le T$,
$c^*\in\cS_t$ and the empty-set check is never triggered. The upper
confidence bounds $V_t^\delta$ dominate true policy costs, and
$\Pi_t^\delta(x)$ contains $\pi_0$ and only truly feasible distributions.
Moreover, for every $c,c'\in\cS_t$,
\[
\sum_{i<t}(c(Z_i)-c'(Z_i))^2\le4B_t^{\cC}.
\]
\end{lemma}
We prove these statements in Appendix~\ref{app:cost-property}.
The empirical pair bound is deterministic for models in $\cS_t$;
its confidence event is needed only to retain $c^*$.

\begin{property}\label{prop:sequential_width}
With probability at least $1-\delta/2$, the cost confidence sets satisfy
\begin{equation}\label{eq:cumulative_cost_width}
\sum_{t=1}^T\EE_{X,A\sim\pi_t}w_t(X,A)\le\mathcal W_T
\end{equation}
for the actual history-dependent policy sequence. The result holds
for any such sampling rule using these confidence sets.
\end{property}
The width argument uses only the class geometry and the empirical
pair certificate. It applies to any history-measurable confidence
sets satisfying that certificate. The expectations track the actual
sampling policies, linking cost learning to exploration under
selected-action feedback.

\paragraph{Proof sketch.}
We first use the offline cost guarantee to retain $c^*$ in every
convex confidence set. For any $c,c'\in\cS_t$, comparison with the
same current fit gives
\[
\sum_{i<t}(c(Z_i)-c'(Z_i))^2\le4B_t^{\cC}\le B_T^{\rm pair}.
\]
The convex entropy bound then lets us approximate all cost differences
by a common finite-dimensional space. Convexity keeps the symmetric
convex hull of cost differences inside the difference class. If too
many differences remain separated from the span of those already
chosen at the target resolution, that hull would require more covering
centers than the entropy bound permits.
John's ellipsoid theorem supplies features of norm at most one and
coefficients of norm at most
$2\sqrt{m_T}$. The empirical pair budget bounds their quadratic norm
under the matrix of actual sampled features. Because this matrix
accumulates the selected pairs, its next increment is generated by
the policy whose width enters the regret bound. The elliptical-potential
inequality and a martingale comparison inequality therefore yield
Property~\ref{prop:sequential_width}.

For rewards, offline estimation and a uniform covering argument give,
simultaneously for all policies,
\[
\left|\EE_X[\hat f^t(X,\pi)-f^*(X,\pi)]\right|
\le\beta_t\EE_Xq_t(X,\pi)+\frac{K\beta_t}{4t}.
\]
In the analysis, we mix the true optimal policy with the safe action
until it is robustly feasible. Its reward loss is bounded by $\alpha_r$
times its own cost width. The policy optimization then gives
\begin{equation}\label{eq:optimism_main}
\begin{aligned}
\EE_X f^*(X,\pi^*)\le{}&\EE_X[\hat f^t(X,\pi_t)
 +\beta_tq_t(X,\pi_t)\\[-2pt]
&+\alpha_rw_t(X,\pi_t)]+\frac{K\beta_t}{4t}.
\end{aligned}
\end{equation}
Applying the reward-confidence bound once more leaves the actual-policy cost
width in~\eqref{eq:cumulative_cost_width}. The comparator's role ends
with the optimism argument; the remaining uncertainty is charged to
the policies whose observations enter the next confidence-set construction. Finally,
the deterministic potential bound in Lemma~\ref{lemma:potential_lemma},
$\sum_t\EE_Xq_t(X,\pi_t)\le K\log(1+T/K)/\log2$,
and $\sum_{t\le T}t^{-1/2}\le2\sqrt T$ give
Theorem~\ref{thm:main_guarantee}. The full argument appears in
Appendix~\ref{app:proofs1}.

\section{Extension to Infinite Actions}\label{sec:extension}
To extend CS-UCCB to infinite action spaces, we replace the finite-action
reward bonus with a regularized nonlinear design bonus. The cost update
continues to use the selected action's observed cost. The proofs are
in Appendix~\ref{app:infinite-proofs}, and a nonlinear continuous-action
example appears in Appendix~\ref{app:rank-two-example}.

\begin{definition}\label{def:policy_div}
For a context $x$, action $a$, and distribution $\rho$ on $\cA$, define
\[
\cV_x(a\|\rho):=
\sup_{f,f'\in\cF}
\frac{(f(x,a)-f'(x,a))^2}
{\int_{\cA}(f(x,a')-f'(x,a'))^2\,d\rho(a')}.
\]
We adopt the conventions $0/0=0$ and $b/0=+\infty$ for $b>0$.
Write $\cV_x(\rho)=\sup_a\cV_x(a\|\rho)$, and let
$\rho_x^*$ minimize this quantity whenever a measurable minimizing
design exists.
\end{definition}
We choose this design using the reward class~\citep{agarwal2024non}.
For a known reference mass $d_x\ge1$,
we set
\begin{equation}\label{eq:infinite_bonus}
\begin{aligned}
\bar\rho_{t-1}(\cdot|x)
&=\frac{d_x\rho_x^*+\sum_{i<t}\pi_i(\cdot|x)}{d_x+t-1},\\
q_t^\infty(x,a)
&=\frac{\cV_x(a\|\bar\rho_{t-1})}{d_x+t-1}.
\end{aligned}
\end{equation}
The reference regularizes the denominator from round one. We do not
execute the reference design or add any reference observations to either
regression dataset. For a tabular reward class, the uniform design and
$d_x=K$ recover $q_t(x,a)=1/(1+\sum_{i<t}\pi_i(a|x))$.

\begin{assumption}
\label{ass:policy_div2}
The optimal designs $\rho_x^*$ are measurable and satisfy
$\cV_x(\rho_x^*)<\infty$. There are known measurable $d_x\ge1$ and
$\cM_x\ge\max\{1,d_x,\cV_x(\rho_x^*)\}$, a known
$D_0\ge\EE_Xd_X$ and $\bar{\cM}\ge\EE_X\cM_X$, and fixed constants
$C_V>0$, $p_V\ge0$, such that, for every $x\in\cX$, every sequence of policy kernels, and $T\ge1$,
\[
\sum_{t=1}^T
\EE_{a\sim\pi_t(\cdot|x)}q_t^\infty(x,a)
\le C_V\cM_x(1+\log T)^{p_V+1}.
\]
The constants are independent of the horizon and confidence level.
\end{assumption}
This assumption controls the cumulative bonus at the actions used by
the policy sequence. We verify it for nonlinear rewards with two
independent action features in Appendix~\ref{app:rank-two-example}.

Using the same robust feasible set $\Pi_t^\delta(x)$ and cost width
$w_t$ as in the finite-action algorithm, we select
\begin{equation}\label{eq:infinite_policy}
\pi_t(\cdot\mid x)\in\argmax_{p\in\Pi_t^\delta(x)}
\{\hat f^t(x,p)+\beta_tq_t^\infty(x,p)+\alpha_rw_t(x,p)\}.
\end{equation}
Here $q_t^\infty(x,\pi)$ and $w_t(x,\pi)$ denote action expectations.
We retain the same measurable tie-breaking rule and choose the safe
action $a_0$ in this and all remaining rounds if a confidence set is empty.
Counterfactual replay reconstructs each historical policy at the current
context using its archived predictor and confidence set.
Algorithm~\ref{alg:UCCB_infinite_action_pseudo_code} gives the full
procedure, and Lemma~\ref{lemma:equivalent_infinite} establishes replay
equivalence.

\begin{theorem}\label{thm:main_infinite_guarantee}
Suppose Assumptions~\ref{ass:safe_action}--\ref{ass:oracle}
and~\ref{ass:policy_div2} hold. Let $T\ge2$,
$\delta\in(0,1)$, and
$\alpha_r=1/(\tau-c_0)$. With $H_t,\bar H_T$ defined in
\eqref{eq:reward_budget}, use
\begin{equation}\label{eq:beta_infinite_entropy}
\beta_t=\sqrt{\frac{t(H_t+D_0)}{\bar{\cM}}},\qquad 1\le t\le T.
\end{equation}
Let $\mathcal W_T$ be the cost-width bound in
\eqref{eq:cost_width_total_bound} and set
$P_T=C_V\EE_X\cM_X(1+\log T)^{p_V+1}$.
With probability at least $1-\delta$, every selected policy is
feasible from round one, and
\[
\operatorname{Reg}(T)
\le \left(1+\frac{2P_T}{\bar{\cM}}\right)
\sqrt{\bar{\cM}T(\bar H_T+D_0)}
   +\alpha_r\mathcal W_T.
\]
\end{theorem}
Substituting $P_T$, $\bar H_T$, and $\mathcal W_T$, with
$\rho=\delta/(4T)$, gives
\[
\begin{aligned}
&\operatorname{Reg}(T)=\widetilde{\cO}\Bigl((1+C_V)\\
&\quad\times\sqrt{\bar{\cM}T[D_0+D_{\cF}+\Estoff(\cF,T,\rho)]}\\
&\quad+\frac{\sqrt{TD_{\cC}[D_{\cC}+\Estoff(\cC,T,\rho)]}}{\tau-c_0}
\Bigr)=\widetilde{\cO}(\sqrt T).
\end{aligned}
\]
The last equality holds for polylogarithmic offline errors, fixed
entropy and design parameters, and a fixed positive safety gap.
We give the logarithmic factors and recover Theorem~\ref{thm:main_guarantee}
for tabular rewards in Appendix~\ref{app:infinite-proofs}.

\section{Discussion}
CS-UCCB addresses contextual bandits with stage-wise constraints and
selected-action reward and cost feedback. Offline cost estimates
define confidence sets that certify
feasibility and bound the cost uncertainty in the regret analysis.
For convex cost classes with polylogarithmic covering entropy, this
analysis yields $\widetilde{\mathcal O}(\sqrt T)$ regret under the stated
reward and oracle conditions. The geometry needed for this bound is
derived directly from the cost class, and each confidence-set query
uses one convex regression constraint. Our bound has no explicit
eluder-dimension term; the cost assumptions themselves imply a
polylogarithmic eluder dimension. The example in Appendix~\ref{app:pacchiano-comparison}
gives a logarithmic cost improvement over the published bound at its
stated resolution, together with a polynomial reward improvement.
An open question is which weaker conditions on nonconvex cost classes
permit comparable guarantees without paying for a much larger convex
hull. The counterexample in Appendix~\ref{app:entropy-only-counterexample}
shows that small covering entropy alone is insufficient.

\label{migration:main_end}

\subsection*{AI Use Statement}
The authors used AI tools to assist with checking
mathematical arguments and illustrative constructions, revising proofs, and polishing the manuscript.

\bibliographystyle{plainnat}
\bibliography{refs}

\clearpage
\appendix
\phantomsection\label{migration:appendix_start}
\section{Detailed Related Work}\label{app:related-work}
\paragraph{Constrained contextual bandits.}
Stage-wise constraints have been studied for linear payoffs
\citep{amani2019linear,moradipari2020stage,pacchiano2021stochastic,liu2021efficient}.
For nonlinear models, \citet{pacchiano2024contextual} use regression
confidence sets and eluder dimension to control uncertainty in rewards
and costs. We use the same selected-action reward and cost feedback,
with the additional structure of a convex cost class and its entropy
profile. Our reward class need not be convex. These model assumptions,
rather than a change in the observation protocol, distinguish the
results. We give a quantitative comparison with the published finite-class
theorem in Appendix~\ref{app:pacchiano-comparison}.

Regression-oracle methods also address packing and covering constraints
\citep{slivkins2023contextual} and contextual knapsack budgets
\citep{han2023optimal}. Bandits with non-monotonic resource use and
online optimization with stochastic or long-term constraints provide
related cumulative formulations
\citep{kumar2022nonmonotonicresourceutilizationbandits,yu2017online,liakopoulos2019cautious}.
These formulations impose constraints over a horizon, whereas our
feasibility requirement applies to each
context-dependent randomized decision at every round. Conservative
bandits compare cumulative performance with a baseline
\citep{wu2016conservative,kazerouni2017conservative,deb2025conservative};
this differs from learning a separate cost response subject to a
per-round threshold.

\paragraph{Kernel bandits and safe Gaussian-process optimization.}
GP-UCB relates regret to maximal information gain
\citep{srinivas2010gaussian}, and \citet{chowdhury2017kernelized}
develop confidence bounds and algorithms for kernelized bandits with
rewards in a reproducing kernel Hilbert space.
SafeOpt \citep{sui2015safe} and StageOpt \citep{sui2018stagewise}
use Gaussian-process confidence bounds to expand a safe region and
certify sampled decisions; \citet{amani2020safegp} give regret bounds
for safe Gaussian-process bandit optimization. These approaches use a
specified kernel or Gaussian-process model. Our assumptions concern
the convex cost class and its uniform entropy, from which we derive
an approximation geometry for analysis. Our safety requirement applies
to the conditional expected cost of each randomized decision.

\paragraph{Regression-oracle contextual bandits.}
RegCB \citep{foster2018practical} uses least-squares confidence bounds,
while SquareCB \citep{foster2020beyond} reduces realizable contextual
bandits to online regression. Offline-oracle approaches include
inverse-gap weighting \citep{simchi2020bypassing} and UCB-style methods
\citep{xu2020upper,pmlr-v235-levy24a}. Other developments study tabular
MDPs with low oracle complexity \citep{qian2024offline}, functional
regression, finite hypothesis classes, and related sequential settings
\citep{hu2025contextual,NEURIPS2024_4b6908c6,krishnamurthy2021adapting,wei2021non,sen2021top}.
\citet{pmlr-v336-qin26a} develop an offline-oracle framework using
exploitative $F$-design and a decision--offline-estimation complexity
measure. We add an unknown cost model and certify stage-wise feasible
policies from the selected actions' outcomes. For infinite actions, we
state explicitly the additional cumulative reward-design condition.
Small-ball methods also yield complexity-dependent error rates
for regularized regression \citep{lecue2017regularization}.

Our use of offline regression concerns statistical estimation.
Evaluating convex confidence-set supports and solving the constrained
policy problem are separate computational operations. The version-space
formulation avoids enumerating a covering net. We give its optimization
formulations and access requirements in
Appendix~\ref{app:computation}.

\paragraph{Eluder dimension and disagreement complexity.}
\citet{russo2013eluder} introduce eluder dimension for optimistic
exploration. For reproducing kernel Hilbert spaces,
\citet{huang2021information} establish quantitative relations between
eluder dimension and information gain.
\citet{foster2021instance} develop a disagreement-based
view of instance-dependent contextual-bandit complexity and implement
AdaCB with offline regression. Broader treatments cover neural-network
learning complexity \citep{anthony2009neural} and structural relations
among eluder and disagreement-type dimensions
\citep{li2022understanding,hanneke2024star}. The difference between a
sequential complexity and a covering entropy motivates our analysis.
Convexity alone does not make these quantities equal: we give an explicit
convex-class separation in Appendix~\ref{app:pacchiano-comparison}.
Our proof bounds a cumulative prediction width from the full entropy
profile and the empirical certificates supplied by the cost confidence sets.

\paragraph{Entropy, approximation widths, and misspecification.}
Relations between entropy and finite-dimensional approximation are
classical in approximation theory: \citet{carl1981entropy} compares
entropy and approximation numbers for Banach-space operators, and
\citet{pinkus1985widths} develops the theory of approximation widths.
In bandits, \citet{lattimore2020goodfeatures} study the interaction
between linear approximation error, feature geometry, and regret.
Our argument constructs an approximating space from the function
class and accounts explicitly for its uniform approximation error.

\section{Useful Math Tools}\label{app:math-tools}
We collect the tools used in the cost-width analysis and the regret
comparison. The first lemma controls bounded martingale differences
in Appendix~\ref{app:pacchiano-comparison}. The remaining two compare
conditional means with observed sums and bound the elliptical
potential in Appendix~\ref{app:cost-property}. Let $(\cH_t)$ be a
filtration, and write $\EE_{t-1}$ for conditional expectation given
$\cH_{t-1}$.

\begin{lemma}\label{Azuma-Hoeffding}
If $X_t$ is adapted and $|X_t|\le R$, then for any fixed $n$ and $\delta\in(0,1)$, with probability at least $1-\delta$,
\[
\left|\sum_{t=1}^n(X_t-\EE_{t-1}X_t)\right|
\le R\sqrt{2n\log(2/\delta)}.
\]
\end{lemma}
This is the bounded-increment martingale inequality; see, for example, \citet{sason2011refined}. We include the argument with the constants used here.
\begin{proof}[Proof of Lemma~\ref{Azuma-Hoeffding}]
The case $R=0$ is immediate. Otherwise, put
$\mu_t=\EE_{t-1}X_t$ and $D_t=X_t-\mu_t$.
For $\lambda\in\RR$, define the conditional log-moment function
\[
\psi_t(\lambda)=\log\EE_{t-1}e^{\lambda X_t}-\lambda\mu_t.
\]
Boundedness permits differentiation under the conditional expectation.
We have $\psi_t(0)=\psi_t'(0)=0$, and
\[
\psi_t''(\lambda)
=\frac{\EE_{t-1}[X_t^2e^{\lambda X_t}]}{\EE_{t-1}e^{\lambda X_t}}
 -\left(\frac{\EE_{t-1}[X_te^{\lambda X_t}]}{\EE_{t-1}e^{\lambda X_t}}\right)^2
\le R^2.
\]
Integrating twice gives $\psi_t(\lambda)\le\lambda^2R^2/2$ for
both signs of $\lambda$. Consequently,
\[
\EE_{t-1}e^{\lambda D_t}\le e^{\lambda^2R^2/2}.
\]
Iterating conditional expectations, we obtain
$\EE e^{\lambda\sum_{t=1}^nD_t}\le e^{n\lambda^2R^2/2}$.
For $u>0$, Markov's inequality with $\lambda=u/(nR^2)$ yields
\[
\PP\!\left(\sum_{t=1}^nD_t\ge u\right)
\le e^{-\lambda u+n\lambda^2R^2/2}
=e^{-u^2/(2nR^2)}.
\]
Applying the same argument to $-D_t$ and adding the two tail
probabilities gives $2e^{-u^2/(2nR^2)}$. Substituting
$u=R\sqrt{2n\log(2/\delta)}$ proves the claim.
\end{proof}

\begin{lemma}
\label{lemma:predictable-nonnegative}
Let $0\le L_t\le1$ be adapted and put
$\mu_t=\EE[L_t\mid\cH_{t-1}]$. For every $\zeta\in(0,1)$, with
probability at least $1-\zeta$, simultaneously for all $n\le T$,
\begin{equation}\label{eq:predictable-nonnegative}
\sum_{t=1}^n\mu_t
\le2\sum_{t=1}^n L_t+2\log(1/\zeta).
\end{equation}
The conditional means may depend on the entire past.
\end{lemma}
\begin{proof}[Proof of Lemma~\ref{lemma:predictable-nonnegative}]
Set $a=1-e^{-1}>1/2$. For $u\in[0,1]$, the chord bound for the convex
function $e^{-u}$ gives
\[
e^{-u}\le1-au.
\]
Consequently,
\[
\EE[e^{-L_t}\mid\cH_{t-1}]
\le1-a\mu_t\le e^{-a\mu_t}.
\]
We define the nonnegative process
\[
M_n=\exp\left\{a\sum_{t=1}^n\mu_t-\sum_{t=1}^nL_t\right\},
\qquad M_0=1.
\]
The preceding conditional inequality implies
$\EE[M_n\mid\cH_{n-1}]\le M_{n-1}$.
Let $\nu$ be the first $n\le T$ with $M_n\ge1/\zeta$, and set
$\nu=T$ if there is no such index. We obtain
$\EE M_\nu\le1$ by the optional stopping theorem for bounded stopping times. Since $M_\nu$ is
nonnegative,
\[
\PP\{\max_{n\le T}M_n\ge1/\zeta\}
\le\zeta\EE M_\nu\le\zeta.
\]
Outside this event, taking logarithms yields
$a\sum_{t\le n}\mu_t\le\sum_{t\le n}L_t+\log(1/\zeta)$ for every
$n\le T$. Both terms on the right are nonnegative, so $a^{-1}<2$
implies~\eqref{eq:predictable-nonnegative}. This argument uses no
independence assumption on the observations.
\end{proof}

\begin{lemma}
\label{lemma:elliptical-potential}
Let $m\ge1$, let $\phi_t\in\RR^m$ satisfy $\|\phi_t\|_2\le1$, and
define
\[
\Sigma_t=I_m+\sum_{i<t}\phi_i\phi_i^\top,
\qquad L_t=\phi_t^\top \Sigma_t^{-1}\phi_t.
\]
For every sequence, including an adaptively selected sequence,
\begin{equation}\label{eq:elliptical-potential}
0\le L_t\le1,
\qquad
\sum_{t=1}^n L_t\le2m\log(1+n/m)\quad(n\ge1).
\end{equation}
\end{lemma}
\begin{proof}[Proof of Lemma~\ref{lemma:elliptical-potential}]
The matrix inequality $\Sigma_t\succeq I_m$ gives
$0\le L_t\le\|\phi_t\|_2^2\le1$.
By the matrix determinant lemma for a rank-one update,
\[
\det \Sigma_{t+1}=\det \Sigma_t(1+L_t).
\]
For $u\in[0,1]$, the derivative of
$\log(1+u)-u/2$ is nonnegative. It follows that
\[
\sum_{t=1}^nL_t
\le2\sum_{t=1}^n\log(1+L_t)
=2\log\det \Sigma_{n+1}.
\]
We apply the arithmetic--geometric mean inequality to the $m$
positive eigenvalues of $\Sigma_{n+1}$. Since
$\operatorname{tr}(\Sigma_{n+1})=m+\sum_{t\le n}\|\phi_t\|_2^2\le m+n$,
\[
\log\det \Sigma_{n+1}
\le m\log\left(\frac{\operatorname{tr}(\Sigma_{n+1})}{m}\right)
\le m\log(1+n/m).
\]
Combining the two bounds proves~\eqref{eq:elliptical-potential}.
All steps are deterministic.
\end{proof}

\section{Proofs of the Cost Confidence Guarantees}
\label{app:cost-property}

We first prove Lemma~\ref{lemma:cost_confidence}: the convex regression confidence
sets retain the true cost function, and their defining regression
constraints control every pair of candidate models. We then prove
Property~\ref{prop:sequential_width} by converting this empirical
bound into cumulative-width control. Convexity and covering entropy
provide a fixed finite-dimensional approximation; an elliptical
potential argument controls its uncertainty along the selected pairs.
We state auxiliary results where they enter these arguments and give
their detailed proofs in Appendix~\ref{app:omitted-spps-proofs}.
Finally, we retain the approximation radius in a refinement used in
Appendix~\ref{app:pacchiano-comparison}. Throughout,
$Z_t=(X_t,A_t)$, and the cost oracle uses only
$\{(Z_i,Y_i):i<t\}$.

\begin{proof}[Proof of Lemma~\ref{lemma:cost_confidence}]
The cost confidence sets are
\begin{equation}\label{eq:cost-confidence-sets}
\cS_1=\cC,\qquad
\cS_t=\left\{c\in\cC:
\sum_{i<t}(c(Z_i)-\hat c^t(Z_i))^2\le B_t^{\cC}\right\}
\quad(t\ge2).
\end{equation}
At round $t$, we test the cost class against the current fit on all
past observations. The cost fit need not belong to $\cC$. Since point
evaluations are linear in $c$ and squared loss is convex, this test
defines a convex subset of $\cC$. Each confidence set is formed from
its own fit and data prefix and is archived for counterfactual replay.
All uniform norms below are taken on $\cZ=\cX\times\cA$.

If an implementation stops fitting after an empty-set check instructs
choosing $a_0$ in all remaining rounds, we define virtual cost-oracle
calls on its actual recorded data at every $t\ge2$. Selected costs
continue to be observed when $a_0$ is chosen, so these calls require
no observations at unselected actions. Each actual call is coupled
with its virtual counterpart, using the same random seed when
necessary. The oracle guarantee applies to this adaptive sequence
under the original interaction law, with failure probability at most
$\delta/(4T)$ per call. A union bound over $T-1$ calls gives
\[
\mathcal E_C=
\bigcap_{t=2}^T
\left\{\sum_{i<t}(\hat c^t(Z_i)-c^*(Z_i))^2\le B_t^{\cC}\right\},
\qquad
\PP(\mathcal E_C)\ge1-\delta/4.
\]
We do not condition the oracle guarantee on the outcomes of earlier
empty-set checks.

\begin{lemma}\label{lemma:cost-retention}
On $\mathcal E_C$, simultaneously for $1\le t\le T$,
$c^*\in\cS_t$ and $\cS_t\ne\varnothing$. For every $t\ge2$ and
$c\in\cS_t$,
\begin{equation}\label{eq:cost-empirical-radius}
\sum_{i<t}(c(Z_i)-c^*(Z_i))^2\le4B_t^{\cC}.
\end{equation}
\end{lemma}

The same regression constraint gives a pairwise bound without any
probability event.

\begin{lemma}\label{lemma:cost-pair-budget}
For every $t\ge2$ and every $c,c'\in\cS_t$,
\begin{equation}\label{eq:cost-pair-budget}
\sum_{i<t}(c(Z_i)-c'(Z_i))^2\le4B_t^{\cC}.
\end{equation}
At $t=1$, the empirical sum is empty and equals zero.
\end{lemma}

Lemma~\ref{lemma:cost-retention} ensures that no empty-set check
instructs a permanent switch to $a_0$ on $\mathcal E_C$.
Lemma~\ref{lemma:cost-pair-budget} supplies the empirical pair bound
in the main lemma. For any context $x$ and action distribution $p$,
truth retention gives
\[
c^*(x,p)\le\sup_{c\in\cS_t}c(x,p)=V_t^\delta(x,p).
\]
Every member of $\cC$ satisfies the safe-action identity. Thus
$V_t^\delta(x,\pi_0)=c_0<\tau$, so
$\pi_0\in\Pi_t^\delta(x)$, and every
$p\in\Pi_t^\delta(x)$ obeys
$c^*(x,p)\le V_t^\delta(x,p)\le\tau$.
These conclusions hold simultaneously on $\mathcal E_C$;
no additional union bound over models, contexts, or action
distributions is required. This proves Lemma~\ref{lemma:cost_confidence}.
\end{proof}

\begin{proof}[Proof of Property~\ref{prop:sequential_width}]
We use Lemma~\ref{lemma:cost-pair-budget} to control each difference
between models in the current confidence set. The following
geometric construction converts that empirical control into a
pointwise width bound without imposing a linear model on the costs.

Let $\mathbb B$ be the Banach space of bounded measurable real
functions on $\cZ$, with the uniform norm. Define
\[
\cD=\overline{\operatorname{absconv}(\cC-\cC)}
     =\overline{\cC-\cC},
\]
where the closure is in $\mathbb B$. The equality holds because
$\cC$ is convex: its difference set is convex, symmetric, and contains
zero. Since each cost takes values in $[0,1]$, every $h\in\cD$ obeys
$\|h\|_\infty\le1$.

For $0<u\le1$, take a proper $(u/2)$-net of $\cC$. Differences of its
centers cover $\cC-\cC$ within distance $u$. A finite union of closed
$u$-balls is closed, so the same balls cover $\cD$. The entropy
assumption therefore implies
\begin{equation}\label{eq:difference-entropy}
\log\mathcal N(\cD,u,\|\cdot\|_\infty)
\le2D_{\cC}\left[\log\frac{2A_{\cC}}{u}\right]^{p_{\cC}}.
\end{equation}
This step uses the entropy of the convex cost class, not the entropy
of an unrelated feature space.

\begin{lemma}
\label{lemma:entropy-linear-span}
Suppose $D_{\cC}\ge1$, $A_{\cC}\ge e$, $p_{\cC}\ge1$, and the stated uniform
entropy bound holds for the convex class $\cC$. For each $0<r\le1$
there is a finite-dimensional linear space $V_r\subseteq\mathbb B$
such that
\begin{equation}\label{eq:span-approximation}
\sup_{h\in\cD}\inf_{v\in V_r}\|h-v\|_\infty\le r,
\qquad
\dim V_r\le C_{p_{\cC}}D_{\cC}
\left[\log\frac{eA_{\cC}D_{\cC}}{r}\right]^{p_{\cC}-1},
\end{equation}
where $C_{p_{\cC}}$ is a finite constant depending only on $p_{\cC}$.
If $p_{\cC}=1$, the affine dimension of $\cC$ is at most
$\lfloor D_{\cC}\rfloor$, so an exact span of this dimension is available.
\end{lemma}

Let $m(r)$ denote the smallest dimension among spaces satisfying the
first inequality in~\eqref{eq:span-approximation}. The preceding lemma
shows that this is a finite nonnegative integer. We may select one
such space once the class and $r$ are fixed; its construction uses no
observations.

\begin{lemma}
\label{lemma:john-features}
Let $0<r\le1$ and $m=m(r)\ge1$. There is a fixed measurable map
$\phi:\cZ\to\RR^m$ with $\|\phi(z)\|_2\le1$ such that, for every
$h\in\cD$, some $\theta_h\in\RR^m$ satisfies
\begin{equation}\label{eq:john-approximation}
\|h-\theta_h^\top\phi\|_\infty\le r,
\qquad
\|\theta_h\|_2\le(1+r)\sqrt m\le2\sqrt m.
\end{equation}
If $m=0$, then $\|h\|_\infty\le r$ for every $h\in\cD$.
\end{lemma}

We use $\phi$ only in the analysis. The algorithm constructs each
confidence set from the original cost class and the current regression
fit; it requires neither $\phi$ nor the ellipsoid used to construct it.
The treatment of uniform approximation error is related to
misspecified linear bandits \citep{lattimore2020goodfeatures}.

Fix $r\in(0,1]$ and first suppose $m=m(r)\ge1$. Using the fixed
features from Lemma~\ref{lemma:john-features}, define
\[
\Sigma_t=I_m+\sum_{i<t}\phi(Z_i)\phi(Z_i)^\top,\qquad
\ell_t(z)=\phi(z)^\top \Sigma_t^{-1}\phi(z),\qquad
L_t=\ell_t(Z_t).
\]
These quantities use only the actual selected pairs. For the purpose
of this proof, the virtual oracle calls above define $\cS_t$ on
every history, including histories after a permanent safe-action
switch. Set
\[
w_t(z)=
\begin{cases}
\sup_{c,c'\in\cS_t}|c(z)-c'(z)|,&\cS_t\ne\varnothing,\\
0,&\cS_t=\varnothing.
\end{cases}
\]
On $\mathcal E_C$ these are exactly the confidence sets and widths
used by the algorithm. Off this event they remain well-defined
analysis quantities. Write
\[
B_T^{\rm pair}=4B_T^{\cC},\qquad
R^2(r)=4m+2B_T^{\rm pair}+2Tr^2
      =4m+8B_T^{\cC}+2Tr^2.
\]

\begin{lemma}\label{lemma:cost-width-potential}
For every history, the width defined above satisfies
\begin{equation}\label{eq:width-pointwise-leverage}
w_t(z)\le r+R(r)\sqrt{\ell_t(z)}
\qquad(z\in\cZ,\ 1\le t\le T).
\end{equation}
The bound is deterministic and does not require the oracle event.
\end{lemma}
The same bound applies to any history-measurable subsets of $\cC$
with the empirical pair budget in Lemma~\ref{lemma:cost-pair-budget}.
Indeed, its proof uses only this budget and the fixed class geometry.
Truth retention is used separately to certify feasible policies.

Let $\cH_{t-1}$ contain the information available before the current
context, including the randomness used to specify the current policy
map and regression fits. The matrix $\Sigma_t$ and map $\ell_t$ are then
determined before $X_t$ is observed. For the actual adaptive policy
$\pi_t$, set
\[
\mu_t=\EE[L_t\mid\cH_{t-1}]
=\EE_{X\sim Q_x}\sum_{a\in\cA}\pi_t(a\mid X)\ell_t(X,a).
\]
For infinite actions we replace the sum by an integral. Since
\eqref{eq:width-pointwise-leverage} holds on every history, we may
take its conditional expectation under the original interaction law
and apply Jensen's inequality to obtain
\begin{equation}\label{eq:cost-width-predictable}
\EE[w_t(Z_t)\mid\cH_{t-1}]
\le r+R(r)\sqrt{\mu_t}.
\end{equation}
In particular, this step does not condition on $\mathcal E_C$ or
on a future geometric success event.

Lemma~\ref{lemma:elliptical-potential} gives, for every realized
sequence,
\[
0\le L_t\le1,\qquad
\sum_{t=1}^T L_t\le2m\log(1+T/m).
\]
Apply Lemma~\ref{lemma:predictable-nonnegative} with
$\zeta=\delta/4$. With probability at least $1-\delta/4$, an event
we denote by $\mathcal E_{\rm geom}(r)$, we have
\begin{equation}\label{eq:predictable-cost-potential}
\sum_{t=1}^T\mu_t
\le4m\log(1+T/m)+2\log(4/\delta).
\end{equation}
This event is also established under the original interaction law.
Consequently, $\mathcal E_C\cap\mathcal E_{\rm geom}(r)$ has
probability at least $1-\delta/2$, even though both events depend
on the same data and policies.

The next proposition combines the empirical pair bound and the
predictable potential. We include both the zero-dimensional case
and the deterministic bound supplied by the cost-class diameter.

\begin{proposition}\label{prop:cost-width-bound}
Set $r_T=1/T$, $m=m_T=m(r_T)$, and
\[
R_T^2=4m+8B_T^{\cC}+2Tr_T^2.
\]
For $m\ge1$, define
\begin{equation}\label{eq:cost-width-budget}
\mathcal W_T=
\min\left\{
T\operatorname{diam}_\infty(\cC),\
Tr_T+\sqrt{T R_T^2
\left[4m\log(1+T/m)+2\log(4/\delta)\right]}
\right\}.
\end{equation}
For $m=0$, set
$\mathcal W_T=\min\{T\operatorname{diam}_\infty(\cC),Tr_T\}$.
With probability at least $1-\delta/2$, all cost confidence sets
retain $c^*$ and
\begin{equation}\label{eq:cumulative-cost-width}
\sum_{t=1}^T
\EE_{X\sim Q_x}\sum_{a\in\cA}\pi_t(a\mid X)w_t(X,a)
\le\mathcal W_T.
\end{equation}
The quantity $\mathcal W_T$ is defined in
\eqref{eq:cost_width_total_bound}. The conclusion holds for the
actual adaptive policies of the algorithm and requires no minimum
exploration probability.
\end{proposition}

For $m\ge1$, combine~\eqref{eq:cost-width-predictable} and
\eqref{eq:predictable-cost-potential} on
$\mathcal E_C\cap\mathcal E_{\rm geom}(r_T)$. Summation and the
Cauchy--Schwarz inequality give the square-root term in
\eqref{eq:cost-width-budget}. The diameter bound follows from
$\cS_t\subseteq\cC$. For $m=0$, every difference in $\cD$ has
uniform norm at most $r_T$, so only $\mathcal E_C$ is needed.
We give the complete calculation in the proof of
Proposition~\ref{prop:cost-width-bound} in
Appendix~\ref{app:omitted-spps-proofs}.
Its conclusion is exactly Property~\ref{prop:sequential_width}.
The conditional action expectations use the same actual policies
that generate the cost regression data throughout this argument.

Finally, Lemma~\ref{lemma:entropy-linear-span} yields
\[
m(1/T)\le C_{p_{\cC}}D_{\cC}
[\log(eA_{\cC}D_{\cC}T)]^{p_{\cC}-1}.
\]
Thus polylogarithmic offline cost budgets make the cumulative width
of order $\sqrt T$ times polylogarithmic factors, as used in the
regret analysis.
\end{proof}

We can also choose a coarser approximation radius when the class
contains many small-amplitude directions. This choice affects only
the analysis, since the algorithm uses the convex confidence sets
directly.

\begin{corollary}\label{cor:cost-width-radius}
Fix a deterministic $r\in(0,1]$ and put $m=m(r)$. For $m\ge1$, define
\begin{equation}\label{eq:cost-width-arbitrary-radius}
\begin{aligned}
R_T^2(r)&=4m+8B_T^{\cC}+2Tr^2,\\
L_T(r)&=4m\log(1+T/m)+2\log(4/\delta),\\
\mathcal W_T(r)&=\min\{T\operatorname{diam}_\infty(\cC),\ Tr+
                 \sqrt{T R_T^2(r)L_T(r)}\}.
\end{aligned}
\end{equation}
For $m=0$, set $\mathcal W_T(r)=\min\{T\operatorname{diam}_\infty(\cC),Tr\}$.
With probability at least $1-\delta/2$, all cost confidence sets
retain $c^*$ and the left-hand side of
\eqref{eq:cumulative-cost-width} is at most $\mathcal W_T(r)$.
Under the assumptions of Theorem~\ref{thm:main_guarantee}, feasibility
at every context and round and the following regret bound hold
jointly with probability at least $1-\delta$:
\begin{equation}\label{eq:finite-regret-arbitrary-radius}
\operatorname{Reg}(T)\le\sqrt{KT\bar H_T}
\left[1+\frac{2\log(1+T/K)}{\log2}\right]
+\frac{\mathcal W_T(r)}{\Delta}.
\end{equation}
\end{corollary}

The radius is fixed before observing the data; $r=1/T$ recovers the
default bound. We give the detailed proofs of this corollary and all
auxiliary lemmas and propositions stated above in
Appendix~\ref{app:omitted-spps-proofs}.

\section{Proof of Regret Guarantee with Finite Actions}\label{app:proofs1}
We prove Theorem~\ref{thm:main_guarantee} in three steps. First,
a uniform covering argument turns offline reward error into
confidence bounds for every policy. Next, we verify that replay
computes the required historical policy masses and bound the
cumulative reward bonus. Finally, the cost confidence guarantees allow
us to construct a safe comparator and establish constrained
optimism. The resulting one-round regret bound separates into
reward-bonus terms and a cost-width term, which we sum using the
reward potential below and Property~\ref{prop:sequential_width}.

We work with the pre-context history $\cH_{t-1}$ from
Section~\ref{sec:model}. It contains both observed signals from
previous rounds and any random seed used to define the current
policy rule. Given this history, $\pi_t(\cdot\mid x)$ is a fixed
measurable kernel for every $x$. Both current signals arrive after
the current action, so neither enters this kernel. Throughout this
appendix, an expectation over $X\sim Q_x$ integrates a fresh context
with the realized history fixed. We never condition that fresh
context on a future confidence event.

\subsection{A uniform reward transfer bound}
The offline oracle controls error on the observed pairs. To obtain
policy confidence, we first relate this empirical error to error
averaged over the historical sampling policies. We establish the
relation for a fixed predictor and then extend it uniformly to the
reward class so that it applies to the fitted predictor.

Fix $T\ge2$, $\delta\in(0,1)$, and
$\eta_{\cF}=1/(6T)$. We choose once and for all a deterministic proper
$\eta_{\cF}$-net $\cF_\eta\subseteq\cF$ of size $N_{\cF}$.
This net is used only in the proof. Its centers need not include $f^*$,
and the algorithm does not enumerate them.

For each fixed $f\in\cF$, define
\[
L_{f,i}=(f(Z_i)-f^*(Z_i))^2,\qquad
\mu_{f,i}=\int_{\cX}\int_{\cA}
  (f(x,a)-f^*(x,a))^2\,\pi_i(da\mid x)\,Q_x(dx).
\]
The second expression holds the pre-round history fixed. Because
$X_i$ is independent of that history and
$A_i\sim\pi_i(\cdot\mid X_i)$, it is a version of
$\EE[L_{f,i}\mid\cH_{i-1}]$ for every deterministic $f$.
The integral formula defines $\mu_{f,i}$ jointly for all functions,
without choosing different exceptional sets for different $f$.
Boundedness gives $0\le L_{f,i}\le1$.

\begin{lemma}\label{lemma:apply_Freedman}
For each deterministic $f\in\cF$, $1\le t\le T$, and
$\rho\in(0,1)$, with probability at least $1-\rho$,
\[
 \sum_{i<t}\mu_{f,i}
 \le2\sum_{i<t}L_{f,i}+2\log(1/\rho).
\]
\end{lemma}
\begin{proof}[Proof of Lemma~\ref{lemma:apply_Freedman}]
Let $b=1-e^{-1}>1/2$. Convexity of the exponential function gives
$e^{-u}\le1-bu$ for $u\in[0,1]$. Hence
\[
 \EE[e^{-L_{f,i}}\mid\cH_{i-1}]
 \le1-b\mu_{f,i}\le e^{-b\mu_{f,i}}.
\]
We obtain the nonnegative supermartingale
\[
 M_{f,n}=\exp\left\{b\sum_{i=1}^n\mu_{f,i}
                      -\sum_{i=1}^nL_{f,i}\right\},\qquad M_{f,0}=1.
\]
In particular, $\EE M_{f,t-1}\le1$. Markov's inequality shows that
$M_{f,t-1}\le1/\rho$ except on an event of probability at most $\rho$.
Taking logarithms and dividing by $b$ gives the claim, because
$b^{-1}<2$ and both terms on the right are nonnegative.
For $t=1$, both sums are empty and the same conclusion holds
deterministically.
\end{proof}

The latest fit depends on the same observations used to measure
its error. We therefore make the preceding inequality simultaneous
over all predictors and rounds before substituting that fit. A
finite covering net supplies the common event in the next lemma.

\begin{lemma}
\label{lemma:apply_Freedman_union_t}
There is an event $\cE_{R,\mathrm{cov}}$ of probability at least
$1-\delta/4$ on which, simultaneously for every $1\le t\le T$
and every $f\in\cF$,
\begin{equation}\label{inequa:important1}
 \sum_{i<t}\mu_{f,i}
 \le2\sum_{i<t}(f(Z_i)-f^*(Z_i))^2
       +2\log\frac{4TN_{\cF}}{\delta}
       +6(t-1)\eta_{\cF}.
\end{equation}
\end{lemma}
\begin{proof}[Proof of Lemma~\ref{lemma:apply_Freedman_union_t}]
We apply Lemma~\ref{lemma:apply_Freedman} to each of the
$TN_{\cF}$ pairs $(t,v)$ with $v\in\cF_\eta$, with failure probability
at most $\rho=\delta/(4TN_{\cF})$. A union bound gives a single event
whose complement has probability at most $\delta/4$ and on which
\[
 \sum_{i<t}\mu_{v,i}
 \le2\sum_{i<t}L_{v,i}+2\log\frac{4TN_{\cF}}{\delta}
 \quad\text{for every }t\le T,\ v\in\cF_\eta.
\]
Now fix any $f\in\cF$ and choose $v\in\cF_\eta$ with
$\|f-v\|_\infty\le\eta_{\cF}$. Since all three functions
$f,v,f^*$ take values in $[0,1]$, we have pointwise
\[
 |(f-f^*)^2-(v-f^*)^2|
 =|f-v|\,|f+v-2f^*|\le2\eta_{\cF}.
\]
The same bound after integration gives
$|\mu_{f,i}-\mu_{v,i}|\le2\eta_{\cF}$. Substituting both
approximations into the preceding finite-net inequality yields
\[
 \sum_{i<t}\mu_{f,i}
 \le2\sum_{i<t}L_{f,i}
       +2\log\frac{4TN_{\cF}}{\delta}
       +(4+2)(t-1)\eta_{\cF},
\]
which is~\eqref{inequa:important1}. The extension from the finite
net is deterministic and holds for every $f$ on the same event.
We may therefore evaluate this inequality at an adaptively fitted
function after the event has been formed.
\end{proof}

In particular, substituting $f=\hat f^t$ means evaluating the random
map $f\mapsto\mu_{f,i}$ at the latest fitted function. It does not
assert that $\hat f^t$ was measurable before observation $i$, and
does not move a future fit inside an earlier conditional expectation.
This distinction is the reason for taking the fixed-net event first.

\subsection{From offline reward fits to policy confidence}
We now apply Lemma~\ref{lemma:apply_Freedman_union_t} to the
latest offline fit and combine it with the oracle error budget.
The goal is a reward-confidence bound for any policy, with
uncertainty measured by the bonus used in the algorithm.
For finite actions, write
\[
 \omega_t(x,a)=\sum_{i<t}\pi_i(a\mid x),\qquad
 q_t(x,a)=\frac1{1+\omega_t(x,a)},\qquad
 Q_t(\pi)=\EE_X\sum_a\pi(a\mid X)q_t(X,a).
\]
These are the bonuses in~\eqref{eq:finite_bonus}. The unit
pseudocount makes every denominator positive, including at $t=1$.

Let $\cE_{R,\mathrm{or}}$ be the event that
\[
 \sum_{i<t}(\hat f^t(Z_i)-f^*(Z_i))^2\le B_t^{\cF}
 \quad\text{for every }2\le t\le T.
\]
Each fit uses confidence $\delta/(4T)$ as specified in
\eqref{eq:oracle_budgets}. Assumption~\ref{ass:oracle} and a union
bound therefore imply
\[
 \PP(\cE_{R,\mathrm{or}}^c)\le(T-1)\frac{\delta}{4T}
 \le\frac\delta4.
\]
We set
$\cE_R=\cE_{R,\mathrm{or}}\cap\cE_{R,\mathrm{cov}}$, so that
$\PP(\cE_R)\ge1-\delta/2$.

On this event, the uniform inequality~\eqref{inequa:important1}
can be evaluated at $\hat f^t\in\cF$. By the conditional-mean
identity
\begin{equation}\label{eq:reward_conditional_mean_identity}
 \mu_{f,i}=\EE_{X,A\sim\pi_i}
              [(f(X,A)-f^*(X,A))^2],
\end{equation}
we obtain, for all $t\le T$,
\begin{equation}\label{eq:reward_population_error}
\begin{aligned}
 \sum_{i<t}\EE_{X,A\sim\pi_i}
       [(\hat f^t(X,A)-f^*(X,A))^2]
 &\le 2B_t^{\cF}
       +2\log\frac{4TN_{\cF}}{\delta}
       +6(t-1)\eta_{\cF}\\
 &=H_t-1.
\end{aligned}
\end{equation}
For $t=1$, the left side is zero and the inequality holds without
an oracle guarantee. Notice that every historical expectation in
\eqref{eq:reward_population_error} evaluates the same latest fit
$\hat f^t$, with the corresponding historical policy held fixed.

Equation~\eqref{eq:reward_population_error} controls squared error
under the historical policies. The next lemma converts this bound
into expected reward error under any candidate policy; the factor
$Q_t(\pi)$ accounts for how little mass its actions received from
those historical policies.

\begin{lemma}
\label{lemma:estimation_error}
On $\cE_R$, simultaneously for every $1\le t\le T$ and every
measurable policy $\pi$,
\[
 \left|\EE_X[\hat f^t(X,\pi)-f^*(X,\pi)]\right|
 \le\sqrt{H_tQ_t(\pi)}.
\]
\end{lemma}
\begin{proof}[Proof of Lemma~\ref{lemma:estimation_error}]
Fix a realized history in $\cE_R$, a round $t$, and a policy $\pi$.
Set $e_t(x,a)=\hat f^t(x,a)-f^*(x,a)$, which belongs to $[-1,1]$.
We apply the Cauchy--Schwarz inequality with respect to the measure
$Q_x(dx)\pi(a\mid x)$:
\[
\begin{aligned}
 &\left|\EE_X\sum_a\pi(a\mid X)e_t(X,a)\right|^2\\
 &\quad\le
 \left(\EE_X\sum_a\frac{\pi(a\mid X)}{1+\omega_t(X,a)}\right)
 \left(\EE_X\sum_a
     \pi(a\mid X)(1+\omega_t(X,a))e_t(X,a)^2\right).
\end{aligned}
\]
The first factor equals $Q_t(\pi)$. In the second factor, boundedness
and $\pi(a\mid x)\le1$ imply
\[
\begin{aligned}
 \EE_X\sum_a\pi(a\mid X)(1+\omega_t(X,a))e_t(X,a)^2
 &\le1+\EE_X\sum_a \omega_t(X,a)e_t(X,a)^2\\
 &=1+\sum_{i<t}\EE_{X,A\sim\pi_i}e_t(X,A)^2\\
 &\le H_t,
\end{aligned}
\]
where the last step is~\eqref{eq:reward_population_error}.
Taking square roots proves the claim. The argument applies to
every policy on the same fixed history, so it also covers policies
chosen using all information currently available, including the
safe mixture constructed below. No union bound over policies is
required.
\end{proof}

To match the linear bonus in the policy objective, we next separate
this square-root confidence bound into a multiple of $Q_t(\pi)$
and a policy-independent remainder. We will use this form both for
the safe comparator and for the policy selected by the algorithm.

\begin{lemma}\label{lemma:square_trick}
On $\cE_R$, with $\beta_t=\sqrt{tH_t/K}$, every $t\le T$ and
policy $\pi$ satisfy
\begin{equation}\label{eq:finite-linearized-confidence}
 \left|\EE_X[\hat f^t(X,\pi)-f^*(X,\pi)]\right|
 \le\beta_tQ_t(\pi)+\frac{K\beta_t}{4t}.
\end{equation}
\end{lemma}
\begin{proof}[Proof of Lemma~\ref{lemma:square_trick}]
For $H,Q\ge0$ and $\beta>0$,
\[
 0\le\left(\sqrt{\beta Q}-\sqrt{H/(4\beta)}\right)^2
 \quad\Longrightarrow\quad
 \sqrt{HQ}\le\beta Q+\frac{H}{4\beta}.
\]
We apply this to Lemma~\ref{lemma:estimation_error}.
Since $H_t>0$ and $\beta_t^2=tH_t/K$,
$H_t/(4\beta_t)=K\beta_t/(4t)$. This proves the result at every
round, including $t=1$ with its arbitrary initial fit in $\cF$.
\end{proof}

\subsection{Replay and the deterministic reward potential}
The confidence bounds above use the historical kernels evaluated
at a fresh context. We first check that counterfactual replay
recovers exactly these kernels and their cumulative masses. This
also identifies the policy-optimality inequality needed for
constrained optimism.

\begin{lemma}\label{lemma:equivalent_finite}
Algorithm~\ref{alg:counter-factual_simulation}, evaluated at any
context $x$ and time $t$, reconstructs every archived historical
kernel $\pi_i(\cdot\mid x)$ for $1\le i\le t$. In particular, its
output equals the rule in~\eqref{equation:ucb_policy_space} unless an
empty-set check has instructed the algorithm to choose $a_0$ in all
remaining rounds, in which case it equals $\pi_0$.
\end{lemma}
\begin{proof}[Proof of Lemma~\ref{lemma:equivalent_finite}]
Fix the actual history through the round-$t$ model update and a
queried context $x$. All archived predictors, cost confidence
objects, coefficients, recorded instructions, and tie-breaking rules are now fixed.
At $i=1$, both replay and the historical rule have zero cumulative
mass and bonus $q_1(x,a)=1$. If the archive instructs choosing $a_0$, both
return $\pi_0$; otherwise they maximize the same objective over
$\Pi_1^\delta(x)$ with the same tie-breaking rule.

Suppose the kernels agree for all indices $j<i$. Replay has updated
its stored mass after each preceding stage, including a stage at which
the recorded instruction selects $a_0$. Consequently its current value satisfies
\[
 \omega_i(x,a)=\sum_{j<i}\widetilde\pi_{t,j}(a)
     =\sum_{j<i}\pi_j(a\mid x).
\]
The computed reciprocal is exactly $q_i(x,a)$. If the round-$i$
archive instructs choosing $a_0$, the two outputs again both equal $\pi_0$.
Otherwise the archived $\hat f^i$, $\cS_i$, $w_i$, $\beta_i$, and
$\alpha_r$ define the same feasible set and the same objective
\[
 p\longmapsto\hat f^i(x,p)+\beta_iq_i(x,p)+\alpha_rw_i(x,p).
\]
Fixed tie-breaking gives identical maximizers. Induction proves
the assertion for every $i\le t$.

On a history where the empty-set check has not been triggered, pointwise maximality also implies
maximality after integrating over a fresh context: for every
measurable feasible comparator $\pi$,
\[
\begin{aligned}
 &\EE_X[\hat f^i(X,\pi_i)+\beta_iq_i(X,\pi_i)+\alpha_rw_i(X,\pi_i)]\\
 &\qquad\ge
 \EE_X[\hat f^i(X,\pi)+\beta_iq_i(X,\pi)+\alpha_rw_i(X,\pi)].
\end{aligned}
\]
Measurability is part of the setup. Throughout replay, the real
dataset remains fixed; no oracle is refitted, and no new observation
is created.
\end{proof}

With replay connected to the policy rule, it remains to control
the bonuses accumulated by that rule. The next bound holds for
every sequence of policy kernels, so it applies directly to the
algorithm's history-dependent choices.

\begin{lemma}\label{lemma:potential_lemma}
For every context $x$, every sequence of policy kernels, and every
$T\ge1$,
\[
 \sum_{t=1}^Tq_t(x,\pi_t)
 \le\frac K{\log2}\log(1+T/K).
\]
Consequently, for every realized adaptive history,
\begin{equation}\label{eq:finite_expected_potential}
 \sum_{t=1}^TQ_t(\pi_t)
 \le\frac K{\log2}\log(1+T/K).
\end{equation}
\end{lemma}
\begin{proof}[Proof of Lemma~\ref{lemma:potential_lemma}]
Fix $x$ and put $p_{t,a}=\pi_t(a\mid x)$ and
$S_{t,a}=\sum_{i\le t}p_{i,a}$, with $S_{0,a}=0$.
Then $u_{t,a}=p_{t,a}/(1+S_{t-1,a})$ lies in $[0,1]$.
Concavity of the logarithm gives
$\log(1+u)\ge u\log2$ on that interval. We obtain
\[
\begin{aligned}
 \sum_{t=1}^Tq_t(x,\pi_t)
 &=\sum_a\sum_{t=1}^T\frac{p_{t,a}}{1+S_{t-1,a}}\\
 &\le\frac1{\log2}\sum_a\sum_{t=1}^T
       \log\left(1+\frac{p_{t,a}}{1+S_{t-1,a}}\right)\\
 &=\frac1{\log2}\sum_a\log(1+S_{T,a})\\
 &\le\frac K{\log2}
       \log\left(1+\frac1K\sum_a S_{T,a}\right)
 =\frac K{\log2}\log(1+T/K).
\end{aligned}
\]
The penultimate step is Jensen's inequality, and the last uses
$\sum_a S_{T,a}=T$. This bound is deterministic for each sequence
of kernels. After fixing a realized history, we apply it at every
context and integrate. The result is~\eqref{eq:finite_expected_potential};
adaptivity introduces no additional event or independence
requirement.
\end{proof}

\subsection{Safe mixtures and optimism}
Following the safe-mixture argument of
\citet{pacchiano2021stochastic,pacchiano2024contextual}, we next
connect reward confidence to the cost constraint. The true
optimal policy may lie outside the robust feasible set. Mixing it
with the known safe action gives a feasible comparator whose
reward loss is covered by its own cost-width bonus. Together with
Lemmas~\ref{lemma:square_trick} and~\ref{lemma:equivalent_finite},
this comparison yields the following optimism bound.

Let $\cE_{C,\mathrm{or}}$ denote the simultaneous cost-oracle event.
Its failure probability is at most $\delta/4$, by exactly the union
bound used for the reward fits. Lemma~\ref{lemma:cost_confidence}, proved in
Appendix~\ref{app:cost-property}, implies that on this event
$c^*\in\cS_t$ at every round and the empty-set check is not triggered.
The same event suffices for truth retention, nonemptiness, and the
resulting policy-cost confidence bounds.

\begin{lemma}\label{lemma:optimism_upperbound}
On $\cE_R\cap\cE_{C,\mathrm{or}}$, every $1\le t\le T$ satisfies
\[
 \EE_X f^*(X,\pi^*)
 \le\EE_X[\hat f^t(X,\pi_t)+\beta_tq_t(X,\pi_t)
                       +\alpha_rw_t(X,\pi_t)]
       +\frac{K\beta_t}{4t},
 \qquad \alpha_r=\frac1\Delta.
\]
\end{lemma}
\begin{proof}[Proof of Lemma~\ref{lemma:optimism_upperbound}]
Fix a realized history in this intersection and a round $t$.
For each context and distribution $p$, true-cost retention gives
\[
 0\le V_t^\delta(x,p)-c^*(x,p)\le w_t(x,p).
\]
Indeed, $c^*$ is one of the functions in the supremum defining
$V_t^\delta$, and every other $c\in\cS_t$ satisfies
\[
 c(x,p)-c^*(x,p)
 \le\sum_a p(a)|c(x,a)-c^*(x,a)|
 \le\sum_a p(a)w_t(x,a).
\]
Taking the supremum proves the stated upper bound.

All candidates have the known safe-action cost $c_0$. Therefore
$V_t^\delta(x,\pi_0)=c_0$ and $w_t(x,\pi_0)=0$. For any
$\gamma\in[0,1]$ we have the exact mixture identity
\[
 V_t^\delta(x,\gamma p+(1-\gamma)\pi_0)
 =\gamma V_t^\delta(x,p)+(1-\gamma)c_0.
\]
We obtain it by taking the supremum of
$\gamma c(x,p)+(1-\gamma)c_0$ over $c\in\cS_t$.
No attainment of the supremum is needed.

Define, pointwise in $x$,
\[
 h_t(x)=[V_t^\delta(x,\pi^*)-\tau]_+,\qquad
 \gamma_t(x)=\frac{\Delta}{\Delta+h_t(x)},\qquad
 \widetilde\pi_t=\gamma_t\pi^*+(1-\gamma_t)\pi_0.
\]
These functions are measurable by the setup assumptions, and
$\gamma_t(x)\in(0,1]$. The mixture identity yields
\[
\begin{aligned}
 V_t^\delta(x,\widetilde\pi_t)
 &\le\gamma_t(x)(\tau+h_t(x))+(1-\gamma_t(x))c_0\\
 &=c_0+\gamma_t(x)(\Delta+h_t(x))=\tau.
\end{aligned}
\]
Thus $\widetilde\pi_t$ is robustly feasible at every context.
Since $c^*(x,\pi^*)\le\tau$, the upper-cost error bound also gives
$h_t(x)\le w_t(x,\pi^*)$.
Width is linear in the action distribution, and its safe-action
value vanishes, so
$w_t(x,\widetilde\pi_t)=\gamma_t(x)w_t(x,\pi^*)$.

We now bound the reward loss using only its range $[0,1]$:
\begin{equation}\label{eq:safe_mixture_reward}
\begin{aligned}
 f^*(x,\pi^*)-f^*(x,\widetilde\pi_t)
 &=(1-\gamma_t(x))
       [f^*(x,\pi^*)-f^*(x,a_0)]\\
 &\le1-\gamma_t(x)
 =\frac{\gamma_t(x)h_t(x)}{\Delta}\\
 &\le\frac{\gamma_t(x)w_t(x,\pi^*)}{\Delta}
 =\alpha_rw_t(x,\widetilde\pi_t).
\end{aligned}
\end{equation}
This calculation requires no known reward for the safe action.

For the fixed history, set
$J_t(\pi)=\EE_X[\hat f^t(X,\pi)+\beta_tq_t(X,\pi)
+\alpha_rw_t(X,\pi)]$. We integrate
\eqref{eq:safe_mixture_reward} and apply
Lemma~\ref{lemma:square_trick} to the entire policy
$\widetilde\pi_t$, obtaining
\[
\begin{aligned}
 \EE_X f^*(X,\pi^*)
 &\le\EE_X f^*(X,\widetilde\pi_t)
                  +\alpha_r\EE_Xw_t(X,\widetilde\pi_t)\\
 &\le J_t(\widetilde\pi_t)+\frac{K\beta_t}{4t}\\
 &\le J_t(\pi_t)+\frac{K\beta_t}{4t}.
\end{aligned}
\]
The last inequality uses robust feasibility of the mixture,
nonemptiness of the cost sets, and the policy optimality identified
by Lemma~\ref{lemma:equivalent_finite}. It proves the claim.
All context-dependent mixture coefficients stay inside their
corresponding integrals.
\end{proof}

\subsection{Completion of the finite-action theorem}
We now combine these results. Constrained optimism and reward
confidence give the one-round regret bound, while
Lemma~\ref{lemma:potential_lemma} and
Property~\ref{prop:sequential_width} control its cumulative reward
and cost terms, respectively. We first place both arguments on a
common event to obtain the theorem's joint safety and regret guarantee.

\begin{proof}[Proof of Theorem~\ref{thm:main_guarantee}]
Let $\cE_{C,\mathrm{pot}}$ be the cost-geometry martingale event
used in the proof of Property~\ref{prop:sequential_width}; when
$m_T=0$, take this event to be the whole sample space. Its failure
probability is at most $\delta/4$. On
$\cE_{C,\mathrm{or}}\cap\cE_{C,\mathrm{pot}}$, that property gives
\[
 \sum_{t=1}^T\EE_{X,A\sim\pi_t}w_t(X,A)\le\mathcal W_T.
\]
The bound applies to the actual sampled policy sequence, whose
observations generated both regression datasets. We work on
\[
 \cE=\cE_{R,\mathrm{or}}\cap\cE_{R,\mathrm{cov}}
       \cap\cE_{C,\mathrm{or}}\cap\cE_{C,\mathrm{pot}}.
\]
The four failure probabilities are each at most $\delta/4$.
A union bound therefore gives $\PP(\cE)\ge1-\delta$, with no
independence assumption between reward and cost noise.

On $\cE$, all confidence sets retain the true cost and the
empty-set check is not triggered. The policy chosen at each round belongs to
$\Pi_t^\delta(x)$ at every context. Consequently
\[
 c^*(x,\pi_t)\le V_t^\delta(x,\pi_t)\le\tau
 \quad\text{for all }x,\ 1\le t\le T.
\]
This includes the first round: $\cS_1=\cC$ before any
observations, and the safe policy
already makes the robust feasible set nonempty.

For regret, subtract $\EE_Xf^*(X,\pi_t)$ from the optimism bound
in Lemma~\ref{lemma:optimism_upperbound} and apply
Lemma~\ref{lemma:square_trick} once more,
this time to the actual policy. At every round, including $t=1$,
we obtain
\begin{equation}\label{eq:finite_regret_step}
\begin{aligned}
 \EE_X[f^*(X,\pi^*)-f^*(X,\pi_t)]
 &\le2\beta_tQ_t(\pi_t)+\frac{K\beta_t}{2t}
           +\alpha_r\EE_Xw_t(X,\pi_t).
\end{aligned}
\end{equation}
Summing and using $\beta_t\le\sqrt{T\bar H_T/K}$ and
Lemma~\ref{lemma:potential_lemma}, we get
\[
\begin{aligned}
 2\sum_{t=1}^T\beta_tQ_t(\pi_t)
 &\le2\sqrt{\frac{T\bar H_T}{K}}\,
       \frac K{\log2}\log(1+T/K)\\
 &=\frac2{\log2}\sqrt{KT\bar H_T}\log(1+T/K).
\end{aligned}
\]
For the remaining reward term, the round-dependent bound on
$\beta_t$ yields
\[
 \sum_{t=1}^T\frac{K\beta_t}{2t}
 \le\frac{\sqrt{K\bar H_T}}2\sum_{t=1}^Tt^{-1/2}
 \le\sqrt{KT\bar H_T}.
\]
\begin{samepage}
The final inequality follows from
$t^{-1/2}\le\int_{t-1}^t u^{-1/2}\,du$ and
$\int_0^T u^{-1/2}\,du=2\sqrt T$.
We bound the accumulated cost term by
Property~\ref{prop:sequential_width}. Combining these estimates
after summing~\eqref{eq:finite_regret_step} over rounds proves
\begin{equation}\label{eq:finite_regret_explicit_constants}
 \operatorname{Reg}(T)
 \le\sqrt{KT\bar H_T}
      \left[1+\frac{2\log(1+T/K)}{\log2}\right]
       +\frac{\mathcal W_T}{\Delta}.
\end{equation}
\end{samepage}
This is exactly~\eqref{eq:finite_main_bound}. All inequalities
hold on $\cE$, so safety and regret have the asserted joint
probability.
\end{proof}

Under the covering-entropy assumptions and the following concrete
oracle-error conditions, we now specify the constants and logarithmic
exponents in Theorem~\ref{thm:main_guarantee}. We work on the same event
$\cE$ as in the proof above. Let $\ell_T=\log(eT/\delta)$ and suppose
\begin{equation}\label{eq:concrete-rate-conditions}
B_T^{\cF}\le C_F\ell_T^{b_F},\qquad
B_T^{\cC}\le C_C\ell_T^{b_C},
\end{equation}
where $C_F,C_C>0$ and $b_F,b_C\ge0$ are fixed independently of $T$
and $\delta$. The entropy parameters are those of
Assumptions~\ref{ass:reward_entropy} and~\ref{ass:cost_entropy}.
Define
\[
h_F=\max\{1,p_{\cF},b_F\},\qquad
s_C=\max\{b_C,p_{\cC}-1\},\qquad
\gamma_C=\frac{p_{\cC}+s_C}{2}.
\]
We may choose the following upper-bound constants:
\begin{equation}\label{eq:explicit-rate-constants}
\begin{aligned}
C_H&=2+2C_F+2D_{\cF}[1+\log(6A_{\cF})]^{p_{\cF}}
       +2(1+\log4),\\
C_m&=C_{p_{\cC}}D_{\cC}
       [1+\log(A_{\cC}D_{\cC})]^{p_{\cC}-1},\\
C_R&=4C_m+8C_C+2,\qquad
C_L=4C_m+2(1+\log4).
\end{aligned}
\end{equation}
Here $C_{p_{\cC}}$ is the constant in
\eqref{eq:entropy_dimension_bound}. These constants depend only on the
stated entropy and oracle-error parameters and are independent of
$T$ and $\delta$.

To verify these choices, note that $\ell_T\ge1$ and
\[
\log(6A_{\cF}T)\le[1+\log(6A_{\cF})]\ell_T,\qquad
\log(eA_{\cC}D_{\cC}T)
\le[1+\log(A_{\cC}D_{\cC})]\ell_T.
\]
Substituting these inequalities into
\eqref{eq:reward_budget} and~\eqref{eq:entropy_dimension_bound} gives
\[
\bar H_T\le C_H\ell_T^{h_F},\qquad
m_T\le C_m\ell_T^{p_{\cC}-1}.
\]
For $m_T\ge1$, the definitions of the confidence radius and
predictable potential then yield
\[
\begin{aligned}
R_T^2
&=4m_T+8B_T^{\cC}+2/T
 \le C_R\ell_T^{s_C},\\
L_T^{\cC}
&=4m_T\log(1+T/m_T)+2\log(4/\delta)
 \le C_L\ell_T^{p_{\cC}}.
\end{aligned}
\]
We used $\log(1+T/m_T)\le\ell_T$ and
$\log(4/\delta)\le(1+\log4)\ell_T$.
Consequently,
\begin{equation}\label{eq:explicit-cost-width-rate}
\mathcal W_T\le1+\sqrt{C_RC_LT}\,\ell_T^{\gamma_C}.
\end{equation}
This also holds when $m_T=0$, since then $\mathcal W_T\le1$.

Finally, $\log(1+T/K)\le\ell_T$ in
\eqref{eq:finite_regret_explicit_constants} gives
\begin{equation}\label{eq:finite-concrete-regret}
\operatorname{Reg}(T)
\le\left(1+\frac2{\log2}\right)\sqrt{C_HKT}\,
      \ell_T^{1+h_F/2}
 +\frac{1+\sqrt{C_RC_LT}\,\ell_T^{\gamma_C}}{\Delta}.
\end{equation}
Thus we obtain $\widetilde{\cO}(\sqrt T)$ regret for fixed $K$,
positive safety gap, and class and oracle-error parameters.
The reward class need not be convex; convexity is used only for the
cost approximation geometry. We give concrete models satisfying
these conditions in Appendix~\ref{app:parametric-entropy}.

\section{Proof of Regret Guarantee with Infinite Actions}\label{app:infinite-proofs}
We prove Theorem~\ref{thm:main_infinite_guarantee} by extending the
reward-confidence and replay arguments to the regularized design bonus.
We first derive policy confidence from the offline reward error and
verify that replay reconstructs the kernels used in this bound.
Assumption~\ref{ass:policy_div2} then controls their cumulative bonus;
we also give a sufficient condition used in the continuous-action
example of Appendix~\ref{app:rank-two-example}.
Finally, a safe comparator yields constrained optimism. Combining it
with reward confidence and the cost-width bound in
Property~\ref{prop:sequential_width} completes the joint feasibility
and regret guarantee.

We use the same observation history $\cH_{t-1}$ as in the finite-action
analysis. A policy kernel $\pi_t(\cdot|x)$ is fixed by this history before
the new context is drawn. Throughout the proof, $\EE_X$ integrates an
independent $X\sim Q_x$ while holding that history fixed. Rewards and
costs are bounded, so their expectations are finite under every action
distribution. Bonus integrability instead follows from
Assumption~\ref{ass:policy_div2}, as verified in
Lemma~\ref{lemma:infinite_action_1} below.

\subsection{Reward Confidence with a Reference Design}
We first verify that the reward population-error argument does not
require a finite action set. We then use the reference design to
convert that error budget into confidence for any candidate policy.
The design regularizes the bonus from round one without contributing
observations to the regression data. Put $n=t-1$ and fix a reward model
$f\in\cF$. For the actual observations $Z_i=(X_i,A_i)$, write
\[
L_{f,i}=(f(Z_i)-f^*(Z_i))^2,\qquad
\mu_{f,i}=\EE[L_{f,i}\mid\cH_{i-1}]
=\EE_X\int_{\cA}(f(X,a)-f^*(X,a))^2\,d\pi_i(a|X).
\]
Both quantities lie in $[0,1]$. Convexity of the exponential function on this
interval gives
$\EE[e^{-L_{f,i}}\mid\cH_{i-1}]
 \le1-(1-e^{-1})\mu_{f,i}\le e^{-\mu_{f,i}/2}$.
Iterating this conditional-expectation bound and applying Markov's
inequality show that, for any fixed
$f$ and prefix length $n$, with failure probability at most $\zeta$,
\[
\sum_{i=1}^n\mu_{f,i}
\le2\sum_{i=1}^nL_{f,i}+2\log(1/\zeta).
\]
We apply this statement to a fixed proper uniform
$\eta_{\cF}=1/(6T)$-net of $\cF$ and all $T$ prefixes. The union bound,
with $\zeta=\delta/(4TN_{\cF})$, costs at most $\delta/4$.

For functions within uniform distance $\eta_{\cF}$, their squared errors
relative to $f^*$ differ by at most $2\eta_{\cF}$, both pointwise and after
integration. Approximating an arbitrary $f\in\cF$ by its net point gives,
simultaneously for every $f$ and $t$,
\[
\sum_{i<t}\mu_{f,i}
\le2\sum_{i<t}(f(Z_i)-f^*(Z_i))^2
 +2\log(4TN_{\cF}/\delta)+6(t-1)\eta_{\cF}.
\]
We form this common event before substituting the random fit
$f=\hat f^t$. The reward oracle inequalities, each with failure probability
at most $\delta/(4T)$, hold jointly with failure probability at most $\delta/4$.
Consequently their intersection, denoted $\cE_R$, has probability at
least $1-\delta/2$, and it gives
\begin{equation}\label{eq:infinite_population_budget}
\EE_X\sum_{i<t}\int_{\cA}
  (\hat f^t(X,a)-f^*(X,a))^2\,d\pi_i(a|X)
\le H_t-1,\qquad 1\le t\le T.
\end{equation}
At $t=1$ the sum is zero, so the statement also covers the prescribed
initial fit $\hat f^1\in\cF$. Evaluating the uniform inequality at
$\hat f^t$ does not assert that this fit was known at time $i<t$;
the historical population integrals retain the archived kernels.

The next lemma turns this historical error budget into expected reward
confidence under an arbitrary policy. The reference part of the
denominator contributes the additional $D_0$ term, and the final
inequality puts the bound in the linear bonus form used by the policy
objective.

\begin{lemma}\label{lemma:infinite_action_1}
On $\cE_R$, simultaneously for every $1\le t\le T$ and every measurable
policy $\pi$, let $Q_t^\infty(\pi)=\EE_Xq_t^\infty(X,\pi)$. Then
\[
\left|\EE_X[\hat f^t(X,\pi)-f^*(X,\pi)]\right|
\le\sqrt{(H_t+D_0)Q_t^\infty(\pi)}
\le\beta_tQ_t^\infty(\pi)+\frac{\bar{\cM}\beta_t}{4t}.
\]
\end{lemma}
\begin{proof}[Proof of Lemma~\ref{lemma:infinite_action_1}]
For $h=f-f'$ with $f,f'\in\cF$, define
\[
D_{t,h}(x)=d_x\int h(x,a)^2\,d\rho_x^*(a)
+\sum_{i<t}\int h(x,a)^2\,d\pi_i(a|x).
\]
The normalization in \eqref{eq:infinite_bonus} cancels, giving
\begin{equation}\label{eq:infinite_divergence_normalized}
q_t^\infty(x,a)=
\sup_{f,f'\in\cF}\frac{(f(x,a)-f'(x,a))^2}{D_{t,f-f'}(x)}.
\end{equation}
If the design integral of a difference vanishes, finiteness of
$\cV_x(\rho_x^*)$ implies that the difference vanishes at every action.
Otherwise the denominator is at least its positive design contribution.
This establishes
\begin{equation}\label{eq:infinite_initial_design_bound}
0\le q_t^\infty(x,a)
\le\frac{\cV_x(a\|\rho_x^*)}{d_x}
\le\cM_x .
\end{equation}
In particular, the bonus is integrable over contexts under every policy.

Now put $e_t=\hat f^t-f^*$ and
\[
\widetilde S_t(x)
=d_x\int e_t(x,a)^2\,d\rho_x^*(a)
+\sum_{i<t}\int e_t(x,a)^2\,d\pi_i(a|x).
\]
Because both reward functions lie in $[0,1]$, the design contribution
is at most $d_x$. Equation~\eqref{eq:infinite_population_budget} gives
\[
\EE_X\widetilde S_t(X)\le H_t-1+\EE_Xd_X\le H_t+D_0.
\]
We have $e_t(x,a)^2\le q_t^\infty(x,a)\widetilde S_t(x)$ whenever
$\widetilde S_t(x)>0$. If this quantity is zero, the preceding
finite-design argument gives $e_t(x,a)=0$ for every action, so the
same inequality remains valid. We integrate and apply the Cauchy--Schwarz inequality
on the joint distribution of $X$ and $A\sim\pi(\cdot|X)$:
\[
\left|\EE_{X,A}e_t(X,A)\right|
\le\EE_{X,A}\sqrt{q_t^\infty(X,A)\widetilde S_t(X)}
\le\sqrt{Q_t^\infty(\pi)\EE_X\widetilde S_t(X)}.
\]
Finally, $\sqrt{uv}\le\beta u+v/(4\beta)$ and
$\beta_t^2=t(H_t+D_0)/\bar{\cM}$ give the second assertion. All statements hold
pointwise on the common event for every policy kernel, so data-dependent
policy selection requires no additional confidence allocation.
\end{proof}

\subsection{Counterfactual Replay}
The confidence bound depends on the actual historical kernels
evaluated at a fresh context. We next verify that the implemented
replay reconstructs these kernels exactly. This also justifies the
policy-optimality comparison used in the optimism argument.

\begin{lemma}\label{lemma:equivalent_infinite}
Archive each round's reward fit, cost confidence set, any instruction to
choose $a_0$ in all remaining rounds, and
tie-breaking rule. At a queried context $x$, reconstruct rounds
$i=1,\ldots,t$ in order, forming the bonus from $d_x\rho_x^*$ and the
already reconstructed kernels and applying \eqref{eq:infinite_policy}
with that round's archived objects. This replay returns
$\pi_t(\cdot|x)$ exactly.
\end{lemma}
\begin{proof}[Proof of Lemma~\ref{lemma:equivalent_infinite}]
Fix the observed history through round $t-1$ and a context $x$. Write
$\widetilde\pi_{t,i}(\cdot|x)$ for the kernel obtained at replay step $i$.
At $i=1$, the bonus is $\cV_x(a\|\rho_x^*)/d_x$ in both constructions.
The initial fit, confidence set, coefficient $\beta_1$, and tie-breaking
rule are also identical. Thus $\widetilde\pi_{t,1}=\pi_1$.

Suppose the equality holds through step $i-1$. If round $i$'s archive
instructs choosing $a_0$, both constructions return the safe point mass.
Otherwise,
\[
\frac{d_x\rho_x^*+\sum_{j<i}\widetilde\pi_{t,j}(\cdot|x)}
     {d_x+i-1}
=
\frac{d_x\rho_x^*+\sum_{j<i}\pi_j(\cdot|x)}
     {d_x+i-1}.
\]
The bonuses agree, and the remaining optimization data are the same
archived objects. Identical tie-breaking therefore gives
$\widetilde\pi_{t,i}=\pi_i$. Induction proves the claim.

The replay uses distributions rather than hypothetical sampled actions.
Neither oracle is refitted, and no past confidence set is recomputed
using later labels. Hence every reconstructed kernel is the evaluation
of its actual historical rule at the queried context. Pointwise
maximization implies the corresponding integrated objective comparison
after integration against $Q_x$; the algorithm does not need to know
this distribution.
\end{proof}

\subsection{Cumulative Reward Bonus and a Sufficient Design Condition}
\label{app:cumulative-potential-examples}
After identifying the replayed kernels, we need to bound the bonuses
they accumulate. The theorem uses
Assumption~\ref{ass:policy_div2} for this step. The following calculation
verifies its cumulative-potential inequality at a fixed context when
reward differences admit a finite feature representation. It supplies the bound used in
Appendix~\ref{app:rank-two-example}; the general theorem requires only
the stated design assumption. The features may be nonlinear in the
action, and the reward class itself need not be convex.

Fix a context and suppose every reward difference has the form
$h(a)=\theta_h^\top\phi(a)$ with $\phi(a)\in\RR^{d_{\rm feat}}$. Let
$G_0=\int\phi(a)\phi(a)^\top\,d\rho_x^*(a)\succ0$, and choose $d_x$
so that $\sup_a\phi(a)^\top(d_xG_0)^{-1}\phi(a)\le1$.
For arbitrary policy kernels set
\[
J_t=\int\phi(a)\phi(a)^\top\,d\pi_t(a|x),\qquad
S_{t-1}=d_xG_0+\sum_{i<t}J_i.
\]
The Cauchy--Schwarz inequality in the $S_{t-1}$ norm gives
$q_t^\infty(x,a)\le\phi(a)^\top S_{t-1}^{-1}\phi(a)$.
We therefore have
$q_t^\infty(x,\pi_t)\le\operatorname{tr}(S_{t-1}^{-1}J_t)$.
For $\Xi_t=S_{t-1}^{-1/2}J_tS_{t-1}^{-1/2}$, every eigenvalue is
nonnegative and
$\operatorname{tr}(\Xi_t)\le\operatorname{tr}((d_xG_0)^{-1}J_t)\le1$.
The scalar inequality $u\le2\log(1+u)$ for $0\le u\le1$ then gives
\[
\operatorname{tr}(\Xi_t)
\le2\log\det(I+\Xi_t)
=2\log\frac{\det S_t}{\det S_{t-1}}.
\]
After summing, the arithmetic--geometric mean inequality for eigenvalues
and $\sum_t\operatorname{tr}((d_xG_0)^{-1}J_t)\le T$ give
\begin{equation}\label{eq:linear-kernel-cumulative-potential}
\sum_{t=1}^Tq_t^\infty(x,\pi_t)
\le2\log\frac{\det S_T}{\det S_0}
\le2d_{\rm feat}\log(1+T/d_{\rm feat}).
\end{equation}
This proof applies to every policy sequence, including sequences that
never select an informative action. Their unselected-action bonuses may
remain large, but the cumulative on-policy quantity still satisfies
the displayed bound.

\subsection{Safe Mixtures and Constrained Optimism}
We now combine reward confidence with the cost confidence sets.
As in the finite-action argument, mixing the optimal feasible policy
with the safe action produces a robustly feasible comparator. Its
reward loss is covered by its own cost-width bonus. Applying
Lemma~\ref{lemma:infinite_action_1} to this mixture and using the
policy optimality established in
Lemma~\ref{lemma:equivalent_infinite} yields constrained optimism.

Let $\cE_{C,\mathrm{all}}$ be the intersection of the cost-oracle event and the
cost-geometry event in Property~\ref{prop:sequential_width}.
When $m_T=0$, take the latter event to be the whole sample space.
Each constituent event has failure probability at most $\delta/4$, so
$\Pr(\cE_{C,\mathrm{all}}\cap\cE_R)\ge1-\delta$. On this intersection the true cost
belongs to every $\mathcal S_t$, and the cumulative width bound applies
to the actual policy sequence.

\begin{lemma}\label{lemma:optimism_bound_infinite}
On $\cE_{C,\mathrm{all}}\cap\cE_R$, for every $1\le t\le T$,
\[
\EE_Xf^*(X,\pi^*)\le
\EE_X[\hat f^t(X,\pi_t)+\beta_tq_t^\infty(X,\pi_t)
+\alpha_rw_t(X,\pi_t)]+\frac{\bar{\cM}\beta_t}{4t}.
\]
\end{lemma}
\begin{proof}[Proof of Lemma~\ref{lemma:optimism_bound_infinite}]
Fix a history on this event. Truth inclusion and the definition of the
action-wise width imply, for every context and policy,
\begin{equation}\label{eq:infinite_cost_width_comparison}
0\le V_t^\delta(x,\pi)-c^*(x,\pi)\le w_t(x,\pi).
\end{equation}
Every candidate has cost $c_0$ at $a_0$. Consequently,
$w_t(x,a_0)=0$ and, for $\lambda\in[0,1]$,
\[
\begin{aligned}
V_t^\delta(x,\lambda\pi+(1-\lambda)\delta_{a_0})
 &=\lambda V_t^\delta(x,\pi)+(1-\lambda)c_0,\\
w_t(x,\lambda\pi+(1-\lambda)\delta_{a_0})
 &=\lambda w_t(x,\pi).
\end{aligned}
\]
The first identity follows by taking the supremum after integrating
each candidate cost. It does not require a maximizing candidate.

Put $\Delta=\tau-c_0$, $h_t(x)=[V_t^\delta(x,\pi^*)-\tau]_+$, and
\[
\lambda_t(x)=\frac{\Delta}{\Delta+h_t(x)},\qquad
\widetilde\pi_t(\cdot|x)
=\lambda_t(x)\pi^*(\cdot|x)+(1-\lambda_t(x))\delta_{a_0}.
\]
Since $c^*(x,\pi^*)\le\tau$, equation
\eqref{eq:infinite_cost_width_comparison} gives
$h_t(x)\le w_t(x,\pi^*)$. If $h_t(x)=0$, the optimal policy itself
is robustly feasible. If $h_t(x)>0$, the mixture's upper cost is
$c_0+\lambda_t(x)(\Delta+h_t(x))=\tau$. Thus the mixture belongs to
$\Pi_t^\delta(x)$ in both cases.

Rewards lie in $[0,1]$, without requiring knowledge of the safe action's
reward. We therefore obtain
\begin{equation}\label{eq:infinite_safe_mixture_reward}
\begin{aligned}
f^*(x,\pi^*)-f^*(x,\widetilde\pi_t)
&=(1-\lambda_t(x))
  [f^*(x,\pi^*)-f^*(x,a_0)]\\
&\le1-\lambda_t(x)
 =\frac{\lambda_t(x)h_t(x)}{\Delta}
 \le\alpha_rw_t(x,\widetilde\pi_t).
\end{aligned}
\end{equation}
We integrate this inequality and apply
Lemma~\ref{lemma:infinite_action_1} to $\widetilde\pi_t$. Its robust
feasibility lets us compare its score with the maximizing score of
$\pi_t$ by Lemma~\ref{lemma:equivalent_infinite}.
The resulting two inequalities give
\[
\begin{aligned}
\EE_Xf^*(X,\pi^*)
&\le\EE_X[f^*(X,\widetilde\pi_t)
                 +\alpha_rw_t(X,\widetilde\pi_t)]\\
&\le\EE_X[\hat f^t(X,\widetilde\pi_t)
 +\beta_tq_t^\infty(X,\widetilde\pi_t)
 +\alpha_rw_t(X,\widetilde\pi_t)]+\frac{\bar{\cM}\beta_t}{4t}\\
&\le\EE_X[\hat f^t(X,\pi_t)+\beta_tq_t^\infty(X,\pi_t)
 +\alpha_rw_t(X,\pi_t)]+\frac{\bar{\cM}\beta_t}{4t}.
\end{aligned}
\]
The context-dependent mixing coefficient stays inside each expectation.
\end{proof}

\subsection{Completion of the Infinite-Action Theorem}
The remaining step is to sum the one-round regret bound.
Lemma~\ref{lemma:optimism_bound_infinite} compares the optimal
feasible reward with the algorithm's score, and
Lemma~\ref{lemma:infinite_action_1} controls the reward estimation
error of its selected policy. Their combination leaves a cumulative
reward bonus, bounded by Assumption~\ref{ass:policy_div2}, and a
cumulative cost width, bounded by
Property~\ref{prop:sequential_width}.

\begin{proof}[Proof of Theorem~\ref{thm:main_infinite_guarantee}]
Work on $\cE_{C,\mathrm{all}}\cap\cE_R$. The safe point mass belongs to every robust
feasible set, and truth inclusion ensures
$c^*(x,\pi_t)\le V_t^\delta(x,\pi_t)\le\tau$ at every round.
In particular, no unsafe initialization is required.

For each fixed context and realized history, Assumption~\ref{ass:policy_div2}
applies to the entire sequence of archived kernels.
Lemma~\ref{lemma:equivalent_infinite} identifies these kernels with
the policies actually used by the algorithm.
We integrate the nonnegative finite sum to obtain
\begin{equation}\label{eq:infinite_cumulative_potential}
\sum_{t=1}^TQ_t^\infty(\pi_t)
\le C_V\EE_X\cM_X(1+\log T)^{p_V+1}=P_T.
\end{equation}
This is a deterministic application of the assumed potential bound;
there is no new probability event.

Subtract the true reward of $\pi_t$ from the optimism inequality in
Lemma~\ref{lemma:optimism_bound_infinite} and apply
Lemma~\ref{lemma:infinite_action_1} once more to $\pi_t$. We get
\[
\EE_X[f^*(X,\pi^*)-f^*(X,\pi_t)]
\le2\beta_tQ_t^\infty(\pi_t)+\frac{\bar{\cM}\beta_t}{2t}
+\alpha_r\EE_Xw_t(X,\pi_t).
\]
\begin{samepage}
The inequality includes $t=1$. Since
$\beta_t\le\sqrt{t(\bar H_T+D_0)/\bar{\cM}}$ and
$\sum_{t=1}^Tt^{-1/2}\le2\sqrt T$, summation yields
\[
\begin{aligned}
\operatorname{Reg}(T)
&\le2\sqrt{\frac{T(\bar H_T+D_0)}{\bar{\cM}}}P_T
 +\frac{\sqrt{\bar{\cM}(\bar H_T+D_0)}}2\sum_{t=1}^Tt^{-1/2}
 +\alpha_r\sum_{t=1}^T\EE_Xw_t(X,\pi_t)\\
&\le\left(1+\frac{2P_T}{\bar{\cM}}\right)
\sqrt{\bar{\cM}T(\bar H_T+D_0)}
 +\alpha_r\mathcal W_T .
\end{aligned}
\]
\end{samepage}
For the last step we use Property~\ref{prop:sequential_width}.
Its design matrix uses the actual selected-action covariates $Z_i$,
so the same cost-width proof applies on an infinite action space.
All four allocated failure probabilities sum to at most $\delta$.
Thus feasibility and the stated regret bound hold on the same event
with probability at least $1-\delta$, as claimed in
Theorem~\ref{thm:main_infinite_guarantee}.
\end{proof}

Under Assumptions~\ref{ass:reward_entropy} and~\ref{ass:cost_entropy}
and the concrete oracle-error conditions in
\eqref{eq:concrete-rate-conditions}, we now specify the corresponding
constants and logarithmic exponents for
Theorem~\ref{thm:main_infinite_guarantee}. Let
$\ell_T=\log(eT/\delta)$ and use $C_H,C_R,C_L,h_F,\gamma_C$ from
the calculation following the proof of Theorem~\ref{thm:main_guarantee}
in Appendix~\ref{app:proofs1}. The same reward and cost estimates give
\[
\bar H_T+D_0\le(C_H+D_0)\ell_T^{h_F},\qquad
\mathcal W_T\le1+\sqrt{C_RC_LT}\,\ell_T^{\gamma_C}.
\]
The additional contribution comes from the reward-design potential.
Since $1+\log T\le\ell_T$, Assumption~\ref{ass:policy_div2} gives
\[
\frac{P_T}{\bar{\cM}}
=C_V\frac{\EE_X\cM_X}{\bar{\cM}}(1+\log T)^{p_V+1}
\le C_V\ell_T^{p_V+1}.
\]
Substituting these bounds into Theorem~\ref{thm:main_infinite_guarantee},
and using $\alpha_r=1/\Delta$, we obtain on the same confidence event
\begin{equation}\label{eq:infinite-concrete-regret}
\operatorname{Reg}(T)
\le(2C_V+1)\sqrt{\bar{\cM}(C_H+D_0)T}\,
      \ell_T^{p_V+1+h_F/2}
 +\frac{1+\sqrt{C_RC_LT}\,\ell_T^{\gamma_C}}{\Delta}.
\end{equation}
This makes the dependence on the design parameters and safety gap
explicit. For fixed $D_0,C_V,\bar{\cM},p_V$, positive safety gap,
and class and oracle-error parameters, the rate is
$\widetilde{\cO}(\sqrt T)$. The continuous-action example in
Appendix~\ref{app:rank-two-example} verifies the design condition
and evaluates this guarantee for a specific nonlinear model.

\paragraph{Recovery of the finite-action guarantee.}
For a tabular reward class, choose the uniform reference design,
$d_x=\cM_x=K$, and $D_0=\bar{\cM}=K$. Then
$q_t^\infty(x,a)=1/(1+\sum_{i<t}\pi_i(a\mid x))$.
In this case, Lemma~\ref{lemma:estimation_error} gives the sharper
confidence bound $\sqrt{H_tQ_t^\infty(\pi)}$: its proof uses
$\pi(a\mid x)\le1$ to bound the reference contribution by one,
rather than summing it over all $K$ actions. We may therefore
replace $H_t+D_0$ by $H_t$ in the confidence argument and choose
$\beta_t=\sqrt{tH_t/K}$. The one-round regret inequality above
then becomes
\[
\EE_X[f^*(X,\pi^*)-f^*(X,\pi_t)]
\le2\beta_tQ_t^\infty(\pi_t)+\frac{K\beta_t}{2t}
+\alpha_r\EE_Xw_t(X,\pi_t).
\]
Using the exact finite-action potential from
Lemma~\ref{lemma:potential_lemma}, we obtain
\[
\operatorname{Reg}(T)
\le\sqrt{KT\bar H_T}
\left[1+\frac{2\log(1+T/K)}{\log2}\right]
+\alpha_r\mathcal W_T,
\]
which is Theorem~\ref{thm:main_guarantee}. The cost confidence-set construction and the
policy objective are unchanged. This sharper specialization uses the
tabular confidence bound; substituting $D_0=K$ into the general
confidence bound alone would retain an unnecessary additive $K$
inside the square root.

\section{Pseudo-Code for the Infinite-Action Algorithm}\label{app:pseudo-codes}
The following procedure uses the same cost confidence-set construction and selected-action
datasets as the finite-action algorithm, replacing its reward bonus
by~\eqref{eq:infinite_bonus}. Replay fixes all archived predictors,
confidence objects, coefficients, and tie-breaking rules, including
any recorded instruction to choose $a_0$. The reference design contributes no observations;
Lemma~\ref{lemma:equivalent_infinite} establishes replay equivalence.

\begin{algorithm}[!ht]
\setcounter{AlgoLine}{0}\SetAlgoLined\LinesNumbered
\caption{\textbf{CS-UCCB with infinite actions}}
\label{alg:UCCB_infinite_action_pseudo_code}
\KwIn{Classes $\cF,\cC$, horizon $T$, confidence $\delta$, safe
action $a_0$, designs $\rho_x^*$, reference masses $d_x$, and
$D_0,\bar{\cM}$ from Assumption~\ref{ass:policy_div2}.}
Initialize $\cS_1=\cC$, with no instruction yet to choose $a_0$ in all remaining rounds\;
Set $\alpha_r=1/(\tau-c_0)$, use the budgets
in~\eqref{eq:oracle_budgets}--\eqref{eq:reward_budget}, and
set $\beta_t=\sqrt{t(H_t+D_0)/\bar{\cM}}$ for every $t$\;
\For{$t=1,\ldots,T$}{
Fit $\hat f^t$ to $\{(Z_s,R_s)\}_{s<t}$ and $\hat c^t$ to
$\{(Z_s,Y_s)\}_{s<t}$ at confidence $\delta/(4T)$; at $t=1$,
use any $\hat f^1\in\cF$ and no cost fit\;
Run Algorithm~\ref{alg:cost_confidence} and archive its defined
cost set, $\hat f^t$, $\beta_t$, and any instruction to choose $a_0$ in this and all remaining rounds\;
Observe $X_t$\;
\For{$i=1,\ldots,t$}{
\eIf{the archive at $i$ records the instruction to choose $a_0$}{
Set $\widetilde\pi_{t,i}=\pi_0$\;
}{
Set $\widetilde\rho_{t,i-1}
 =(d_{X_t}\rho_{X_t}^*+\sum_{j<i}\widetilde\pi_{t,j})/(d_{X_t}+i-1)$ and
$\widetilde q_{t,i}(a)
 =\cV_{X_t}(a\|\widetilde\rho_{t,i-1})/(d_{X_t}+i-1)$\;
Choose $\widetilde\pi_{t,i}\in\Pi_i^\delta(X_t)$ maximizing
$\hat f^i(X_t,p)+\beta_i\int\widetilde q_{t,i}(a)\,p(da)
+\alpha_rw_i(X_t,p)$, using the archived cost set and historical
tie-breaking rule\;
}
}
Set $\pi_t(\cdot\mid X_t)=\widetilde\pi_{t,t}$, sample
$A_t\sim\pi_t(\cdot\mid X_t)$, and observe $R_t,Y_t$\;
Append $(Z_t,R_t,Y_t)$, with $Z_t=(X_t,A_t)$, to the real dataset\;
}
\end{algorithm}
\FloatBarrier

\section{Nonlinear Examples and Their Regret Bounds}
\label{app:parametric-entropy}
In this section, we give nonlinear examples to illustrate our model
and show how its assumptions can be verified. We construct a convex
cost class with nonlinear context dependence and use it in both
finite-action and continuous-action problems. For these examples,
we verify the covering-number and offline regression conditions and
derive the resulting regret bounds. The continuous-action example
also verifies the reward-design condition in
Assumption~\ref{ass:policy_div2}. Throughout this section, write
$\ell_T=\log(eT/\delta)$.

\subsection{A Convex Cost Class with Nonlinear Context Dependence}
\label{app:convex-cost-example}
Let $x\in[-1,1]$, $a\in[0,1]$, and
$\psi_j(x)=\cos(j\pi x)$. We define the cost class directly on
context--action pairs:
\begin{equation}\label{eq:convex-series-cost-class}
\mathcal C=
\left\{
c_\theta(x,a)=\frac a2+
\frac a{16}\sum_{j=1}^{\infty}2^{-j}\theta_j\psi_j(x):
\theta\in[-1,1]^{\mathbb N}
\right\}.
\end{equation}
The series converges uniformly. Since its coefficient set is convex,
$\lambda c_\theta+(1-\lambda)c_{\theta'}
=c_{\lambda\theta+(1-\lambda)\theta'}$ belongs to $\mathcal C$.
Moreover,
\begin{equation}\label{eq:convex-series-cost-range}
\frac{7a}{16}\le c_\theta(x,a)\le\frac{9a}{16},
\qquad c_\theta(x,0)=0.
\end{equation}
All costs are bounded and share the known safe action $a_0=0$.
The linear independence of the Fourier functions shows that the affine
span is infinite-dimensional.

The parameter cube is compact in its countable product topology.
The uniformly summable coefficients make the map from this cube to
$\mathcal C$ continuous in the uniform norm: convergence of a finite
prefix controls the corresponding finite sum, and the remaining tail
is uniformly small. Thus $\mathcal C$ is compact in the uniform norm.

We next verify the entropy condition. For $0<u\le1$, let
$J=\lceil\log_2(1/u)\rceil$ and set all coordinates beyond $J$ to zero.
The discarded tail changes the function by at most
$2^{-J}/16\le u/16$. In each retained coordinate use a grid on $[-1,1]$
whose mesh is at most $u$ and which has at most $4/u$ points.
Rounding the first $J$ coordinates changes the function by at most
\[
\frac1{16}\sum_{j=1}^J2^{-j}u\le u/16.
\]
Every grid function remains in $\mathcal C$. This proper net has radius
at most $u/8$, and hence is also a proper $u$-net. Its cardinality gives
\begin{equation}\label{eq:convex-series-entropy}
\log \cN(\mathcal C,u,\|\cdot\|_\infty)
\le J\log(4/u)\le2[\log(e/u)]^2.
\end{equation}
At $u=1$, the truncation has no coordinates and the inequality holds
directly. This verifies the cost entropy assumption with
$D_{\cC}=2$, $A_{\cC}=e$, and $p_{\cC}=2$.

There is also an explicit approximation space for cost differences.
The tail of $c_\theta-c_{\theta'}$ after $J$ coordinates has norm at most
$2^{-J}/8$. The span of
$a\psi_1(x),\ldots,a\psi_J(x)$ therefore approximates every difference
to this precision. We obtain $m_T=\cO(\log(eT))$ without imposing a
finite-dimensional representation on the full cost class.

For a finite-dimensional variant, fix $J_0\ge2$ and require
$\theta_j=0$ for $j>J_0$. The resulting convex class
$\mathcal C^{(J_0)}$ has affine dimension $J_0$ and satisfies
$\log \cN(\mathcal C^{(J_0)},u,\|\cdot\|_\infty)\le J_0\log(4/u)$.
The same range, safe action, and decision calculations below apply.
Its cost oracle is a box-constrained quadratic program with $J_0$
variables. Thus the construction covers both fixed finite dimension
and an infinite-dimensional limit.

\subsection{A Finite-Action Problem with an Active Safety Constraint}
\label{app:nonlinear-example}
Let $X\sim\operatorname{Unif}([-1,1])$, $\mathcal A=\{0,1\}$,
$\tau=1/4$, and $c_0=0$. Restrict
\eqref{eq:convex-series-cost-class} to these two actions.
For rewards, put $\ell(z)=(1+e^{-z})^{-1}$ and use
\[
\mathcal F=
\left\{f_b(x,a)=\frac12+\frac a4[2\ell(x+b)-1]:
b\in[-1,1]\right\}.
\]
These rewards lie in $[1/4,3/4]$. After selecting $A_t$, we observe
reward and cost labels that, conditional on the selected pair and the
history, are independent Bernoulli variables with respective means
$f_{b^*}(X_t,A_t)$ and $c_{\theta^*}(X_t,A_t)$.
The algorithm receives no labels for other actions.

Writing $p$ for the probability of action one, feasibility at context
$x$ is exactly
$p\,c_{\theta^*}(x,1)\le1/4$.
Equation~\eqref{eq:convex-series-cost-range} ensures that its upper
feasible probability lies strictly between zero and one. Except at the
single context $x=-b^*$, the unique optimal probability is
\begin{equation}\label{eq:finite-example-optimal-policy}
p^*(x)=
\begin{cases}
0,&x+b^*<0,\\
1/[4c_{\theta^*}(x,1)],&x+b^*>0.
\end{cases}
\end{equation}
For $b^*=0$, both regions have positive context probability.
The reward parameter determines which contexts receive the costly
action, and the unknown cost function determines its optimal positive
probability. Thus neither learning problem is vacuous.

Since $\ell'\le1/4$,
$\|f_b-f_{b'}\|_\infty\le|b-b'|/8$.
A parameter grid of mesh at most $u$ gives
$\log \cN(\mathcal F,u,\|\cdot\|_\infty)\le\log(4/u)$ for $0<u\le1$.
Thus Assumption~\ref{ass:reward_entropy} holds with
$D_{\cF}=1$, $A_{\cF}=4$, and $p_{\cF}=1$.
The regression construction below supplies the required oracle budgets
for both classes. It follows that all assumptions of
Theorem~\ref{thm:main_guarantee} hold, with
$B_T^{\cF}=\cO(\ell_T)$, $B_T^{\cC}=\cO(\ell_T^2)$, and
$\alpha_r=4$.

We also verify policy optimization in this example. For a nonempty
convex confidence set, write
$\bar c_t(x)=\sup_{c\in\mathcal S_t}c(x,1)$ and
$\underline c_t(x)=\inf_{c\in\mathcal S_t}c(x,1)$.
The feasible probabilities form the closed interval
$[0,1/(4\bar c_t(x))]$, and
$w_t(x,1)=\bar c_t(x)-\underline c_t(x)$.
The score is affine in $p$, so choosing its maximizing endpoint with
a fixed tie rule gives a policy. We verify that its coefficients are
measurable. The current squared-error constraint is continuous in
$c$ under the uniform norm. Thus the nonempty set $\mathcal S_t$ is
a closed subset of the compact class $\mathcal C$, and both extrema
are attained.

For completeness, fix a countable dense subset $D\subset\mathcal C$ and put
\[
L_t(c)=\sum_{i<t}(c(Z_i)-\hat c^t(Z_i))^2-B_t^{\cC}
\qquad(t\ge2),
\]
with $L_1(c)=-1$. Each $L_t(c)$ is measurable in the history and
continuous in $c$. Compactness and density imply
$\mathcal S_t\ne\varnothing$ if and only if
$\inf_{c\in D}L_t(c)\le0$, so the empty-set check is measurable.
On nonempty histories, compactness and density give
\[
\bar c_t(x)=\lim_{k\to\infty}
\sup\{c(x,1):c\in D,\ L_t(c)<1/k\}.
\]
Indeed, density approximates every feasible function within each strict
relaxation, while compactness gives a feasible limit of approximate
maximizers as $k\to\infty$. Each supremum is countable, hence measurable
in the history and queried context. Applying the same argument to
$-c(x,1)$ proves measurability of the lower endpoint. Consequently the
endpoint policy with its fixed tie rule is measurable. The theorem gives feasible policies from round one and
regret $\cO(\sqrt T\,\ell_T^2)$ for this example.

\subsection{Least-Squares Oracles for Selected-Action Data}
\label{app:example-oracle}
We give a common oracle guarantee for the examples rather than assuming
an additional independence between past actions and observations.
Let $\cK\subset[0,1]^{\mathcal Z}$ be either working class, with
true function $h^*\in\cK$. Suppose each label $Y_i\in[0,1]$ has
conditional mean $h^*(Z_i)$ given the information available after
selecting $Z_i$ and before observing $Y_i$. The covariates may have been
chosen using all earlier observations.

Let $n=N-1\ge1$ and let $\widehat h\in\cK$ satisfy
\[
\sum_{i=1}^n(Y_i-\widehat h(Z_i))^2
\le\inf_{h\in\cK}\sum_{i=1}^n(Y_i-h(Z_i))^2+\frac12.
\]
If $\mathsf H_{\cK}(u)$ is a known upper bound on
$\log \cN(\cK,u,\|\cdot\|_\infty)$, then a valid error budget is
\begin{equation}\label{eq:example-oracle-budget}
\Estoff(\cK,N,\rho)=
\begin{cases}
0,&N=1,\\
2\mathsf H_{\cK}(1/(N-1))
+2\log(1/\rho)+7,&N\ge2.
\end{cases}
\end{equation}
We may take a nonincreasing entropy envelope in its radius, which makes
this budget nondecreasing in sample size and nonincreasing in the
failure probability $\rho$.

To prove the bound, write $w_i=Y_i-h^*(Z_i)$.
For a fixed function $h$ in a proper $1/n$-net $V$, set
\[
e_h(i)=h(Z_i)-h^*(Z_i),\quad
E_h=\sum_{i=1}^n e_h(i)^2,\quad
Z_h=\sum_{i=1}^n w_i e_h(i).
\]
Conditional on the information before the label, $w_i$ has mean zero
and lies in an interval of length one. Hoeffding's lemma gives
\[
\EE[\exp(\lambda w_i e_h(i))\mid
       \text{information before }Y_i]
\le\exp(\lambda^2 e_h(i)^2/8).
\]
Using $\lambda=2$ and iterating yields
$\EE\exp(2Z_h-E_h/2)\le1$.
Markov's inequality followed by the finite union bound therefore gives,
with probability at least $1-\rho$,
\begin{equation}\label{eq:example-oracle-net}
2Z_h\le E_h/2+\log(|V|/\rho)\qquad(h\in V).
\end{equation}
Using $h^*$ as a comparator in the approximate least-squares guarantee gives
$E_{\widehat h}\le2Z_{\widehat h}+1/2$.
On the common net event choose $h\in V$ within $1/n$ of $\widehat h$.
Since both functions and $h^*$ lie in $[0,1]$, and $|w_i|\le1$,
\[
|E_h-E_{\widehat h}|\le2,\qquad
|Z_h-Z_{\widehat h}|\le1.
\]
We substitute these two estimates into \eqref{eq:example-oracle-net}:
\[
E_{\widehat h}
\le2Z_h+\frac52
\le\frac12E_h+\log(|V|/\rho)+\frac52
\le\frac12E_{\widehat h}+\log(|V|/\rho)+\frac72.
\]
Rearranging proves \eqref{eq:example-oracle-budget}.
The event is uniform over the net before selecting the fitted function;
no claim that the fit was fixed during data collection is needed.

For the finite-dimensional reward families, exact least squares over
the compact parameter interval or square attains its minimum.
The objective is continuous in the parameters and measurable in the
data, so we fix a measurable minimizing selection. This is a proper
regression oracle, even though the reward family need not be convex.

For the cost class \eqref{eq:convex-series-cost-class}, the oracle can
use only finitely many coefficients. Take
$J_n=\lceil\log_2(n\vee2)\rceil$, fix $\theta_j=0$ for $j>J_n$, and
solve the box-constrained quadratic program
\begin{equation}\label{eq:series-cost-regression}
\min_{\theta\in[-1,1]^{J_n}}
\sum_{i=1}^n
\left[
Y_i-\frac{A_i}{2}
-\frac{A_i}{16}\sum_{j=1}^{J_n}2^{-j}\theta_j\psi_j(X_i)
\right]^2.
\end{equation}
A fixed minimum-norm selection among minimizers makes the output
measurable. Truncating any full-class function changes each prediction
by at most $2^{-J_n}/16\le1/(16n)$. Squared loss on $[0,1]$ is
$2$-Lipschitz in its prediction, so the total loss increases by at most
$1/8$. Hence the truncated minimizer is a $1/2$-approximate minimizer
over the full class and satisfies the preceding oracle guarantee.
This procedure uses $\cO(\log(en))$ variables and returns a member of
the original convex class.

Applying \eqref{eq:convex-series-entropy} and the displayed reward
covering bounds to \eqref{eq:example-oracle-budget}, we obtain
polylogarithmic budgets with nonzero observation noise. The same proof
applies to the continuous-action example below, because its labels
are also observed only at the selected context--action pair.

\subsection{Continuous Actions with Two Independent Reward Features}
\label{app:rank-two-example}
We now give a continuous-action example in which the optimal decisions
range over an interval and the reward differences span two independent
action features. The cost class remains the infinite-dimensional convex
class in \eqref{eq:convex-series-cost-class}.

\paragraph{Model and optimal decisions.}
Let $X\sim\operatorname{Unif}([-1,1])$, $\mathcal A=[0,1]$,
$a_0=0$, $c_0=0$, and $\tau=1/4$. With
$\ell(z)=(1+e^{-z})^{-1}$, take
\begin{equation}\label{eq:continuous-example-rewards}
\mathcal F^{(2)}=
\left\{
f_{b_1,b_2}(x,a)=
\frac12+\frac a4[2\ell(x+b_1)-1]
-\frac{a^2}{8}\ell(x+b_2):
(b_1,b_2)\in[-1,1]^2
\right\}.
\end{equation}
These functions lie in $[1/8,3/4]$. After choosing $A_t$, we observe
reward and cost labels that, conditional on the selected pair and the
history, are independent Bernoulli variables with respective means
$f_{b_1^*,b_2^*}(X_t,A_t)$ and $c_{\theta^*}(X_t,A_t)$.

Let $s_j(x)=\ell(x+b_j^*)$ and
$\mu_{\max}^*(x)=1/[4c_{\theta^*}(x,1)]\in[4/9,4/7]$.
A policy is feasible exactly when its mean action
$\mu(x)=\EE_\pi a$ is at most $\mu_{\max}^*(x)$. Since $s_2(x)>0$, Jensen's inequality gives
\[
f^*(x,\pi)
=\frac12+\frac{2s_1(x)-1}{4}\mu(x)
-\frac{s_2(x)}8\EE_\pi a^2
\le\frac12+\frac{2s_1(x)-1}{4}\mu(x)
-\frac{s_2(x)}8\mu(x)^2.
\]
Equality holds only when the action is constant $\pi$-almost surely. Maximizing this strictly
concave quadratic on $[0,\mu_{\max}^*(x)]$ yields the unique optimal policy
\begin{equation}\label{eq:continuous-example-optimal-action}
\pi^*(\cdot|x)=\delta_{a^*(x)},\qquad
a^*(x)=
\min\left\{\mu_{\max}^*(x),
\max\left\{0,\frac{2s_1(x)-1}{s_2(x)}\right\}\right\}.
\end{equation}
Both reward parameters can affect interior optimal actions, while the
unknown cost function determines the binding safety limit.

For a concrete true instance, set $b_1^*=b_2^*=0$ and
$\theta_j^*=0$ for every $j$. Then $c^*(x,a)=a/2$ and
\[
a^*(x)=
\begin{cases}
0,&-1\le x\le0,\\
1-e^{-x},&0<x<\log2,\\
1/2,&\log2\le x\le1.
\end{cases}
\]
The safe, interior, and constraint-binding regions all have positive
context probability. This example cannot be reduced to randomizing
between actions zero and one. On $0<x<\log2$, an endpoint mixture with
mass $p$ at action one has reward
$1/2+p[3\ell(x)-2]/8\le1/2$, whereas the optimal interior action gives
\[
\frac12+\frac{(2\ell(x)-1)^2}{8\ell(x)}>\frac12.
\]
In addition, varying any of the cost coefficients on a region where its
Fourier feature is nonzero changes the safety boundary. Thus the cost
learning problem is not removed by this choice of a simple true
parameter.

\paragraph{Covering numbers and regression oracles.}
The derivative bound $\ell'\le1/4$ gives
\[
\|f_{b_1,b_2}-f_{b_1',b_2'}\|_\infty
\le\frac18|b_1-b_1'|+\frac1{32}|b_2-b_2'|.
\]
Grids of mesh at most $u$ in each coordinate therefore form a proper
$u$-net with at most $(4/u)^2$ elements for $0<u\le1$. Hence
\[
\log \cN(\mathcal F^{(2)},u,\|\cdot\|_\infty)\le2\log(4/u).
\]
This verifies Assumption~\ref{ass:reward_entropy} with
$D_{\cF}=2$, $A_{\cF}=4$, and $p_{\cF}=1$.
The cost entropy is bounded by \eqref{eq:convex-series-entropy}.
We use exact least squares over the reward parameter square and the
finite box-constrained program \eqref{eq:series-cost-regression} for
costs. The proof in Appendix~\ref{app:example-oracle} gives
$B_T^{\cF}=\cO(\ell_T)$ and $B_T^{\cC}=\cO(\ell_T^2)$ for the actual
selected-action datasets. The safe action and cost range are already
verified in \eqref{eq:convex-series-cost-range}.

\paragraph{An optimal reward design.}
Put $\phi(a)=(a,a^2)^\top$. At any fixed context, the coefficient vector
of these two action features ranges over a rectangle with nonempty
interior as $(b_1,b_2)$ varies. Its difference set contains a
neighborhood of zero. Every direction in $\RR^2$ can therefore be
scaled into a realizable reward difference.

For $G_\rho=\int\phi(a)\phi(a)^\top\,d\rho(a)\succ0$, maximizing the
generalized Rayleigh quotient over reward-difference directions gives
\[
\cV_x(a\|\rho)=\phi(a)^\top G_\rho^{-1}\phi(a).
\]
If $G_\rho$ is singular, choose a nonzero vector in its null space.
The corresponding polynomial in $a$ is not identically zero, and a
scaled version is a reward difference. The integral of its squared value
under $\rho$ is zero, while its value is nonzero somewhere, so the design condition
number is infinite. For nonsingular designs,
\[
\int\cV_x(a\|\rho)\,d\rho(a)
=\operatorname{tr}(G_\rho^{-1}G_\rho)=2.
\]
Thus every finite design condition number is at least two.

Consider the context-independent design
\begin{equation}\label{eq:rank-two-optimal-design}
\rho_x^*=\frac12\delta_{1/2}+\frac12\delta_1,\qquad
G_0=\begin{pmatrix}5/8&9/16\\9/16&17/32\end{pmatrix},\qquad
G_0^{-1}=\begin{pmatrix}34&-36\\-36&40\end{pmatrix}.
\end{equation}
Its leverage is $34a^2-72a^3+40a^4$. We compute
\[
2-\phi(a)^\top G_0^{-1}\phi(a)
=2(1-a)(2a-1)^2(5a+1)\ge0,\qquad 0\le a\le1.
\]
Equality holds at $a=1/2$ and $a=1$. The design is therefore optimal,
with $\cV_x(\rho_x^*)=2$, and is measurable.

\paragraph{Potential for every policy sequence.}
Take $d_x=\mathcal M_x=2$ and $D_0=\bar{\cM}=2$.
For any sequence of policies, write
\[
J_t(x)=\int\phi(a)\phi(a)^\top\,d\pi_t(a|x),\qquad
S_{t-1}(x)=2G_0+\sum_{i<t}J_i(x).
\]
The same Rayleigh-quotient calculation gives
$q_t^\infty(x,a)=\phi(a)^\top S_{t-1}(x)^{-1}\phi(a)$ exactly.
Moreover,
$\sup_a\phi(a)^\top(2G_0)^{-1}\phi(a)=1$.
We may therefore apply \eqref{eq:linear-kernel-cumulative-potential}
with $d_{\rm feat}=2$ to obtain
\begin{equation}\label{eq:rank-two-cumulative-potential}
\sum_{t=1}^Tq_t^\infty(x,\pi_t)
\le4\log(1+T/2)\le4(1+\log T).
\end{equation}
This proves Assumption~\ref{ass:policy_div2} with $C_V=2$ and $p_V=0$,
uniformly over contexts and every policy sequence. The reference design
is used only in the bonus and is never an executed initialization policy.

\paragraph{Continuous policy optimization.}
For a nonempty convex confidence set define
\[
\bar c_t(x)=\max_{c\in\mathcal S_t}c(x,1),\qquad
\underline c_t(x)=\min_{c\in\mathcal S_t}c(x,1),\qquad
d_t^c(x)=\bar c_t(x)-\underline c_t(x).
\]
The set is compact: it is the full compact cost class at $t=1$,
and its intersection with one closed squared-error sublevel set
at $t\ge2$. Thus both extrema exist.
Every candidate is linear in $a$ with its context-dependent
coefficient, so
\[
V_t^\delta(x,\pi)=\bar c_t(x)\EE_\pi a,\qquad
w_t(x,\pi)=d_t^c(x)\EE_\pi a.
\]
The feasible mean is consequently bounded by
$\mu_{t,\max}(x)=1/(4\bar c_t(x))\in[4/9,4/7]$.

Choose $\hat f^1=f_{0,0}$ and let
$(\widehat b_{1,t},\widehat b_{2,t})$ denote the fitted parameters
thereafter. The objective in \eqref{eq:infinite_policy} integrates the
continuous polynomial
\[
\begin{aligned}
P_t(x,a)={}&\frac12+
 \frac a4[2\ell(x+\widehat b_{1,t})-1]
 -\frac{a^2}{8}\ell(x+\widehat b_{2,t})\\
&+\beta_t\phi(a)^\top S_{t-1}(x)^{-1}\phi(a)
 +\alpha_r d_t^c(x)a .
\end{aligned}
\]
We show that a maximizing distribution needs at most two support
points. The attainable pairs
$(\EE_\pi a,\EE_\pi P_t(x,a))$ form the compact convex hull of the
graph $(a,P_t(x,a))$. At a maximizing mean, a maximizing pair lies
on the upper boundary of this hull. At mean zero it is represented
by $\delta_0$. At an interior mean, a supporting line meets the graph
at contact points whose convex hull contains the maximizing pair.
The smallest and largest contact abscissas suffice to represent its
mean and value. Since $\mu_{t,\max}(x)<1$, these cases cover every optimizer.

It thus suffices to solve the finite-dimensional problem
\[
\max_{(p,a_1,a_2)\in[0,1]^3}
\{pP_t(x,a_1)+(1-p)P_t(x,a_2)\}
\quad\text{subject to}\quad
pa_1+(1-p)a_2\le \mu_{t,\max}(x).
\]
Its feasible set is nonempty and compact. As $\mu$ varies in $[4/9,4/7]$,
the feasible correspondence is continuous: if the bound decreases to
$\mu'$, scaling both actions by $\mu'/\mu$ gives a feasible triple converging
to the original one; if it increases, the original triple remains
feasible. The coefficients of the objective and $\mu_{t,\max}(x)$ are measurable.
The countable dense-set argument in Appendix~\ref{app:nonlinear-example}
shows that the cost extrema are measurable in the history and queried
context. For each fixed history, the uniformly convergent Fourier series
also makes them continuous in $x$. The archived policy integrals and
the matrix inverse are measurable as well. The measurable maximum theorem therefore gives
a measurable maximizing triple, which we fix as the tie-breaking rule.

This verifies all optimization and statistical requirements of
Theorem~\ref{thm:main_infinite_guarantee}. Its reward term is
$\cO(\sqrt T\,\ell_T^{3/2})$ and its cost term is
$\cO(\sqrt T\,\ell_T^2)$. We obtain
$\cO(\sqrt T\,\ell_T^2)$ regret and feasible policies in every round
for a noisy continuous-action problem with an infinite-dimensional
convex cost class and two independent nonlinear reward features.

\section{Comparison with the Published Regret Bound of Pacchiano et al.}
\label{app:pacchiano-comparison}

We construct unknown reward and cost classes satisfying our setup and
compare our guarantee with the published bound of
\citet{pacchiano2024contextual}. Both learners observe only the selected
action's reward and cost. The reward contributions exhibit a polynomial
gap, while the cost contributions exhibit a logarithmic gap in the
published upper-bound expressions. We state the scope of the cost
comparison explicitly below.
We then explain the role of the cost-class assumptions through a
nonconvex lower bound and a general entropy-to-eluder implication.

\subsection{Model Classes and Verification of the Assumptions}
Fix an integer $M\ge3$ and set
\[
 \varepsilon_M=e^{-M},\qquad T=\lceil e^{2M}\rceil,\qquad
 N_r=\lfloor T^{1/4}\rfloor,\qquad N_c=M^4.
\]
Contexts $X=(I,J)$ are independent across rounds, with independent
uniform coordinates on $[N_r]$ and $[N_c]$. Take
$\cA=\{a_0,a_1\}$ and
$\Lambda_{N_c}=\{v\in\RR_+^{N_c}:\sum_{j=1}^{N_c}v_j\le1\}$.
Define
\begin{equation}\label{eq:comparison-convex-classes}
\begin{aligned}
 f_h((i,j),a_0)&=\frac14,
 &f_h((i,j),a_1)&=\frac18+\frac14\mathbf1\{i\le h\},
 &&h\in\{0,\ldots,N_r\},\\
 c_{\theta,v}((i,j),a_0)&=0,
 &c_{\theta,v}((i,j),a_1)&=\theta+\varepsilon_M v_j,
 &&\theta\in[3/8,5/8],\quad v\in\Lambda_{N_c}.
\end{aligned}
\end{equation}
Write $\cF_{N_r}=\{f_h:0\le h\le N_r\}$ and
$\cC_M=\{c_{\theta,v}:\theta\in[3/8,5/8],\ v\in\Lambda_{N_c}\}$.
After choosing $A_t$, the learner observes conditionally independent
Bernoulli reward and cost labels with the corresponding true means.
Set $\tau=1/2$ and $c_0=0$, so the known safe action satisfies
Assumption~\ref{ass:safe_action}, $\Delta=1/2$, and $\alpha_r=2$.
All functions take values in $[0,1]$; the cost class is compact and
convex. Proposition~\ref{prop:comparison-covering} verifies the
entropy assumptions with cost parameters independent of $M$.

Both unknown functions affect the optimal decision. Its probability
of choosing $a_1$ is
\begin{equation}\label{eq:comparison-convex-optimum}
 p^*_{h,\theta,v}(i,j)=
 \begin{cases}
 0,&i>h,\\
 \min\{1,[2(\theta+\varepsilon_M v_j)]^{-1}\},&i\le h.
 \end{cases}
\end{equation}
Indeed, the reward advantage of $a_1$ is $-1/8$ or $1/8$, and
feasibility requires $p(\theta+\varepsilon_M v_j)\le1/2$. For later use, let
\[
 h^*=\lfloor N_r/2\rfloor,\qquad
 \theta_\pm=\frac9{16}\pm\frac1{64\sqrt T}.
\]
At the concrete truth $(f_{h^*},c_{\theta_-,0})$, the constraint binds on
$\{I\le h^*\}$, whose probability is at least $1/3$. Thus neither
the safe action nor a single fixed policy solves the whole family.
The cost-class diameter is $1/4+\varepsilon_M$, so a diameter-only cumulative
width bound would be linear in $T$.

\paragraph{Offline regression oracles.}
Reward ERM over the $N_r+1$ thresholds has the valid budget
$\Estoff(\cF_{N_r},n,\rho)=8\log((N_r+1)/\rho)$ for $n\ge2$,
by the finite-class argument in Appendix~\ref{app:example-oracle}.
For costs, a scalar fit exploits the size of the coordinate
perturbations. Put $s=n-1\ge1$, let $K_s$ count the observations
with $A_i=a_1$, and clip their average label to $[3/8,5/8]$ to
obtain $\hat\theta$. Return $c_{\hat\theta,0}$; when $K_s=0$,
use any fixed $\hat\theta$ in this interval.

We verify its guarantee under any admissible sampling rule. Let
$S_s=\sum_{i<n:A_i=a_1}(Y_i-c^*(Z_i))$.
Conditional Hoeffding's lemma gives
$\EE\exp(\lambda S_s-\lambda^2K_s/8)\le1$.
For each possible $K_s=k\ge1$, optimize $\lambda$ in both tail
bounds and take a union bound over $k=1,\ldots,s$. With probability
at least $1-\rho$,
\[
 S_s^2/K_s\le\tfrac12\log(2s/\rho)\qquad(K_s>0).
\]
Every selected risky mean lies in $[\theta^*,\theta^*+\varepsilon_M]$.
Projection onto $[3/8,5/8]$ therefore gives
$|\hat\theta-\theta^*|\le\varepsilon_M+|S_s|/K_s$, and hence
\begin{equation}\label{eq:comparison-scalar-oracle}
 \sum_{i<n}(c_{\hat\theta,0}(Z_i)-c^*(Z_i))^2
 \le 2K_s(\hat\theta-\theta^*)^2+2K_s\varepsilon_M^2
 \le 6s\varepsilon_M^2+2\log(2s/\rho).
\end{equation}
For $K_s=0$, the left-hand side is zero.

For a polylogarithmic budget at every horizon for each fixed class,
also consider full ERM, the compact convex quadratic program
\begin{equation}\label{eq:comparison-cost-qp}
 \min_{\theta\in[3/8,5/8],\ v\in\Lambda_{N_c}}
 \sum_{i<n:A_i=a_1}(\theta+\varepsilon_M v_{J_i}-Y_i)^2.
\end{equation}
Proposition~\ref{prop:comparison-covering} and
\eqref{eq:example-oracle-budget} bound its error by
$4[\log(es)]^5+2\log(1/\rho)+7$.
Use the scalar fit when $6s\varepsilon_M^2\le4[\log(es)]^5+7$, and full ERM
otherwise. This choice depends only on sample size and the known class.
A nondecreasing budget for this proper oracle is
\begin{equation}\label{eq:comparison-oracle-budgets}
 \Estoff(\cC_M,n,\rho)
 =2\log(2s/\rho)+\min\{6s\varepsilon_M^2,\ 4[\log(es)]^5+7\},
 \qquad s=n-1\ge1.
\end{equation}
Set both oracle budgets to zero at $n=1$. Since $(T-1)\varepsilon_M^2<1$,
at $\rho=\delta/(4T)$ we obtain
\begin{equation}\label{eq:comparison-cost-oracle-at-horizon}
 B_T^{\cC}\le6+2\log(8T^2/\delta)=\cO(\log(T/\delta)),
\end{equation}
with an absolute constant independent of $M$.
Compactness supplies ERM minima and cost extrema. Each robust policy
problem maximizes a linear function over a closed interval containing
$p=0$; fixed measurable tie-breaking completes the setup.

\subsection{Covering Numbers and Eluder Dimensions}
We first give proper covers and a uniform entropy envelope, then
calculate the eluder dimensions used in the published bound.

\begin{proposition}\label{prop:comparison-covering}
The reward class satisfies
\begin{equation}\label{eq:comparison-reward-entropy}
 \log\cN(\cF_{N_r},u,\|\cdot\|_\infty)=\log(N_r+1)
 \qquad(0<u<1/4).
\end{equation}
For $0<u\le1$, the cost class satisfies
\begin{equation}\label{eq:comparison-cost-cover}
\begin{aligned}
 \log\cN(\cC_M,u,\|\cdot\|_\infty)
 &\le\log\left(1+\frac1{2u}\right)
 +\log\binom{N_c+\lfloor2\varepsilon_M/u\rfloor}{\lfloor2\varepsilon_M/u\rfloor}\\
 &\le2[\log(e/u)]^5.
\end{aligned}
\end{equation}
Thus the entropy assumptions hold with
\begin{equation}\label{eq:comparison-cost-envelope}
 (D_{\cF},A_{\cF},p_{\cF})=(\max\{1,\log(N_r+1)\},e,1),
 \qquad(D_{\cC},A_{\cC},p_{\cC})=(2,e,5).
\end{equation}
The actual cost entropy at a coarser resolution satisfies
\begin{equation}\label{eq:comparison-same-scale}
 \log\cN(\cC_M,\varepsilon_M/8,\|\cdot\|_\infty)=\Theta(M),
\end{equation}
with absolute implied constants.
\end{proposition}

\begin{proof}[Proof of Proposition~\ref{prop:comparison-covering}]
Distinct reward functions have uniform distance $1/4$, so a proper
cover at a smaller radius must contain all $N_r+1$ functions.
Their entropy is at most $\log(N_r+1)$ at every radius.
For costs, round $\theta$ down on a grid of spacing $u/2$ starting
at $3/8$, with at most $1+1/(2u)$ points. Round each coordinate of
$v$ down to a multiple of $u/(2\varepsilon_M)$. The rounded vector remains in
$\Lambda_{N_c}$, and the two rounding errors sum to at most $u$.
The integer coordinate counts have sum at most $\lfloor2\varepsilon_M/u\rfloor$,
which gives the binomial count in~\eqref{eq:comparison-cost-cover}.

Put $z=\log(e/u)\ge1$. If $u>2\varepsilon_M$, the binomial count equals one,
and the scalar cover has entropy at most $z$.
If $u\le2\varepsilon_M$, then $z>M$ and the binomial count is at most
$(1+2\varepsilon_M/u)^{N_c}$. Since
$\log(1+2\varepsilon_M/u)\le\log(2\varepsilon_M/u)+\log2\le z$, we obtain
\[
 \log\cN(\cC_M,u,\|\cdot\|_\infty)
 \le z+M^4z\le2z^5.
\]
This proves the fixed cost parameters in
\eqref{eq:comparison-cost-envelope}.

At $u=\varepsilon_M/8$, the cover has at most
$(1+4/\varepsilon_M)\binom{M^4+16}{16}$ points, giving entropy $\cO(M)$.
Conversely, the scalar subfamily $v=0$ contains a grid of
$\Omega(1/u)$ parameters separated by more than $2u$ within
$[3/8,5/8]$. Every radius-$u$ cover needs that many centers, giving
entropy $\Omega(\log(1/u))=\Omega(M)$ and proving
\eqref{eq:comparison-same-scale}.
\end{proof}

The cost entropy parameters are fixed across this entire family.
The reward coefficient grows as $\log(N_r+1)=\Theta(\log T)$
along the comparison family, and we retain it explicitly in the
reward bound. Below we use the actual cost-oracle budget and an
explicit approximation space to sharpen the generic entropy bound.

To apply the finite-class theorem of
\citet{pacchiano2024contextual} directly, define the smaller cost prior
\begin{equation}\label{eq:comparison-finite-cost-prior}
 \cC_M^{\rm fin}
 =\{c_{3/8,0},c_{3/8,e_1},\ldots,c_{3/8,e_{N_c}},
                  c_{\theta_-,0},c_{\theta_+,0}\}\subseteq\cC_M.
\end{equation}
It has $N_c+3$ elements and contains the concrete truth above.
We use the pairwise eluder dimension in their Definition~36.

\begin{proposition}\label{prop:comparison-eluder}
The reward and cost eluder dimensions are
\begin{equation}\label{eq:comparison-exact-eluder}
 d_E(\cF_{N_r},u)=N_r\quad(0<u<1/4),\qquad
 d_E(\cC_M,u)=d_E(\cC_M^{\rm fin},u)=N_c=M^4
 \quad(0<u<\varepsilon_M).
\end{equation}
\end{proposition}

\begin{proof}[Proof of Proposition~\ref{prop:comparison-eluder}]
For rewards, query $((i,1),a_1)$ for $i=1,\ldots,N_r$.
At the $i$th point, $f_i-f_{i-1}$ equals $1/4$ and vanishes at
all preceding points. For costs, query $((1,j),a_1)$ for
$j=1,\ldots,N_c$. The difference $c_{3/8,e_j}-c_{3/8,0}$ equals
$\varepsilon_M$ at the current point and vanishes at all previous points.
These witnesses belong to the finite prior as well as the convex
class, proving the lower bounds.

For the upper bounds, all differences vanish at $a_0$. At $a_1$,
reward differences depend only on $i$, and cost differences only on
$j$. No independent sequence can repeat such a coordinate: a gap
exceeding its independence threshold would already exceed that
threshold at the earlier occurrence. This holds for every common
threshold $u'\ge u$. There are only $N_r$ reward and $N_c$ cost
coordinates, giving the matching upper bounds.
\end{proof}

\subsection{Comparison of Regret Guarantees}
We evaluate our guarantee using the explicit cost approximation space
and then substitute the two eluder dimensions into the published
OPNLB theorem, which uses resolution $1/T$.

\paragraph{CS-UCCB.}
Choose the analysis radius $r=\varepsilon_M$ in
Corollary~\ref{cor:cost-width-radius} and put
$\psi((i,j),b)=\mathbf1\{b=a_1\}$.
Every cost difference is within $\varepsilon_M$ of
$(\theta-\theta')\psi$, because $|v_j-v'_j|\le1$.
Hence $m(\varepsilon_M)\le1$; the class diameter $1/4+\varepsilon_M>\varepsilon_M$
excludes dimension zero. Therefore $m(\varepsilon_M)=1$.
Since $T\varepsilon_M^2\le2$ and $\varepsilon_M T\le\sqrt{2T}$, the corollary gives
\begin{equation}\label{eq:comparison-refined-width}
 \mathcal W_T(\varepsilon_M)\le\varepsilon_M T+
 \sqrt{T\left[4+8B_T^{\cC}+2T\varepsilon_M^2\right]
             \left[4\log(1+T)+2\log(4/\delta)\right]}
 =\cO\!\left(\sqrt T\log(T/\delta)\right).
\end{equation}
The unknown scalar component is learned, while the remaining
coordinate variation is charged at its actual amplitude.
This choice changes only the analysis radius.

Let $L=\log(4T(N_r+1)/\delta)$. The reward oracle gives
$N_{\cF}=N_r+1$, $B_T^{\cF}\le8L$, and $\bar H_T\le2+18L$.
Using $K=2$ and $\Delta=1/2$ in
\eqref{eq:finite-regret-arbitrary-radius}, we obtain
\begin{equation}\label{eq:our-convex-comparison-bound}
 \operatorname{Reg}_{\textnormal{CS-UCCB}}(T)
 \le\underbrace{\sqrt{2T(2+18L)}
       \left[1+\frac2{\log2}\log(1+T/2)\right]}_{\text{reward contribution}}
       +\underbrace{2\mathcal W_T(\varepsilon_M)}_{\text{cost contribution}}.
\end{equation}
The same event, of probability at least $1-\delta$, gives feasibility
at every context and round.

\paragraph{Published OPNLB bound.}
Apply \citet[Theorem~38]{pacchiano2024contextual} to $\cF_{N_r}$ and
$\cC_M^{\rm fin}$ with the same truth $(f_{h^*},c_{\theta_-,0})$.
Their theorem assumes finite classes; this gives OPNLB the additional
information of the smaller cost prior, whereas our bound holds for
the full convex class. The feedback is identical. Their available
action set can be written as $\{s_0,(I_t,J_t)\}$, where $s_0$
carries the safe-action values and $(i,j)$ carries the values at
$((i,j),a_1)$. This preserves policies and observations.
Since $1/T<\varepsilon_M$, Proposition~\ref{prop:comparison-eluder} gives
$d_E^r=N_r$ and $d_E^c=N_c$. Their guarantee becomes
\begin{equation}\label{eq:pacchiano-literal-convex-family}
\begin{gathered}
 \operatorname{Reg}_{\rm OPNLB}^{\rm ctx}(T)
 \le C\left[
 \underbrace{\sqrt{TN_r\gamma_r}+N_r}_{\text{reward contribution}}
 +\underbrace{\frac32\bigl(\sqrt{TN_c\gamma_c}+N_c\bigr)}_{\text{cost contribution}}
 \right],\\
 \gamma_r=512\log\frac{48(N_r+1)\log(2T)}{\delta},\qquad
 \gamma_c=512\log\frac{48(N_c+3)\log(2T)}{\delta}.
\end{gathered}
\end{equation}
Here $C$ is their absolute constant and
$(1-1/4)/(1/2-0)=3/2$.

\paragraph{A common regret criterion.}
The superscript $\mathrm{ctx}$ denotes regret at the realized
contexts. For either algorithm, its difference from our regret is
$\sum_t[\ell_t(X_t)-\EE_X\ell_t(X)]$, where
$\ell_t(x)=f^*(x,\pi^*)-f^*(x,\pi_t)\in[-1,1]$ and the expectation
fixes the pre-round history. We apply Lemma~\ref{Azuma-Hoeffding}
to the adapted sequence $(\ell_t(X_t))_{t=1}^T$ with $R=1$ and
failure probability $\eta/2$ for each algorithm. A union bound gives
\begin{equation}\label{eq:comparison-regret-conversion}
 \max_{A\in\{\mathrm{CS\text{-}UCCB},\mathrm{OPNLB}\}}
 |\operatorname{Reg}_A^{\rm ctx}(T)-\operatorname{Reg}_A(T)|
 \le\sqrt{2T\log(4/\eta)}
\end{equation}
with probability at least $1-\eta$.
For total failure probability $\zeta$, use $\delta=\zeta/4$ in
both theorems and $\eta=\zeta/2$. Both guarantees then hold on
our regret criterion with probability at least $1-\zeta$, with an
additional $\sqrt{2T\log(8/\zeta)}$ in the OPNLB bound.

For fixed $\zeta$, we have $M=\Theta(\log T)$,
$N_r=\Theta(T^{1/4})$, $L,\gamma_r=\Theta(\log T)$, and
$\gamma_c=\Theta(\log M)=\Theta(\log\log T)$. Thus
\begin{equation}\label{eq:our-convex-comparison-rate}
\begin{aligned}
 \operatorname{Reg}_{\textnormal{CS-UCCB}}(T)
 &\le\cO\!\left(
 \underbrace{T^{1/2}(\log T)^{3/2}}_{\text{reward contribution}}
 +\underbrace{T^{1/2}\log T}_{\text{cost contribution}}\right),\\
 \operatorname{Reg}_{\rm OPNLB}(T)
 &\le\cO\!\left(
 \underbrace{T^{5/8}\sqrt{\log T}}_{\text{reward contribution}}
 +\underbrace{T^{1/2}(\log T)^2\sqrt{\log\log T}}_{\text{cost contribution}}\right).
\end{aligned}
\end{equation}
The reward expressions differ by a factor of order
$T^{1/8}/\log T$, and the cost expressions by
$\log T\sqrt{\log\log T}$. The reward terms dominate both total
expressions asymptotically, giving a total factor
$T^{1/8}/\log T$. These comparisons concern proved upper-bound
expressions, not lower bounds on OPNLB's actual regret. All implied
constants are independent of $M$.

\paragraph{Scope of the cost comparison.}
The logarithmic cost gap concerns the published theorem's $1/T$
resolution. It is not a separation from an eluder analysis with an
optimized cutoff. In fact, $d_E(\cC_M^{\rm fin},2\varepsilon_M)\le2$.
To see this, group the finite prior into the low-cost models
$c_{3/8,0},c_{3/8,e_j}$ and the two high-cost constant models.
Within-group differences are below $2\varepsilon_M$, so every witness must
cross groups. Such a difference equals $\theta_\pm-3/8$ at every risky
coordinate except possibly one, where it is smaller by $\varepsilon_M$.
After two distinct risky coordinates, some historical difference
therefore equals the largest possible current difference. A third
point cannot be independent; repeated coordinates and the safe
action cannot extend an independent sequence either.
Retaining a cutoff $\alpha$ in the usual eluder-width summation gives
$\cO(\alpha T+d_E(\alpha)+\sqrt{T d_E(\alpha)\gamma_c})$.
Choosing $\alpha=2\varepsilon_M$ reduces the cost expression to
$\cO(\sqrt{T\log\log T})$ at fixed confidence.
Thus optimizing that analysis removes the displayed cost gap.

\paragraph{Cost learning is necessary.}
The scalar component requires nontrivial learning in this same
family. Consider any algorithm that is feasible at all contexts
and rounds with probability at least $1-\delta$ for both truths
$(f_{h^*},c_{\theta_-,0})$ and $(f_{h^*},c_{\theta_+,0})$.
Let $P_-$ and $P_+$ be its full transcript laws, including its
randomness. Only the selected risky cost labels change. The chain
rule for relative entropy gives
\[
 \mathrm{KL}(P_-\|P_+)
 \le T\,\mathrm{kl}(\mathrm{Ber}(\theta_-),\mathrm{Ber}(\theta_+))
 \le\frac{T(\theta_--\theta_+)^2}{\theta_+(1-\theta_+)}\le\frac1{192},
\]
since $\theta_+-\theta_-=1/(32\sqrt T)$ and
$\theta_+(1-\theta_+)\ge3/16$. Pinsker's inequality bounds total variation
by $1/\sqrt{384}$. Write $p_t(i,j):=\pi_t(a_1\mid(i,j))$.
The event $p_t(i,j)\le(1/2)/\theta_+$ for all
$t,i,j$ has $P_+$-probability at least $1-\delta$, and hence
$P_-$-probability at least $1-\delta-1/\sqrt{384}$.
On this event, favorable contexts have reward advantage $1/8$
and probability $\mu=h^*/N_r\ge1/3$, so
\[
 \operatorname{Reg}_-(T)
 \ge\frac{\mu T}{8}
       \left(\frac{1/2}{\theta_-}-\frac{1/2}{\theta_+}\right)
 =\frac{\mu\sqrt T}{512\theta_-\theta_+}
 \ge\frac{\sqrt T}{486}.
\]
Unfavorable contexts contribute nonnegative regret. This argument
still holds if the reward function is revealed to the learner;
therefore the square-root cost term reflects necessary cost learning,
not a vanishing diameter or an already known optimal policy.

\subsection{Why Small Covering Entropy Alone Is Insufficient}
\label{app:entropy-only-counterexample}
We now give a separate construction that explains the need for
structure beyond a covering-number bound. Let $M\ge2$ and $T\ge1$
be arbitrary, take $X_t\overset{\mathrm{i.i.d.}}{\sim}
\operatorname{Unif}([M])$, and use two actions $a_0,a_1$.
The reward is known and deterministic:
$f^*(x,a_0)=0$ and $f^*(x,a_1)=1$.
Set $\tau=1/2$ and consider the nonconvex cost class
\begin{equation}\label{eq:spike-cost-class}
\begin{aligned}
\cC_M^{\rm sp}&=\{c^{(0)},c^{(1)},\ldots,c^{(M)}\},\\
c^{(0)}(x,b)&=0,\qquad
c^{(j)}(x,b)=\mathbf1\{x=j,\ b=a_1\},\quad j\in[M].
\end{aligned}
\end{equation}
The learner observes only the selected action's deterministic reward
and cost, after acting. The safe action has known cost zero under
every model, with fixed safety gap $\Delta=1/2$.
Under $c^{(0)}$, always choosing $a_1$ is optimal. Under $c^{(j)}$,
feasibility at context $j$ requires
$\pi_t(a_1\mid j)\le1/2$.

Distinct models in~\eqref{eq:spike-cost-class} have uniform distance
one. With our convention that covering centers lie in the class,
\begin{equation}\label{eq:spike-entropy}
\log\cN(\cC_M^{\rm sp},u,\|\cdot\|_\infty)
=\log(M+1)\qquad(0<u<1).
\end{equation}
In this noiseless instance, consistent least squares has zero
empirical prediction error under every admissible sampling rule:
every returned model agrees with the true cost on all queried
pairs. The known reward also has a zero-error oracle. Nevertheless,
the following lower bound holds for every learning algorithm,
regardless of its computation or access to these regression oracles.

\begin{proposition}\label{prop:spike-safety-lower-bound}
Fix $\delta\in(0,1/4)$. Suppose an algorithm satisfies, for every
$j\in\{0,\ldots,M\}$,
\[
\PP_j\!\left\{c^{(j)}(x,\pi_t)\le\tfrac12
\text{ for every }x\in[M],\ t\le T\right\}\ge1-\delta,
\]
where $\PP_j$ denotes its interaction law when the true cost is
$c^{(j)}$. Then under $c^{(0)}$,
\begin{equation}\label{eq:spike-lower-bound}
\PP_0\!\left\{\operatorname{Reg}(T)
\ge\tfrac18\min\{T,M\}\right\}\ge1-4\delta.
\end{equation}
Moreover,
$\EE_0\operatorname{Reg}(T)\ge
(1/4-\delta/2)\min\{T,M\}$.
\end{proposition}

\begin{proof}[Proof of Proposition~\ref{prop:spike-safety-lower-bound}]
Write $p_t(j)=\pi_t(a_1\mid j)$ and let
$U_j=\inf\{s\le T:(X_s,A_s)=(j,a_1)\}$, with $U_j=\infty$
if this pair is never selected. Define the stopped event
\[
E_j=\{\text{there exists }t\le T\text{ with }t\le U_j
\text{ and }p_t(j)>1/2\}.
\]
The round $t=U_j$ is included because its policy is chosen before
the cost feedback arrives. Couple the null model $c^{(0)}$ and
the spike model $c^{(j)}$ using the same contexts and learner
randomness. All observed labels, histories, and policy choices
coincide until the feedback from the first selected pair $(j,a_1)$.
The event $E_j$ therefore occurs identically under this coupling.
Since it is a safety violation under $c^{(j)}$, we obtain
\[
\PP_0(E_j)=\PP_j(E_j)\le\delta.
\]
Let $B=\sum_{j=1}^M\mathbf1_{E_j}$ under $\PP_0$.
Linearity of expectation gives $\EE_0B\le M\delta$; independence
between these events is not required. Markov's inequality then gives
\begin{equation}\label{eq:spike-good-contexts}
\PP_0(B\le M/4)\ge1-4\delta.
\end{equation}

On each history, by round $t$ at most $t-1$ distinct pairs of the
form $(j,a_1)$ have been observed. Apart from at most $B$ contexts,
every remaining unqueried context has $p_t(j)\le1/2$.
Under $c^{(0)}$, the regret contribution at round $t$ is exactly
$M^{-1}\sum_{j=1}^M(1-p_t(j))$. Hence, with
$n=\min\{T,M\}$ and using nonnegativity of all later contributions,
\begin{equation}\label{eq:spike-pathwise-regret}
\begin{aligned}
\operatorname{Reg}(T)
&\ge\frac12\sum_{t=1}^{n}
\left(1-\frac{t-1}{M}-\frac BM\right)\\
&=\frac12\left[n-\frac{n(n-1)}{2M}-\frac{nB}{M}\right].
\end{aligned}
\end{equation}
On $B\le M/4$, this is at least $n/8$, proving
\eqref{eq:spike-lower-bound}. Taking expectations in
\eqref{eq:spike-pathwise-regret} gives the stated expectation bound.
The calculation uses exactly our regret definition, with the current
context averaged conditional on the pre-round history; no conversion
between regret criteria is needed.
\end{proof}

Taking $M=T$, we obtain a family with logarithmic covering entropy
and zero oracle error, yet linear regret with probability at least
$1-4\delta$. Thus polylogarithmic original-class covering entropy
and offline error alone cannot imply a uniform
$\widetilde{\cO}(\sqrt T)$ regret guarantee. This shows
that additional structure is needed; convexity together with the
entropy condition is the sufficient structure used in our analysis.
The example does not assert that convexity is the only such structure.

Convexification accounts for the difficulty rather than removing it
at the same complexity. In fact,
\[
\operatorname{conv}(\cC_M^{\rm sp})
=\{c_v(x,b)=v_x\mathbf1\{b=a_1\}:
v\in\RR_+^M,\ \textstyle\sum_jv_j\le1\}.
\]
For every subset $S\subseteq[M]$, this hull contains
$v^S=M^{-1}\mathbf1_S$. The $2^M$ resulting functions are pairwise
$1/M$ apart, so a uniform ball of radius $u=1/(8M)$ can contain
at most one of them. Conversely, rounding every coordinate of $v$
down to a multiple of $u$ gives a proper $u$-net with at most
$\binom{9M}{M}\le(9e)^M$ centers. We therefore obtain
\begin{equation}\label{eq:spike-hull-entropy}
M\log2\le
\log\cN(\operatorname{conv}(\cC_M^{\rm sp}),1/(8M),
\|\cdot\|_\infty)
\le M\log(9e).
\end{equation}
Thus the convex hull has entropy of order $M$ at this resolution,
consistent with the convex-hull guarantee in
Remark~\ref{rem:nonconvex-costs}.

\subsection{Cost Entropy and Eluder Dimension}
\label{app:cost-eluder-relation}

Assumption~\ref{ass:cost_entropy} also bounds the pairwise eluder
dimension used in this appendix. We use the convention of
\citet[Definition~36]{pacchiano2024contextual}: $d_E(\cC,\varepsilon)$
is the supremum of the lengths of sequences that are independent
at some common threshold $\varepsilon'\ge\varepsilon$.

\begin{proposition}\label{prop:cost-eluder-entropy}
Under Assumption~\ref{ass:cost_entropy}, there is a finite constant
$C_{p_{\cC}}$, depending only on $p_{\cC}$, such that, for every
$0<\varepsilon\le1$,
\begin{equation}\label{eq:cost-eluder-entropy}
d_E(\cC,\varepsilon)
\le C_{p_{\cC}}D_{\cC}
\left[\log\frac{eA_{\cC}D_{\cC}}{\varepsilon}\right]^{p_{\cC}}.
\end{equation}
In particular, fixed entropy parameters imply
$d_E(\cC,1/T)=\cO_{p_{\cC}}\bigl(
D_{\cC}[\log(eA_{\cC}D_{\cC}T)]^{p_{\cC}}\bigr)$.
\end{proposition}

\begin{proof}[Proof of Proposition~\ref{prop:cost-eluder-entropy}]
Write $D=D_{\cC}$, $A=A_{\cC}$, and $p=p_{\cC}$. Consider any
nonempty independent sequence $z_1,\ldots,z_n$ at a common threshold
$\varepsilon'\ge\varepsilon$. For each $k$, independence supplies
$h_k=c_k-c_k'$ with $c_k,c_k'\in\cC$ such that
\begin{equation}\label{eq:cost-eluder-witnesses}
\sum_{i<k}h_k(z_i)^2\le(\varepsilon')^2,
\qquad |h_k(z_k)|>\varepsilon'.
\end{equation}
Since costs take values in $[0,1]$, such a sequence requires
$\varepsilon'<1$. In particular, $d_E(\cC,1)=0$.

Set $r=\varepsilon/(2\sqrt n)$ and $m=m(r)$, with $m(r)$ as defined
after Lemma~\ref{lemma:entropy-linear-span}. Then $0<r\le1/2$.
If $m=0$, Lemma~\ref{lemma:john-features} gives
$\|h_k\|_\infty\le r<\varepsilon'$, contradicting
\eqref{eq:cost-eluder-witnesses}. Hence $m\ge1$. The same lemma
provides a fixed map $\phi:\cZ\to\RR^m$ and vectors $\theta_k$ with
\[
\|\phi(z)\|_2\le1,\qquad
\|\theta_k\|_2\le2\sqrt m,\qquad
\|h_k-\theta_k^\top\phi\|_\infty\le r.
\]
Define
\[
\lambda=\frac{(\varepsilon')^2}{4m},\qquad
\Sigma_k=\lambda I+\sum_{i<k}\phi(z_i)\phi(z_i)^\top.
\]
The approximation and \eqref{eq:cost-eluder-witnesses} imply
\[
|\theta_k^\top\phi(z_k)|
>\varepsilon'-r\ge\varepsilon'/2
\]
and
\begin{align*}
\theta_k^\top \Sigma_k\theta_k
&=\lambda\|\theta_k\|_2^2
  +\sum_{i<k}(\theta_k^\top\phi(z_i))^2\\
&\le(\varepsilon')^2
  +2\sum_{i<k}h_k(z_i)^2+2(k-1)r^2\\
&\le3(\varepsilon')^2+2nr^2
\le\frac72(\varepsilon')^2.
\end{align*}
The Cauchy--Schwarz inequality in the $\Sigma_k$ norm therefore yields
\[
\ell_k:=\phi(z_k)^\top \Sigma_k^{-1}\phi(z_k)
\ge\frac{(\theta_k^\top\phi(z_k))^2}
         {\theta_k^\top \Sigma_k\theta_k}
>\frac1{14}.
\]
The determinant identity and the arithmetic--geometric mean
inequality, together with $\|\phi(z_k)\|_2\le1$, give
\begin{align*}
n\log(15/14)
&\le\sum_{k=1}^n\log(1+\ell_k)
=\log\frac{\det \Sigma_{n+1}}{\det(\lambda I)}\\
&\le m\log\left(1+\frac{n}{m\lambda}\right)
=m\log\left(1+\frac{4n}{(\varepsilon')^2}\right)
\le m\log\left(1+\frac{4n}{\varepsilon^2}\right).
\end{align*}
As $1/\log(15/14)<15$, this proves
\begin{equation}\label{eq:cost-eluder-implicit}
n\le15\,m\left(\frac{\varepsilon}{2\sqrt n}\right)
\log\left(1+\frac{4n}{\varepsilon^2}\right).
\end{equation}

Let $K_p\ge1$ be a valid constant in
Lemma~\ref{lemma:entropy-linear-span}, and put
\[
L=\log\frac{2eAD}{\varepsilon},\qquad a=30K_pD.
\]
That lemma bounds the approximation dimension by
$K_pD[L+\tfrac12\log n]^{p-1}$. Since $A\ge e$, $D\ge1$,
$n\ge1$, and $\varepsilon\le1$, we also have
\[
\log(1+4n/\varepsilon^2)
\le2\log(2eAD\sqrt n/\varepsilon)
=2[L+\tfrac12\log n].
\]
Substitution into \eqref{eq:cost-eluder-implicit} gives
\[
n\le a[L+\tfrac12\log n]^p.
\]
To remove the remaining dependence on $n$ from the right-hand side,
first suppose $n\ge a$ and set $y=(n/a)^{1/p}\ge1$. Then
\[
y\le L+\tfrac12\log a+\tfrac p2\log y.
\]
Using $\log(y/p)\le y/p-1$ and rearranging gives
\[
y\le2L+\log a-p+p\log p
\le2L+\log a+p\log p.
\]
The last upper bound exceeds one, so the case $n<a$ is covered
as well. Consequently,
\[
n\le30K_pD
\left[2\log\frac{2eAD}{\varepsilon}
      +\log(30K_pD)+p\log p\right]^p.
\]
To absorb constants, write $B=\log(eAD/\varepsilon)\ge1$.
The bracket is at most
$[3+2\log2+\log(30K_p)+p\log p]B$, because $\log D\le B$.
Thus \eqref{eq:cost-eluder-entropy} holds with
\[
C_p=30K_p[3+2\log2+\log(30K_p)+p\log p]^p.
\]
This bound applies to every finite independent sequence and every
common threshold $\varepsilon'\ge\varepsilon$. Taking the supremum
of their lengths proves the claim, including finiteness of the
eluder dimension. If there is no nonempty independent sequence,
the bound holds trivially.
\end{proof}

This general bound need not be sharp for a particular class.
For the comparison family, $(D_{\cC},A_{\cC},p_{\cC})=(2,e,5)$,
so the proposition gives
\[
d_E(\cC_M,1/T)\le2C_5[\log(2e^2T)]^5
=\cO((\log T)^5).
\]
In contrast, Proposition~\ref{prop:comparison-eluder} gives the exact
value $d_E(\cC_M,1/T)=d_E(\cC_M^{\rm fin},1/T)=M^4
=\Theta((\log T)^4)$ when $T=\lceil e^{2M}\rceil$.
We use that exact value in~\eqref{eq:pacchiano-literal-convex-family},
so the general upper bound is consistent with the logarithmic cost
comparison in~\eqref{eq:our-convex-comparison-rate}. That comparison
concerns the published $1/T$ formula, with the cutoff qualification
explained above. The cost assumptions thus control the eluder
dimension without making these particular upper-bound expressions
identical. This consequence of convexity does not apply to the
general nonconvex reward class.

Our results leave the characterization of sufficient structure for
nonconvex cost classes unresolved. The lower bound in
Appendix~\ref{app:entropy-only-counterexample} shows that covering
entropy alone cannot provide the uniform regret guarantee, while
convexity together with the entropy condition provides the structure
used here.

\section{Computation}\label{app:computation}
We describe how to compute the quantities used by CS-UCCB. Its cost
confidence set has one convex squared-error constraint from the
current fit, so support queries are convex optimization problems over
the cost class. We give their explicit formulation and a finite-dimensional
implementation below. Efficient solution requires a tractable class
representation; convexity and covering entropy alone do not provide
one. We count arithmetic operations and oracle calls, and require
numerical implementations to certify feasibility.

\subsection{Computing the Cost Confidence Set and Its Supports}
\label{app:computation-spps}
At $t=1$, we set $\cS_1=\cC$. At each subsequent stage, the current
cost fit defines one convex squared-error constraint:
\[
\cS_t=\left\{c\in\cC:
\sum_{i<t}(c(Z_i)-\hat c^t(Z_i))^2\le B_t^{\cC}\right\}.
\]
We archive each stage's fit, data prefix, and budget to represent its
confidence set during replay. Evaluating the current fit on its full
prefix takes $t-1$ function evaluations; these values can be cached.
Earlier fits specify earlier confidence sets, and impose no additional
constraints on $\cS_t$. Thus each set can be represented implicitly
without enumerating functions. Set emptiness is a convex
feasibility problem; an empty set triggers the safe-action instruction
in Algorithm~\ref{alg:cost_confidence}.

For $b\in\mathbb R^K$, define the support query
\begin{equation}\label{eq:comp-aggregate-support}
M_t(x,b)=\sup_{c\in\cS_t}\sum_a b_ac(x,a).
\end{equation}
This maximizes a linear functional over a convex set. To make the
computational requirements concrete, suppose the class has a known
affine representation
\[
c_\theta(z)=c_{\rm ref}(z)+\phi(z)^\top\theta,
\qquad \theta\in\Theta\subseteq\mathbb R^d,
\]
where $\Theta$ is a tractable convex domain. Let $\Phi_t$ have rows
$\phi(Z_i)^\top$ for $i<t$, and put
$h_t=(\hat c^t(Z_i)-c_{\rm ref}(Z_i))_{i<t}$.
Then~\eqref{eq:comp-aggregate-support} is
\begin{equation}\label{eq:comp-support-convex}
\begin{aligned}
\max_{\theta\in\Theta}\quad&
\sum_a b_a c_{\rm ref}(x,a)
+\left(\sum_a b_a\phi(x,a)\right)^\top\theta\\
\text{subject to}\quad&
\|\Phi_t\theta-h_t\|_2\le\sqrt{B_t^{\cC}}.
\end{aligned}
\end{equation}
At $t=1$, there are no empirical constraints. For a polyhedral
$\Theta$, this is a second-order cone program with one empirical
norm constraint containing $t-1$ residuals; other convex representations
give the corresponding constrained quadratic or conic formulation.
Efficient convex solvers
therefore apply when the domain and its numerical encoding permit
them. Least-squares regression over this affine class is also a
convex optimization problem.

The empirical constraint can instead be stored in the quadratic form
\[
\theta^\top G_t\theta-2g_t^\top\theta+\kappa_t\le B_t^{\cC},
\]
where $G_t=\Phi_t^\top\Phi_t$, $g_t=\Phi_t^\top h_t$, and
$\kappa_t=\|h_t\|_2^2$.
The feature Gram matrix $G_t$ admits a rank-one update. The vector
$h_t$ uses the current predictor on every past sample, so $g_t$ and
$\kappa_t$ are recomputed after each refit. Archiving these coefficients
is an alternative to retaining every residual vector.

A support routine used for safe policy optimization must return a
certified upper value, and a violating model when separation is
required. These requirements also apply to implicit or
infinite-dimensional classes: their support queries need a suitable
optimization procedure. A finite-dimensional approximation is usable
only if its error is included in the safety certificate.

\subsection{Computing Cost Widths and Robustly Feasible Policies}
\label{app:computation-lp}
The support values give the quantities in the policy objective:
\[
V_t^\delta(x,p)=M_t(x,p),\qquad
w_t(x,a)=M_t(x,e_a)+M_t(x,-e_a).
\]
We therefore use $2K$ support queries to form the width vector.
During replay, these queries use the archived stage's own current-fit
confidence set.
Given the reward bonus vector $q_t(x,\cdot)$, set
$s_a=\hat f^t(x,a)+\beta_tq_t(x,a)+\alpha_rw_t(x,a)$.
The exact policy problem is the robust linear program
\begin{equation}\label{eq:comp-robust-lp}
\begin{aligned}
\max_{p\in\mathbb R^K}\quad&\sum_a p_as_a\\
\text{subject to}\quad&p_a\ge0,\qquad \sum_a p_a=1,\\
&\sum_a p_ac(x,a)\le\tau\qquad\text{for all }c\in\cS_t.
\end{aligned}
\end{equation}
Equivalently, the last constraint is $M_t(x,p)\le\tau$. The support
function is convex in $p$, so this is a convex policy optimization
problem with $K$ decision variables. A support query tests feasibility;
if a returned model violates the constraint, its cost vector supplies
a separating row. This gives a separation-based implementation of the
semi-infinite LP. A finite-row LP requires additional structure.

When the convex hull of the attainable cost vectors at $x$ is supplied
as a polytope with vertices $b^1,\ldots,b^J\in[0,1]^K$, the formulation
reduces exactly to
\begin{equation}\label{eq:comp-finite-lp}
\begin{aligned}
\max_{p\ge0}\quad&s^\top p\\
\text{subject to}\quad&\mathbf1^\top p=1,\qquad
(b^j)^\top p\le\tau\quad(j\in[J]).
\end{aligned}
\end{equation}
Ordinary LP algorithms apply. Both formulations include the known safe
action, so their feasible regions are nonempty. If a solver returns a
certified feasible policy whose exact score is within $\zeta_t$ of
the optimum at every context,
we obtain the same optimism inequality with an additional $\zeta_t$.
Summing the regret proof adds at most $\sum_t\zeta_t$. The solver's
precision and tie-breaking rule are archived and reused during replay.

With approximate support values, an upper bound certifies safety,
whereas a violating model supplies a valid separating row. Support and
policy tolerances must also certify the stated objective error
$\zeta_t$; feasibility alone does not control the regret loss.

\subsection{Computing the Counterfactual Bonus by Cumulative Policy Mass}
\label{app:computation-bonus}
At a queried context $x$, initialize $\omega_1(x,a)=0$. Replay each archived
stage and update
\[
q_i(x,a)=\frac1{1+\omega_i(x,a)},\qquad
\omega_{i+1}(x,a)=\omega_i(x,a)+\widetilde\pi_{t,i}(a\mid x).
\]
We compute a stage's policy using~\eqref{eq:comp-robust-lp}, or set it to
$\pi_0$ if the archive instructs choosing the safe action $a_0$.
The updated vector $\omega_{i+1}(x,\cdot)$ is
exactly $\sum_{j\le i}\pi_j(\cdot\mid x)$ by
Lemma~\ref{lemma:equivalent_finite}. Each stage needs $O(K)$ bonus
arithmetic and $O(K)$ working storage.
Through stage $t$, the bonus computation requires $O(Kt)$
arithmetic operations. If the same context appears again,
we reuse its stored cumulative policy masses and update them
through the newly archived stages.

\paragraph{Total computational accounting.}
There are at most $T-1$ calls to each regression oracle and at most
$T-1$ convex feasibility checks for empty confidence sets; replay never
refits either model. Across $T$ distinct queried contexts, replay uses
at most $T(T+1)/2$ policy solves and $O(KT^2)$ bonus arithmetic.
Caching each current fit's values on its data prefix uses at most
$T(T-1)/2$ cost-fit evaluations in total. These values are reused
during replay. For the affine representation, reference values and
feature vectors on the observed pairs can also be cached. The
quadratic coefficients above can then be assembled in
$O(Td^2+T^2d)$ arithmetic operations: $G_t$ is updated incrementally,
while $g_t$ and $\kappa_t$ are recomputed from each new fit's full
prefix. These counts exclude the regression and optimization solves.
If $L_{t,i}$ support-based separation tests are used at replay stage $i$
of round $t$, the cost-support queries are bounded by
\[
\sum_{t=1}^T\sum_{i=1}^t
 (2K+L_{t,i}),
\]
with zero contribution at stages where the recorded instruction selects
$a_0$. These counts separate
regression, bonus calculation, and constrained optimization.

Each archived confidence set after stage one uses that stage's single
current-fit constraint; replay does not replace it with a later set
or refit a predictor.
There is no covering-number factor in these computational counts.
The cost of each support or policy solve depends on the representation,
the prefix length (or the compressed quadratic representation), and
the requested accuracy.
A polynomial running-time guarantee therefore requires efficient
regression and optimization routines with suitable numerical encodings;
it does not follow from an abstract convex class alone.

\subsection{Computing the Bonus and Policy for Infinite Actions}
The cost update is unchanged. For the reward bonus, an implementation
needs the reference design and the function-pair optimization
\[
q_t^\infty(x,a)=\sup_{f,f'\in\cF}
\frac{(f(x,a)-f'(x,a))^2}
{d_x\int(f(x,b)-f'(x,b))^2\,\rho_x^*(db)
 +\sum_{i<t}\int(f(x,b)-f'(x,b))^2\,\pi_i(db\mid x)}.
\]
We use the zero-denominator convention from Definition~\ref{def:policy_div}.
At $t=1$ only the reference term appears; it creates no regression
samples. Policy optimization is~\eqref{eq:comp-robust-lp} with action
integrals replacing finite sums. Its solution requires an optimization
procedure for the supplied class and action space.

\begin{sloppypar}
In the two-feature example of Appendix~\ref{app:rank-two-example},
the bonus is the explicit quadratic form
$\phi(a)^\top S_{t-1}(x)^{-1}\phi(a)$, and we reduce policy optimization
to three variables $(p,a_1,a_2)\in[0,1]^3$. These features are part of
that example's representation. The cost geometry derived for the
general regret proof is not an additional computational input.
\end{sloppypar}

\section{Omitted Proofs in Appendix~\ref{app:cost-property}}
\label{app:omitted-spps-proofs}

In this appendix, we give the omitted proofs of the auxiliary lemmas
and propositions used in Appendix~\ref{app:cost-property} to prove
Lemma~\ref{lemma:cost_confidence} and Property~\ref{prop:sequential_width},
followed by the proof of Corollary~\ref{cor:cost-width-radius}.
We retain the notation introduced there and present the proofs in
the order of their statements.

\begin{proof}[Proof of Lemma~\ref{lemma:cost-retention}]
Since $c^*\in\cC$, it belongs to $\cS_1$. On $\mathcal E_C$,
the true function passes the constraint from every fit
$\hat c^t$, $2\le t\le T$. Each defining test therefore includes
$c^*$ in $\cS_t$. Hence every confidence set is nonempty on this
single event, simultaneously over all rounds.

For every $c\in\cS_t$, its defining constraint and the oracle event
give
\[
\begin{aligned}
\sum_{i<t}(c(Z_i)-c^*(Z_i))^2
&\le2\sum_{i<t}(c(Z_i)-\hat c^t(Z_i))^2
 +2\sum_{i<t}(\hat c^t(Z_i)-c^*(Z_i))^2\\
&\le2B_t^{\cC}+2B_t^{\cC}=4B_t^{\cC}.
\end{aligned}
\]
The calculation holds for every candidate on the same event, so
adaptive or post hoc selection of $c$ requires no further union
bound. The empirical calculation only uses the current regression
constraint and does not require $\hat c^t\in\cC$.
\end{proof}

\begin{proof}[Proof of Lemma~\ref{lemma:cost-pair-budget}]
For $t\ge2$, both candidates satisfy the current regression
constraint. We therefore obtain
\[
\begin{aligned}
\sum_{i<t}(c(Z_i)-c'(Z_i))^2
&\le2\sum_{i<t}(c(Z_i)-\hat c^t(Z_i))^2
 +2\sum_{i<t}(c'(Z_i)-\hat c^t(Z_i))^2\\
&\le4B_t^{\cC}.
\end{aligned}
\]
This implication is deterministic on every history for which the
pair belongs to $\cS_t$; it does not invoke truth retention.
At $t=1$, the empty empirical sum is zero.
\end{proof}

\begin{proof}[Proof of Lemma~\ref{lemma:entropy-linear-span}]
We first give the volume argument for a symmetric convex set $D$
contained in the unit ball of a Banach space. Suppose
$f_1,\ldots,f_m\in D$ satisfy
\[
\operatorname{dist}\bigl(f_i,\operatorname{span}\{f_1,\ldots,f_{i-1}\}\bigr)>r.
\]
For each $i$, we apply the Hahn--Banach theorem in the quotient by the
preceding span to obtain a bounded linear functional $\ell_i$ with norm one,
vanishing on that span, and $\ell_i(f_i)>r$ after a possible sign
change. Let $L(f)=(\ell_1(f),\ldots,\ell_m(f))$.
The matrix with columns $L(f_j)$ is triangular with diagonal entries
larger than $r$. Its determinant consequently has magnitude greater
than $r^m$.

Symmetry and convexity give
$\operatorname{absconv}\{f_1,\ldots,f_m\}\subseteq D$. Its image under
$L$ is an invertible linear image of the $\ell_1$ unit ball in
$\RR^m$. The latter has volume $2^m/m!$, and hence that image has
volume greater than $2^mr^m/m!$.
Every uniform ball of radius $u$ has its image inside an
$m$-dimensional cube of side $2u$, because $\|\ell_i\|=1$.
Comparing volumes for any finite cover of $D$ gives
\begin{equation}\label{eq:greedy-volume}
\mathcal N(D,u,\|\cdot\|_\infty)\ge\frac{(r/u)^m}{m!}.
\end{equation}
Only the compact cross-polytope just constructed enters this volume
comparison; no compactness of $D$ is needed.

Now suppose $\log\mathcal N(D,u)\le a[\log(b/u)]^p$, with
$a\ge1$, $b\ge e$, and $p\ge1$. For $s>0$, choose
$u=re^{-s}/m$, which lies in $(0,1]$. Since $m!\le m^m$,
\eqref{eq:greedy-volume} gives
\begin{equation}\label{eq:greedy-implicit-dimension}
ms\le a[\log(bm/r)+s]^p.
\end{equation}
For $p=1$, division by $s$ followed by $s\to\infty$ yields $m\le a$.
For $p>1$, put $q=p-1$ and
$k_p=p^p/(p-1)^{p-1}$. We choose
$s=\log(bm/r)/(p-1)$ in~\eqref{eq:greedy-implicit-dimension} and obtain
\[
m\le k_pa[\log(bm/r)]^q.
\]
Here is an explicit resolution of this implicit inequality.
Set $a'=k_pa$ and $L_0=\log(ba'/r)\ge1$. If $m<a'$, then
$m\le a'L_0^q$. Otherwise let $x=m/a'\ge1$ and $y=x^{1/q}$.
The last inequality implies $y\le L_0+q\log y$. Using
$\log u\le u-1$ with $u=y/(2q)$, we bound
\[
q\log y
\le y/2+q\log(2q)-q
\le y/2+q\max\{0,\log(2q)\}.
\]
It follows that
$y\le[2+2q\max\{0,\log(2q)\}]L_0$.
This establishes
$m\le C_p'a[\log(eba/r)]^{p-1}$ for a constant depending only on $p$.

We apply this result to $D=\cD$, $a=2D_{\cC}$, and $b=2A_{\cC}$ using
\eqref{eq:difference-entropy}. Starting from the zero space, choose
any $f_i\in\cD$ whose distance from the current span exceeds $r$,
whenever such an element exists. A supremum need not be attained:
the strict inequality provides a permissible choice. The dimension
bound just proved prevents this construction from continuing
arbitrarily long. At termination its span satisfies
\eqref{eq:span-approximation}. Absorbing the fixed factors in $a,b$
into $C_{p_{\cC}}$ gives the stated constants.

For completeness, when $p_{\cC}=1$ we obtain the sharper affine-dimension
claim directly from $\cC$. If $\bar c_0,\ldots,\bar c_d\in\cC$ are affinely
independent, the simplex they span is contained in $\cC$. On the
$d$-dimensional span of $\bar c_i-\bar c_0$, choose continuous coordinate
functionals and extend them to $\mathbb B$ by the Hahn--Banach theorem. Their
joint image of this simplex has positive $d$-dimensional volume.
The image of a uniform $u$-ball has volume at most a constant times
$u^d$. Thus
\[
\log\mathcal N(\cC,u,\|\cdot\|_\infty)
\ge d\log(1/u)+O(1)\quad(u\downarrow0).
\]
Comparison with $D_{\cC}\log(A_{\cC}/u)$ and division by $\log(1/u)$ gives
$d\le D_{\cC}$. No larger affinely independent family can exist, so
$\cC-\cC$ lies in a linear space of dimension at most
$\lfloor D_{\cC}\rfloor$. This finite-dimensional subspace is closed,
so it also contains $\cD$, proving the exact-span assertion.
\end{proof}

\begin{proof}[Proof of Lemma~\ref{lemma:john-features}]
Choose an approximating space $V_r$ of dimension $m$. A closest
approximant $v_h\in V_r$ exists for each $h$: a minimizing sequence
is bounded, hence has a convergent subsequence in the finite-dimensional
space, and the uniform norm is continuous. We can therefore choose
$\|h-v_h\|_\infty\le r$. Since $\|h\|_\infty\le1$, this implies
$\|v_h\|_\infty\le1+r$.

The unit ball $B_{V_r}$ is a compact, centrally symmetric convex body
with nonempty interior relative to $V_r$. The symmetric form of
John's ellipsoid theorem \citep{john1948extremum} supplies an ellipsoid $E$ satisfying
\[
E\subseteq B_{V_r}\subseteq\sqrt m\,E.
\]
Choose a linear isomorphism $J:\RR^m\to V_r$ that maps the Euclidean
unit ball onto $E$, and let
$\phi_j(z)=(Je_j)(z)$. Each coordinate is a bounded measurable
function. For each $z$,
\[
\|\phi(z)\|_2
=\sup_{\|\theta\|_2\le1}|(J\theta)(z)|\le1,
\]
where the last inequality uses $E\subseteq B_{V_r}$.
Since $v_h\in(1+r)B_{V_r}\subseteq(1+r)\sqrt m\,E$, its coordinate
vector $\theta_h=J^{-1}v_h$ obeys
$\|\theta_h\|_2\le(1+r)\sqrt m$.
This gives both claims in~\eqref{eq:john-approximation}.
For $m=0$, the approximating space is $\{0\}$, which proves the final
statement directly.
\end{proof}

\begin{proof}[Proof of Lemma~\ref{lemma:cost-width-potential}]
First suppose $\cS_t\ne\varnothing$. Fix $c,c'\in\cS_t$ and put
$h=c-c'\in\cD$. Choose the approximating vector $\theta_h$ supplied
by Lemma~\ref{lemma:john-features}. The empirical pair bound,
monotonicity of $B_t^{\cC}$, and the empty sum at $t=1$ give
\[
\sum_{i<t}h(Z_i)^2\le B_T^{\rm pair}=4B_T^{\cC}.
\]
We use $|\theta_h^\top\phi(z)-h(z)|\le r$ to obtain
\[
\begin{aligned}
\|\theta_h\|_{\Sigma_t}^2
&=\|\theta_h\|_2^2+\sum_{i<t}(\theta_h^\top\phi(Z_i))^2\\
&\le4m+2\sum_{i<t}h(Z_i)^2+2(t-1)r^2\\
&\le4m+8B_T^{\cC}+2Tr^2=R^2(r).
\end{aligned}
\]
The Cauchy--Schwarz inequality for the inner product induced by
$\Sigma_t$ now gives
\[
|h(z)|
\le r+|\theta_h^\top\phi(z)|
\le r+\|\theta_h\|_{\Sigma_t}\sqrt{\ell_t(z)}
\le r+R(r)\sqrt{\ell_t(z)}.
\]
Taking the supremum over $c,c'\in\cS_t$ proves
\eqref{eq:width-pointwise-leverage}. If $\cS_t=\varnothing$, the
width was defined to be zero, and the same bound holds immediately.
The argument applies on every history, simultaneously to every
candidate pair and every $z$, without a probability event.
\end{proof}

\begin{proof}[Proof of Proposition~\ref{prop:cost-width-bound}]
For $m\ge1$, the all-history conditional bound
\eqref{eq:cost-width-predictable} gives
\[
\sum_{t=1}^T\EE[w_t(Z_t)\mid\cH_{t-1}]
\le Tr_T+R_T\sum_{t=1}^T\sqrt{\mu_t}
\le Tr_T+R_T\sqrt{T\sum_{t=1}^T\mu_t},
\]
where we used the Cauchy--Schwarz inequality in the last step.
On $\mathcal E_{\rm geom}(r_T)$, we substitute
\eqref{eq:predictable-cost-potential} into this display and obtain
the second entry of the minimum in~\eqref{eq:cost-width-budget}.
Because $\cS_t\subseteq\cC$ whenever it is nonempty, and our empty-set
width is zero, we also have $w_t(z)\le\operatorname{diam}_\infty(\cC)$ on every
history. This gives the first entry.

Intersect with $\mathcal E_C$. On this intersection all virtual
confidence sets equal the sets used by the algorithm, retain
$c^*$, and remain nonempty. Thus
\[
\EE[w_t(Z_t)\mid\cH_{t-1}]
=\EE_{X\sim Q_x}\sum_a\pi_t(a\mid X)w_t(X,a),
\]
and the preceding bound is~\eqref{eq:cumulative-cost-width}.
Both events were established under the original law, with failure
probabilities at most $\delta/4$ each, so their intersection has
probability at least $1-\delta/2$.

If $m=0$, Lemma~\ref{lemma:john-features} gives
$\|c-c'\|_\infty\le r_T$ for every $c,c'\in\cC$.
Consequently $w_t\le\min\{\operatorname{diam}_\infty(\cC),r_T\}$ on every history.
Truth retention on $\mathcal E_C$ proves the stated bound with
probability at least $1-\delta/4$, and hence also at least
$1-\delta/2$. No geometric concentration event is needed in this case.
\end{proof}

\begin{proof}[Proof of Corollary~\ref{cor:cost-width-radius}]
Choose a fixed approximating space at the specified radius $r$.
For $m\ge1$, Lemma~\ref{lemma:cost-width-potential} and its
conditional version~\eqref{eq:cost-width-predictable} apply with
this space. The geometric event $\mathcal E_{\rm geom}(r)$ in
\eqref{eq:predictable-cost-potential} has probability at least
$1-\delta/4$. On that event, the Cauchy--Schwarz inequality gives
\[
\begin{aligned}
\sum_{t=1}^T\EE[w_t(Z_t)\mid\cH_{t-1}]
&\le Tr+R_T(r)\sum_{t=1}^T\sqrt{\mu_t}\\
&\le Tr+\sqrt{T R_T^2(r)L_T(r)}.
\end{aligned}
\]
The deterministic bound $w_t\le\operatorname{diam}_\infty(\cC)$ supplies the
other entry of the minimum. Intersecting with $\mathcal E_C$
ensures truth retention and identifies these quantities with the
algorithm's confidence widths, with joint probability at least
$1-\delta/2$. If $m=0$, every difference in $\cD$ has norm at most
$r$, so the deterministic bound $w_t\le r$ and the oracle event
alone suffice.

For finite actions, intersect this cost event with the two reward
events used in the proof of Theorem~\ref{thm:main_guarantee}.
Each reward event fails with probability at most $\delta/4$, so the
intersection has probability at least $1-\delta$. Truth retention
gives feasibility at every round and context. Summing
\eqref{eq:finite_regret_step} and using the same reward-potential
calculation yields~\eqref{eq:finite-regret-arbitrary-radius}, with
$\mathcal W_T(r)$ in place of $\mathcal W_T$.
Neither a policy optimization nor a regression fit depends on the
chosen analysis radius.
\end{proof}

\end{document}